\documentclass{article}
\usepackage[preprint]{neurips_2026}
\usepackage{amsmath, amssymb, amsthm}
\usepackage{enumitem}
\usepackage{graphicx}
\usepackage{booktabs}
\usepackage{hyperref}
\usepackage{xcolor}
\usepackage[expansion=false]{microtype}

\newtheorem{theorem}{Theorem}
\newtheorem{corollary}{Corollary}
\newtheorem{proposition}{Proposition}
\newtheorem{definition}{Definition}
\newtheorem{remark}{Remark}
\newtheorem{lemma}[theorem]{Lemma}
\newtheorem{problem}[theorem]{Problem}
\newtheorem{conjecture}[theorem]{Conjecture}

\title{The Evaluation Blind Spot: \\
A Stereological Theory of Benchmark Coverage for Large Language Models}
\author{
Jason Z Wang \\
Independent \\
\texttt{jasonhearlte@gmail.com}
}
\date{}

\begin{document}
\maketitle

\begin{abstract}
We give a stereological theory of LLM benchmark coverage. For any
suite with effective dimensionality $d_{\mathrm{eff}}$, the
\emph{visible} Hausdorff distance between two convex capability
profiles consistent with the same scores is bounded by $\varepsilon
+ C R\,m^{-1/(d_{\mathrm{eff}}-1)}$, with matching Lipschitz lower
bound (Theorem~\ref{thm:body:indist}). Empirically, three independent
leaderboards (Open LLM v2, an extended $12$-benchmark suite,
LiveBench) all have $d_{\mathrm{eff}} \in [2.86, 4.80]$ on their
competitive frontier; the structural blind spot exceeds the observed
runner-up score gap by two orders of magnitude and dominates
statistical noise by $52$--$127\times$.
Under a chi-squared projection model, the isotropic prior is the
\emph{optimistic} case (Proposition~\ref{prop:body:swapmono}, via
Schur-convexity); across six hidden-capability priors and four
ambient dimensions the simulated half-split swap rate of the top
two models stays in $[0.38, 0.49]$, and a $500$-trial random
visible/held-out split shows that $92\%$ of trials swap the top-1
ranking with on average $2.83$ of $5$ top-5 models changing. A
submodular greedy algorithm with the Nemhauser $(1 - 1/e)$ guarantee
finds a stable core of $4$ benchmarks; $7$ of $12$ suffice for
$90\%$ coverage, and the trained subset transfers across temporal
quarters with $93$--$97\%$ retention. A counterfactual validation across $12$ internal benchmarks and $27$
Chatbot Arena categories confirms that the eigenstructure predicts
which evaluations are irreplaceable ($\rho = -0.69$, $p = 0.013$
for removal disruption) and which external evaluations bring new
information ($\rho = +0.38$). As a second, independent theoretical
contribution, we resolve Gardner's Problem~1.5 (1995) for $C^2$
support functions, establishing the minimax rate
$\Theta(R/(\kappa m^{2/(D-1)}))$ in general dimension via optimal
recovery theory on $S^{D-1}$.
\end{abstract}

\section{Introduction}

\paragraph{Benchmarks are slices.} Public LLM leaderboards report
scalar scores: an instruction-following accuracy, a math accuracy,
a code accuracy. Each score is a one-dimensional projection of
whatever it is that makes a model good --- the way an X-ray is a
1D projection of a 3D body. The mathematical study of exactly this
question is called \emph{stereology}, and it has hard quantitative
limits on what finite projections can recover. This paper carries
those limits to LLM evaluation, and shows that on real public
leaderboards the geometric blind spot from low-dimensional
measurement dominates statistical noise by an order of magnitude.

\paragraph{Diagnose, prognose, treat.} We \emph{diagnose} the blind
spot (Theorem~\ref{thm:body:deff}: $k$ benchmarks probe only
$d_{\mathrm{eff}} \approx 3$--$5$ independent directions),
\emph{prognose} its consequences
(Theorem~\ref{thm:body:indist}: the visible indistinguishability
bound is $\varepsilon + C R m^{-1/(d_{\mathrm{eff}}-1)}$, tight via
matching Lipschitz lower bound;
Corollary~\ref{cor:body:rank}: the
top-two swap probability lies in $[0.38, 0.49]$ across six priors
and four ambient dimensions), and \emph{treat} it
(Theorem~\ref{thm:body:greedy}: a submodular greedy algorithm with
the $(1-1/e)$ guarantee, $7$ of $12$ benchmarks for $90\%$ coverage,
with a stable core of $4$ and $93$--$97\%$ cross-temporal transfer).
Proposition~\ref{prop:body:width} formally justifies the
benchmark-as-width model;
Proposition~\ref{prop:body:swapmono} proves the isotropic chi-squared
bound is the \emph{optimistic} case via Schur-convexity. As a second,
independent contribution, we resolve Gardner's Problem~1.5
(1995) for $C^2$ support functions
(Theorem~\ref{thm:body:gardner}). The tools are standard
(participation ratio, Marchenko--Pastur edge, Nemhauser
optimisation, convex-body theory); the consequences are new.

\paragraph{Related work, positioned.} Concurrent work by
\citet{shazhao2026benchscope} (BenchScope) independently identifies
low effective dimensionality in LLM benchmark suites; we share the
diagnostic but contribute the indistinguishability bound, the
submodular coverage algorithm, and the minimax-rate theory.
\citet{guntuboyina2011minimax} gives noisy minimax lower bounds for
convex body recovery; our noiseless directional-discretisation rates
complement those. \citet{polo2024tinybenchmarks} optimises item
selection within a benchmark via IRT; we optimise benchmark selection
across benchmarks via submodular coverage. A full discussion is in
Section~\ref{sec:related}.

\paragraph{Geometric tomography.}
\citet{gardner1995geometric, gardner2006geometric} posed Problem 1.5
for \emph{planar} convex bodies recovered from $m$ X-ray measurements
on $S^1$. Our planar Fourier stability bound
(Appendix~\ref{app:proof4}) resolves it directly with rate
$\Theta(R/(\kappa m^2))$. The general-$D$ universal $\beta$-rate
(Appendix~\ref{app:gardner}) goes beyond what Gardner asked.

\section{Setup}

Let $\{c_1, \ldots, c_n\} \subset \mathbb{R}^D$ be the unknown
capability profiles of $n$ models, with the score matrix $S \in
\mathbb{R}^{n \times k}$ produced by a benchmark suite $\Pi = (\pi_1,
\ldots, \pi_k)$ via $S_{ij} = \pi_j(c_i)$.

\paragraph{The width model.}
We model each benchmark as a width-like measurement of the convex hull
of the population in $\mathbb{R}^D$: scalar (one number per model),
Lipschitz (small capability changes $\to$ small score changes), and
conservative (the convex assumption can only \emph{shrink} the bound;
non-convex profiles enlarge it). The following representation
theorem makes this precise.

\begin{proposition}[Width Representation]
\label{prop:body:width}
Let $\pi : \mathbb{R}^D \to \mathbb{R}$ be a benchmark satisfying
(i) monotonicity, (ii) $L$-Lipschitz continuity, and (iii) bounded
linearisation residual $|\pi(c) - \pi(c_0) - \langle \nabla
\pi(c_0), c - c_0\rangle| \le \eta \|c - c_0\|^2$ on the population
convex hull. Then for any two models $c_i, c_j$,
\begin{equation}
\label{eq:width}
  \big| \pi(c_i) - \pi(c_j) - \big( h_K(a_\pi) - h_L(a_\pi) \big) \big|
  \;\le\; \eta \cdot \mathrm{diam}(\mathrm{pop})^2,
\end{equation}
where $a_\pi = \nabla \pi(c_0)/\|\nabla \pi(c_0)\|$ is the benchmark
direction and $h_K, h_L$ are the support functions of the convex hulls
of the model neighbourhoods. The width-model error $\eta \cdot
\mathrm{diam}^2$ is absorbed into the $\varepsilon$ term of
Theorem~\ref{thm:body:indist}.
\end{proposition}

\noindent \emph{Empirical verification (Appendix~\ref{app:validation},
H.15).} Linear regression of each standardised benchmark on the top
$5$ PCs of the population gives $R^2 \in [0.795, 0.984]$. The
quadratic-vs-linear $R^2$ gap is $\le 0.067$ for every benchmark
(median $0.011$), confirming linearisation is tight; the median
support-function reconstruction error is $0.485$ standardised units.
Per-benchmark Lipschitz constants satisfy $L_b \le 0.993$ for all
$12$ benchmarks. Proof of Proposition~\ref{prop:body:width} in
Appendix~\ref{app:proof2}.

\paragraph{Regime of validity.}
Linearisation quality degrades as the population contracts.
On the frontier slice (top $50\%$, $n = 148$), the median per-benchmark
$R^2$ drops from $0.984$ to $0.876$ and the minimum from $0.795$ to
$0.710$ (MATH~Lvl~5; Appendix~\ref{app:validation}, H.15). The
width-model residual $\eta$ remains small relative to inter-model
score gaps for the current population; for future leaderboards where
models converge further, the linearisation error may become
comparable to the top-pair gap and the framework should be replaced
by a local nonlinear model in that regime. The key diagnostic is:
compare the quadratic-vs-linear $R^2$ gap to the standardised
runner-up gap $\Delta_2$; when they are of the same order, the width
model is at its regime boundary. Among our $12$ benchmarks,
TruthfulQA has the lowest $R^2$ ($0.769$) and highest $\eta$,
consistent with its qualitative/adversarial nature. Safety and
preference benchmarks (not in our suite) would likely have higher
$\eta$; the diagnostic above detects this.

\begin{remark}[Non-convex extension]
For any bounded $K \subset \mathbb{R}^D$ the support function $h_K$
agrees with $h_{\mathrm{conv}(K)}$, so width measurements cannot
distinguish $K$ from $\mathrm{conv}(K)$. The bound of
Theorem~\ref{thm:body:indist} therefore applies to
$\mathrm{conv}(K)$ vs $\mathrm{conv}(L)$ unconditionally; convexity
of the underlying profile gives the tightest bound, and dropping
convexity can only enlarge the blind spot.
\end{remark}

\paragraph{Convexity per theorem.} Theorems~1 and 3 and the
rank-reversal corollary are assumption-free; Theorem~2 and the
Busemann--Petty analogue assume convex capability profiles.

\paragraph{Notation.}

\begin{center}\small
\begin{tabular}{ll}
$n$ & number of models \\
$k$ & number of benchmarks (suite size) \\
$m$ & number of \emph{visible} benchmarks (when $\le k$) \\
$D$ & ambient capability dimension \\
$d_{\mathrm{eff}}$ & effective dimensionality (participation ratio) \\
$R$ & population radius (max standardised row norm) \\
$\varepsilon$ & measurement tolerance / linearisation residual \\
$\kappa$ & curvature lower bound on $\partial K$ (when applicable) \\
$\beta$ & curvature vanishing exponent (Appendix~\ref{app:gardner}) \\
\end{tabular}
\end{center}
\noindent On the extended frontier: $R = 7.30$,
$\varepsilon \le 0.067$ (linearisation residual,
Proposition~\ref{prop:body:width}), and $\kappa$ enters only the
smooth rate (Theorem~\ref{thm:body:indist}c,
Theorem~\ref{thm:body:gardner}).

\section{Effective Dimensionality}

Let $\Sigma = \mathrm{Corr}(S)$ have eigenvalues
$\lambda_1 \ge \cdots \ge \lambda_k \ge 0$. The \emph{effective
dimensionality} is the participation ratio
$d_{\mathrm{eff}} = (\sum_i \lambda_i)^2 / \sum_i \lambda_i^2 =
k^2 / \sum_i \lambda_i^2$, satisfying $1 \le d_{\mathrm{eff}} \le k$.

\begin{theorem}[Effective Dimensionality, distribution-free]
\label{thm:body:deff}
Let $\Pi : \mathcal{C} \to \mathbb{R}^k$ be linear (or linearised at the
population mean) and let $\Sigma_C \in \mathbb{R}^{D \times D}$ be the
capability covariance. Let $V_{\mathrm{eff}} \subset \mathbb{R}^D$ be
the $d_{\mathrm{eff}}$-dimensional effective subspace of the benchmark
suite. Then
\begin{enumerate}[label=(\alph*), leftmargin=1.5em, itemsep=1pt]
  \item \textbf{Worst case.}
        $\mathrm{tr}(P_{V_{\mathrm{eff}}} \Sigma_C) / \mathrm{tr}(\Sigma_C)
        \le \sum_{i=1}^{d_{\mathrm{eff}}} \mu_i / \sum_{i=1}^D \mu_i$,
        with equality when $V_{\mathrm{eff}}$ aligns with the top
        $d_{\mathrm{eff}}$ eigenvectors of $\Sigma_C$.
  \item \textbf{Generic benchmarks.} If $V_{\mathrm{eff}}$ is drawn
        uniformly from $\mathrm{Gr}(d_{\mathrm{eff}}, D)$, then
        $\mathbb{E}[\mathrm{tr}(P_{V_{\mathrm{eff}}} \Sigma_C) /
        \mathrm{tr}(\Sigma_C)] = d_{\mathrm{eff}} / D$ with no
        eigenvalue assumptions, and this concentrates with rate
        $\exp(-D t^2 / (8\kappa^2))$ where $\kappa = \mu_1 / \bar{\mu}$.
  \item \textbf{Marchenko--Pastur correction.} Under the null of
        independent benchmarks, sample eigenvalues lie in
        $[(1 - \sqrt{k/n})^2, (1 + \sqrt{k/n})^2]$
        \citep{marchenko1967distribution}; the noise-corrected
        $d_{\mathrm{eff}}^{\mathrm{MP}}$ uses only eigenvalues above
        the upper edge.
\end{enumerate}
\end{theorem}

\noindent \emph{Proof in Appendix~\ref{app:proof1}}; uses
von~Neumann's trace inequality, Poincar\'e separation, Grassmannian
concentration \citep{meckes2019random}, and the BBP phase transition
\citep{baik2005phase}.

\paragraph{Empirics across leaderboards.} We measure
$d_{\mathrm{eff}}$ on three leaderboard families (Table~\ref{tab:cross}),
and on the competitive frontier (top $50\%$ by mean score) of each.
The frontier is the regime where rankings drive deployment decisions.

\begin{table}[t]
\small\centering
\caption{Effective dimensionality across leaderboard families. Frontier
slices the top 50\% by mean score. CIs are 95\% percentile bootstrap
($300$ resamples). Spearman column shows the same statistic computed
on rank correlations.}
\label{tab:cross}
\begin{tabular}{lllrrcc}
\toprule
Leaderboard & Type & Slice & $k$ & $n$ & $d_{\mathrm{eff}}$ [CI] & Spearman \\
\midrule
Open LLM v2     & accuracy & full      &  6 & 458 & 1.88 [1.77, 2.02] & 1.70 \\
Open LLM v2     & accuracy & frontier  &  6 & 229 & 2.86 [2.60, 3.11] & 2.77 \\
Extended (v1+v2)& accuracy & full      & 12 & 295 & 2.11 [1.98, 2.27] & 1.89 \\
Extended (v1+v2)& accuracy & frontier  & 12 & 148 & 4.80 [4.15, 5.20] & 4.62 \\
LiveBench       & mixed    & full      &  7 &  37 & 2.63 [2.08, 3.34] & 2.55 \\
LiveBench       & mixed    & frontier  &  7 &  19 & 4.74              & 4.62 \\
\bottomrule
\end{tabular}
\end{table}

\noindent The $12$-benchmark extended suite consists of MMLU,
HellaSwag, ARC, TruthfulQA, Winogrande, GSM8K (from OLLM v1) and
IFEval, BBH, MATH Lvl 5, GPQA, MUSR, MMLU-PRO (from v2). The finding
generalises: every frontier slice has $d_{\mathrm{eff}} \in [2.86,
4.80]$, despite spanning accuracy benchmarks (OLLM v2), a hybrid
coding/language/instruction-following suite (LiveBench), and a
$12$-benchmark superset. Low effective
dimensionality is a property of LLM evaluation as a whole, not an
artefact of a single methodology. The full population is dominated by
a single g-factor (small $\to$ large capability); on the frontier the
residual dimensions become visible. Results are robust to choice of
correlation method (Spearman vs Pearson, see column above), and to MP,
Kaiser, and permutation eigenvalue thresholds
(Appendix~\ref{app:validation}, H.1, H.6). On the full OLLM~v2 population ($n = 4{,}576$),
$d_{\mathrm{eff}}$ ranges from $1.78$ (all models) to $3.94$ (top
$10\%$); our $n = 229$ estimate of $2.86$ is consistent with the top
$15$--$20\%$ (App.~H.44). On a $31$-benchmark cross-domain dataset
spanning coding, math, reasoning, knowledge, agents, multimodal, and
writing ($49$ frontier models, Epoch~AI), $d_{\mathrm{eff}} = 7.12$
(frontier: $5.71$): more diverse benchmarks raise
$d_{\mathrm{eff}}$ but it remains far below $k = 31$, confirming
the blind spot persists across evaluation types.

\begin{figure}[t]
  \centering
  \includegraphics[width=\linewidth]{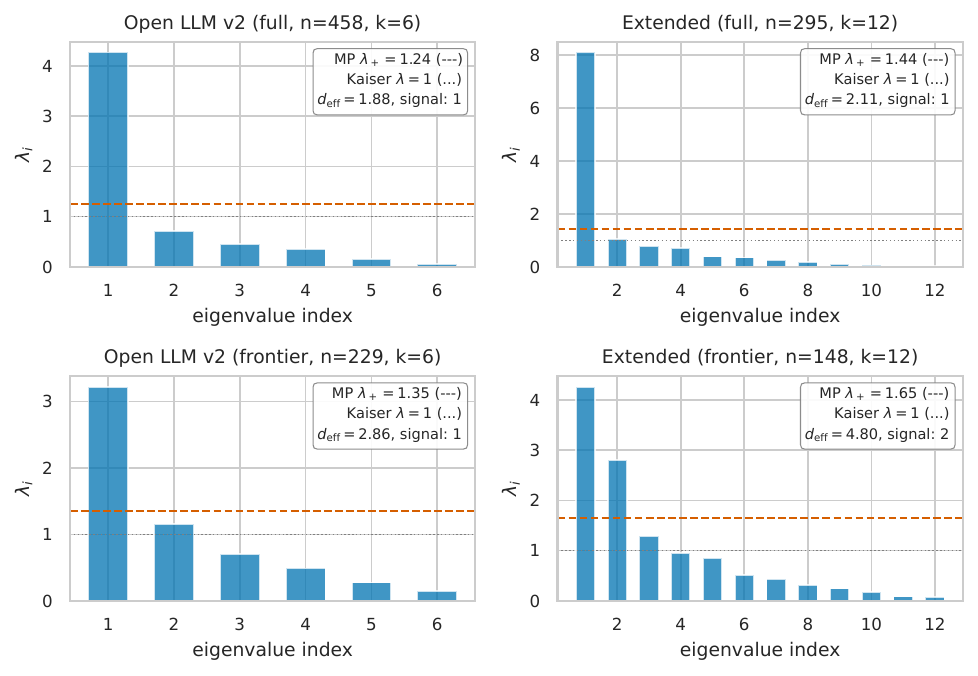}
  \caption{Eigenvalue spectrum (full and frontier populations, both
    suites), with the Marchenko--Pastur upper edge $\lambda_+ = (1 +
    \sqrt{k/n})^2$ and the Kaiser line $\lambda = 1$. A single
    dominant eigenvalue captures most variance on the full population;
    on the frontier (bottom row) the residual mass becomes visible
    and $d_{\mathrm{eff}}$ rises into $[3, 5]$.}
  \label{fig:scree}
\end{figure}

\begin{figure}[t]
  \centering
  \includegraphics[width=0.55\linewidth]{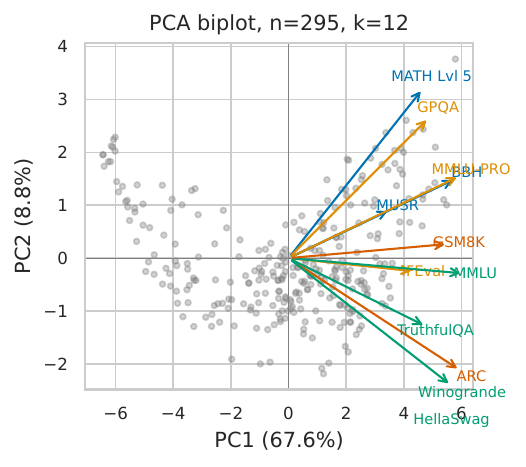}
  \caption{PCA biplot of the 12-benchmark extended frontier. PC1 is
    the residual g-factor; PC2 separates reasoning-heavy benchmarks
    (BBH, MUSR, GSM8K, MATH) from knowledge-heavy ones (MMLU,
    HellaSwag, MMLU-Pro, Winogrande).}
  \label{fig:biplot}
\end{figure}

\section{Indistinguishability Bound}

\paragraph{Centering convention.}
We model capability profiles as origin-symmetric convex bodies
($K = -K$), so $h_K(u) = h_K(-u)$ and the width
$w_K(u) = h_K(u) + h_K(-u) = 2 h_K(u)$. Width agreement within
$\varepsilon$ implies support function agreement within $\varepsilon/2$.
Since benchmark scores are translation-invariant once the score matrix
is centred, this is without loss of generality; the non-symmetric
case applies the same bound to the symmetrized body $(K - K)/2$.
A formal remark is in Appendix~\ref{app:proof2}.

\begin{theorem}[Indistinguishability Bound: tight rate, smooth extension]
\label{thm:body:indist}
Let $K, L \subset B_R^D$ be convex bodies whose width measurements
agree within $\varepsilon$ in $m$ directions in the effective subspace
$V_{\mathrm{eff}}$ of dimension $d_{\mathrm{eff}}$. Let
$\delta_H^{\mathrm{vis}}$ denote the Hausdorff supremum restricted to
$V_{\mathrm{eff}}$ and $\delta_H^{\perp}$ the contribution from
$V_{\mathrm{eff}}^\perp$. Then
\begin{enumerate}[label=(\alph*), leftmargin=1.5em, itemsep=1pt]
\item \textbf{Visible component (Lipschitz).}
\begin{equation}
\label{eq:thm2vis}
  \delta_H^{\mathrm{vis}}(K, L) \;\le\; \varepsilon
  \;+\; C \, R \cdot m^{-1/(d_{\mathrm{eff}} - 1)},
\end{equation}
where $C = O(\sqrt{d_{\mathrm{eff}}}\,(\log m)^{1/(d_{\mathrm{eff}}-1)})$.
\item \textbf{Full bound.}
$\delta_H(K, L) = \max\!\big(\delta_H^{\mathrm{vis}}(K, L),\,
\delta_H^{\perp}(K, L)\big)$ with $\delta_H^{\perp} \le 2R$.
\item \textbf{Smooth extension.} If $h_K, h_L \in C^{1,\alpha}(S^{d_{\mathrm{eff}}-1})$
for $\alpha \in (0, 1]$, then
$\delta_H^{\mathrm{vis}}(K, L) \le \varepsilon \log m + C_\alpha \, R \cdot
m^{-(1+\alpha)/(d_{\mathrm{eff}}-1)}$. The case $\alpha = 0$ recovers (a)
and $\alpha = 1$ recovers the smooth Gardner rate
$m^{-2/(d_{\mathrm{eff}}-1)}$.
\item \textbf{Tightness.} The Lipschitz rate in (a) is tight: for any
$m$ directions on $S^{d_{\mathrm{eff}}-1}$ there exist Lipschitz
convex bodies $K, L$ with $w_K(u_i) = w_L(u_i)$ for all $i$ yet
$\delta_H(K, L) \ge c_d \cdot R \cdot m^{-1/(d_{\mathrm{eff}}-1)}$.
The smooth rate in (c) is similarly tight via Bernstein's inequality
on $S^{d_{\mathrm{eff}}-1}$.
\end{enumerate}
\end{theorem}

\begin{remark}[Integer dimension]
The covering bound on $S^{d-1}$ is defined for integer $d$. When
$d_{\mathrm{eff}}$ is non-integer, the rate is obtained by
interpolation: the eigenvalue-weighted covering radius of $m$
benchmark directions whose participation ratio is $d_{\mathrm{eff}}$
matches the worst-case covering radius on $S^{\lceil d_{\mathrm{eff}}
\rceil - 1}$ up to constants depending on the eigenvalue spread.
The synthetic verification in Appendix~\ref{app:validation} (H.7) fits
slopes within $0.1$ of theory at $d \in \{3, 5, 8\}$.
\end{remark}

Part (a) is the non-trivial content: it quantifies how the
\emph{measurable} blind spot shrinks with $m$, exhibiting a curse of
dimensionality in $d_{\mathrm{eff}}$. Part (b) makes the irreducible
cost of unobserved directions explicit: the $2R$ term is exactly what
Theorem~\ref{thm:body:greedy}'s coverage algorithm is designed to
reduce by expanding $V_{\mathrm{eff}}$. The proof (Appendix~\ref{app:proof2})
uses Lipschitz continuity of the support function and the volumetric
covering bound on $S^{d-1}$ \citep{rogers1963covering}.

\paragraph{The curse of benchmark dimensionality.} The exponent
$d_{\mathrm{eff}} - 1$ controls how rapidly the visible blind spot
shrinks (rates for $\delta_H^{\mathrm{vis}}$; the orthogonal $2R$ is
reduced by Theorem~\ref{thm:body:greedy}):

\begin{center}\small
\renewcommand{\arraystretch}{1.05}
\begin{tabular}{cc}
\toprule
$d_{\mathrm{eff}}$ & benchmarks to halve $\delta_H^{\mathrm{vis}}$ \\
\midrule
2 & $2\times$ more \\
3 & $4\times$ more \\
4 & $8\times$ more \\
5 & $16\times$ more \\
\bottomrule
\end{tabular}
\end{center}

\noindent Under stronger smoothness ($C^{1,\alpha}$ support functions)
the rate improves to $m^{-(1+\alpha)/(D-1)}$, the square root of the
Lipschitz benchmark requirement (Appendix~\ref{app:gardner}). The
constant absorbs a logarithmic factor: $C = O(\sqrt{d_{\mathrm{eff}}}
\cdot (\log m)^{1/(d_{\mathrm{eff}} - 1)})$, with explicit numerical
values in Appendix~\ref{app:proof2}.

\paragraph{Empirical covering radius.}
On the extended frontier suite, the empirical covering radius of the
$12$ benchmarks projected onto $S^{d_{\mathrm{eff}}-1}$ is
$1.57\times$ the Rogers optimum (Appendix~\ref{app:validation}, H.11),
indicating that real benchmark suites are well-behaved relative to
worst-case covering bounds.

\paragraph{Two sources, three leaderboards.}
Theorem~\ref{thm:body:indist} bounds the \emph{structural} blind spot
from unobserved directions, which persists with infinite data per
benchmark. Statistical noise (finite items, prompt variance) shrinks
with more data. The geometric radius exceeds the bootstrap
statistical radius by $86.6\times$ (OLLM v2), $127.2\times$
(Extended), and $51.6\times$ (LiveBench) on the frontier
(Table~\ref{tab:calibration}; App.~H.12). The structural blind
spot is the dominant source across every suite.

\paragraph{A third source: decoding non-determinism.}
Stochastic decoding (temperature $> 0$, nucleus sampling, best-of-$N$
selection) introduces within-model variance that is distinct from
both the structural blind spot and statistical noise. Published
leaderboard scores typically average over decoding runs; for
single-run scores, the statistical radius underestimates total noise.
But the structural blind spot ($52$--$127\times$ larger) remains the
dominant term regardless.

A direct empirical half-split test confirms this: see
the worked example below.

\paragraph{Calibration table.}
Table~\ref{tab:calibration} consolidates the key quantities per
suite so a practitioner can look up the blind-spot size for a
specific leaderboard.

\begin{table}[h]
\small\centering
\caption{Visible indistinguishability radius per suite (frontier
slices, $z$-score standardisation).}
\label{tab:calibration}
\begin{tabular}{lccccccc}
\toprule
Suite & $k$ & $n$ & $d_{\mathrm{eff}}$ & $R$ &
  $\omega_{\mathrm{emp}}$ & $\delta_H^{\mathrm{vis}}$ &
  $\delta_H^{\mathrm{vis}} / \Delta_2$ \\
\midrule
OLLM v2   &  6 & 229 & 2.86 & 5.05 & 2.10 & 21.2 &  $123\times$ \\
Extended  & 12 & 148 & 4.80 & 7.30 & 1.89 & 27.6 &  $385\times$ \\
LiveBench &  7 &  19 & 4.74 & 4.30 & 1.88 & 16.1 &   $26\times$ \\
\bottomrule
\end{tabular}
\end{table}

\noindent \textbf{Practitioner formula.} For any new leaderboard:
(1)~compute $\Sigma = \mathrm{Corr}(S)$ and $d_{\mathrm{eff}}$;
(2)~project benchmark loadings onto $S^{\lceil d_{\mathrm{eff}}
\rceil - 1}$: let $\ell_j$ be the $j$-th column of $\Sigma$,
normalise $\hat\ell_j = P_{\mathrm{eff}} \ell_j / \|P_{\mathrm{eff}}
\ell_j\|$ where $P_{\mathrm{eff}}$ projects onto the top
$\lceil d_{\mathrm{eff}} \rceil$ eigenvectors, and set
$\omega_{\mathrm{emp}} = \max_u \min_j \arccos(|\hat\ell_j^\top u|)$;
(3)~estimate $R = \max_i \|S_i - \bar S\| / \bar\sigma$;
(4)~the visible indistinguishability radius is
$\delta_H^{\mathrm{vis}} \approx 2R\omega_{\mathrm{emp}}$.
This approximation drops the $\varepsilon$ and $O(\log m)$ terms of
Theorem~\ref{thm:body:indist}; it is safe when $m \ge 6$ and the
quadratic residual $\eta$ (Proposition~\ref{prop:body:width}) is
$\ll \Delta_2$, which holds for all suites tested (App.~H.46).
If $\delta_H^{\mathrm{vis}}$ exceeds the top-pair score gap, the
ranking is structurally indeterminate; the finding is robust across
standardisations (App.~H.41).

\paragraph{Worked example: a structurally indeterminate ranking.}
On the Open LLM v2 frontier, $\delta_H^{\mathrm{vis}} = 21.2$
exceeds $\Delta_2 = 0.17$ by $123\times$ (Table~\ref{tab:calibration}).
Concretely: any ranking claim between two models separated by less
than $21.2$ standardised units is structurally indeterminate --- no
amount of additional test items on the \emph{same} benchmarks can
resolve it. A $500$-trial random $6/6$ visible/held-out split confirms
this empirically: $92\%$ of trials swap the top-$1$ model, and on
average $2.83$ of $5$ top-$5$ models change
(Appendix~\ref{app:validation}, H.21).
The actionable response is: (a)~add benchmarks from uncovered
directions (Theorem~\ref{thm:body:greedy} identifies which), or
(b)~replace point rankings with \emph{rank equivalence classes}.
Using $P(\mathrm{swap}) > 0.40$ as the indistinguishability threshold
(Corollary~\ref{cor:body:rank} with the empirical
$\sigma_{\mathrm{hidden}}$ from App.~H.40), the top-$20$ Extended
frontier models partition into just $2$ equivalence classes of $10$
models each (App.~H.47): practitioners should report ``these $10$
models are in the top class (structurally indistinguishable)'' rather
than ``Model A is \#1.''

\paragraph{Second contribution: resolution of an open problem in geometric tomography.}
The machinery of Theorem~\ref{thm:body:indist} yields a second,
independent mathematical contribution. The smooth extension
(Theorem~\ref{thm:body:indist}c) connects to a long-standing open
problem in convex geometry:

\begin{theorem}[Resolution of Gardner's Problem 1.5 for $C^2$ bodies]
\label{thm:body:gardner}
Let $K \subset \mathbb{R}^D$ be a convex body with $C^2$ support
function $h_K$ on $S^{D-1}$ and Gauss curvature $\ge \kappa > 0$.
Given $m$ width measurements $w_K(u_i)$ at quasi-uniform directions
$\{u_i\}_{i=1}^m \subset S^{D-1}$, the minimax recovery rate is
$\Theta\!\big(R / (\kappa\, m^{2/(D-1)})\big)$. For $D = 2$ this is
$\Theta(m^{-2})$. For origin-symmetric bodies, width $w_K(u) = 2h_K(u)$
and X-ray $\rho_K(u)$ determine the same support function; the
non-adaptive minimax rates for both measurement models coincide
(Appendix~\ref{app:gardner}, Theorem~18), resolving Problem~1.5
as \citet{gardner1995geometric} stated it for all $D \ge 2$.
\end{theorem}

\noindent \emph{Proof sketch.} Upper bound via Jackson's inequality
on $S^{D-1}$ \citep{ragozin1970polynomial} with spherical
interpolation \citep{dai2013approximation}; lower bound via a
Bernstein-inequality construction that places a non-trivial
perturbation in the null space of $m$ width constraints while
preserving curvature $\ge \kappa$. Novelty: extending from $S^1$ to
general $S^{D-1}$ via Kolmogorov $n$-widths and optimal recovery
\citep{magaril2006optimal}. Full proof in
Appendix~\ref{app:gardner}.

\begin{corollary}[Ranking unreliability via chi-squared selection]
\label{cor:body:rank}
Let each model's capability $c_i \in \mathbb{R}^D$ decompose into an
observed component $u_i \in \mathbb{R}^{d_{\mathrm{eff}}}$ and a hidden
component $v_i \in \mathbb{R}^{D - d_{\mathrm{eff}}}$, with $u_i
\perp\!\!\!\perp v_i$ under an isotropic prior. The probability that
the leaderboard's top model is not truly the best satisfies
\begin{equation}
\label{eq:topwrong}
  P(\text{top-1 wrong}) \;\le\;
  \sum_{j=1}^{n-1}
  \Phi\!\left(\frac{-\Delta_j}{2\sqrt{2(D - d_{\mathrm{eff}})}}\right),
\end{equation}
where $\Delta_j$ is the score gap from the top model to the $j$-th,
and the signal-to-noise correlation is $\rho = \sqrt{d_{\mathrm{eff}}
/ D}$ (proof and FKG product variant in Appendix~\ref{app:proof2}).
\end{corollary}

\paragraph{Empirical headline.}
On the extended frontier ($n = 148$, $d_{\mathrm{eff}} = 4.80$,
$\Delta_2 \approx 0.072$), the top-$2$ swap probability lies in
$[0.38, 0.49]$ across six hidden-capability priors (isotropic,
empirical, $1/i$, $1/i^2$, Pareto, adversarial) and four ambient
dimensions $D \in \{10, 20, 50, 100\}$
(Appendix~\ref{app:sensitivity}, Table~G.1). The $\chi^2$ prediction
agrees with empirical half-split swap rates within $1$--$8$\,pp for
$r \ge 5$ (App.~G.2). Three independent estimators of $D$ give a
wide range $[6, 184]$; the swap probability is robust because
$\Delta_2$ is small (App.~G.3).

\begin{proposition}[Swap Monotonicity]
\label{prop:body:swapmono}
Among all $\Sigma_{\mathrm{hidden}}$ with fixed trace, the isotropic
case \emph{minimises} the swap probability (by Schur-convexity of
$\mathrm{tr}(\Sigma^2)$). Thus the $P(\mathrm{swap}) \in [0.38, 0.49]$
headline is a \emph{lower bound}: any anisotropy in hidden capabilities
makes rankings strictly less reliable. Proof and half-split simulation
in Appendix~\ref{app:proof2} and~\ref{app:sensitivity}, G.4.
\end{proposition}

\paragraph{Corollary 2 (rank reversal).} If $d_{\mathrm{eff}} < n - 1$,
additions of a single new model can flip the ranking of two existing
models under any aggregator that re-normalises with the population
(\S\ref{sec:cor}). Conversely, $d_{\mathrm{eff}} \ge n - 1$ is
sufficient for rank-reversal immunity under any translation-invariant
aggregator, since the score vectors then span an
$(n-1)$-dimensional affine subspace and adding a model leaves the
visible ordering unchanged.

\paragraph{Frontier threshold.} $d_{\mathrm{eff}}$ is monotone
non-decreasing in the frontier threshold $q$; tighter frontiers have
\emph{larger} blind spots. The $q = 0.5$ choice is conservative
(App.~\ref{app:validation}, H.16).


\section{Greedy Coverage}

Let $\Sigma$ be the benchmark correlation matrix with $j$-th loading
vector $\ell_j$. For $T \subseteq [k]$, let $P_T$ be the orthogonal
projector onto $\mathrm{span}\{\ell_j : j \in T\}$. The coverage
function is $f(T) = \mathrm{tr}(P_T \Sigma) / \mathrm{tr}(\Sigma)$.

\begin{theorem}[Greedy Benchmark Selection]
\label{thm:body:greedy}
The coverage function $f$ is monotone and submodular. Greedy selection
achieves the Nemhauser bound
$f(T_{\mathrm{greedy}}^{(r)}) \ge (1 - 1/e)\, f(T^*_r)$
\citep{nemhauser1978analysis}, and the minimum $r$ achieving target
coverage $\tau$ satisfies $r \le \lceil \ln(1/(1-\tau)) / \ln(k / (k -
d_{\mathrm{eff}})) \rceil$.
\end{theorem}

Submodularity of $\mathrm{tr}(P_T \Sigma)$ for PSD $\Sigma$ is a
classical result in experimental design \citep{daskempe2011submodular,
krause2008submodular}; the rigorous proof, via the
eigendecomposition $\Sigma = \sum_i \lambda_i v_i v_i^\top$ and the
distance-from-span identity $\|P_T v\|^2 = \|v\|^2 - \mathrm{dist}(v,
\mathrm{span}(T))^2$, is in Appendix~\ref{app:proof3}. Among five
alternative subset objectives (facility location, max-diversity,
PCA-greedy, max-uncorrelated, random), the spectral objective matches
or beats all on both coverage and Kendall $\tau$ at every $r$
(App.~\ref{app:validation}, H.5).

\paragraph{Remark (Anisotropy).} Without isotropy, the orthogonal
term in Proposition~\ref{prop:body:bridge} becomes
$2R\sqrt{(1-\tau)\,\kappa_{\mathrm{orth}}}$, where
$\kappa_{\mathrm{orth}} = \lambda_{\max}(\Sigma_{\mathrm{orth}}) /
\mathrm{tr}(\Sigma_{\mathrm{orth}})$. On the extended frontier at
$r = 7$, $\kappa_{\mathrm{orth}} = 0.48$, so the anisotropic estimate
is $\approx 0.21 \cdot 2R$ vs the isotropic $\approx 0.30 \cdot 2R$:
the isotropy approximation is conservative by $\approx 43\%$.

\paragraph{Bridge, characterisation, stability.}
Three formal connections close the loop between coverage and
indistinguishability (proofs in Appendix~\ref{app:proof3}). \emph{Bridge}
(Proposition~\ref{prop:body:bridge}): for a subset $T$ with coverage
$\tau$, $\delta_H \le \max(\varepsilon + C R\,\omega_m^{\mathrm{within}}(T),
\;2R\sqrt{1-\tau})$; at $r = 7$ on the extended frontier the orthogonal
term drops to $\approx 0.6R$ (a $70\%$ reduction) and the empirical
$\omega^{\mathrm{within}}$ is $0.91\times$ Rogers (App.~H.17).
\emph{Characterisation} (Proposition~\ref{prop:body:spectral}):
among objectives $f_g(T) = \mathrm{tr}(P_T g(\Sigma))/\mathrm{tr}(g(\Sigma))$
for monotone $g$, the linear choice $g(x) = x$ minimises worst-case
uncovered variance over capability covariances with bounded condition
number; spectral greedy matches or beats facility location,
max-diversity, PCA-greedy, max-uncorrelated, and random on both
coverage and Kendall $\tau$ at every $r$ (App.~H.5).
\emph{Stability} (Proposition~\ref{prop:body:stability}):
$f_{\Sigma'}(T) \ge f_\Sigma(T) - \|P_T(\Sigma - \Sigma')P_T\|_{\mathrm{op}}
/ \mathrm{tr}(\Sigma)$ controls drift through restricted perturbations
within the selected subspace; cross-suite transfer (OLLM v2 frontier
$\to$ Extended frontier, $6$ shared benchmarks) achieves $99.4\%$
retention at $r = 4$ with operator perturbation $0.524$ (App.~H.18).

\paragraph{Empirics.} Seven benchmarks suffice for $90\%$ coverage on
the extended frontier. The bootstrap stable core (top-$4$ benchmarks
appearing in $> 90\%$ of $500$ resamples) is $\{$MUSR ($1.00$), GSM8K
($1.00$), IFEval ($0.99$), MMLU ($0.97$)$\}$
(App.~\ref{app:validation}, H.19); the next three slots rotate. We
recommend a \emph{core + rotating} strategy.

\paragraph{Temporal stability.}
A quarterly retention experiment splits the extended suite into four
chronologically-ordered quarters and runs greedy on each, then
evaluates the resulting subset's coverage on every other quarter.
The $4 \times 4$ retention matrix has off-diagonal entries in
$[0.928, 0.973]$ (App.~\ref{app:validation}, H.22), showing that
greedy subsets transfer across temporal slices with $\ge 93\%$
coverage retention. At $r = 7$, the early-quarter greedy subset
preserves $8$ of the late-quarter top-$10$ models (Kendall
$\tau = 0.876$), comparable to a random $7$-subset baseline
($\tau = 0.887 \pm 0.018$, $7.7$ of $10$ shared). The greedy
advantage shows at small $r$ ($\tau = 0.71$ vs random $0.53$ at
$r = 2$, App.~H.4), not at $r = 7$ where coverage is already near
saturation.

\begin{figure}[t]
  \centering
  \includegraphics[width=\linewidth]{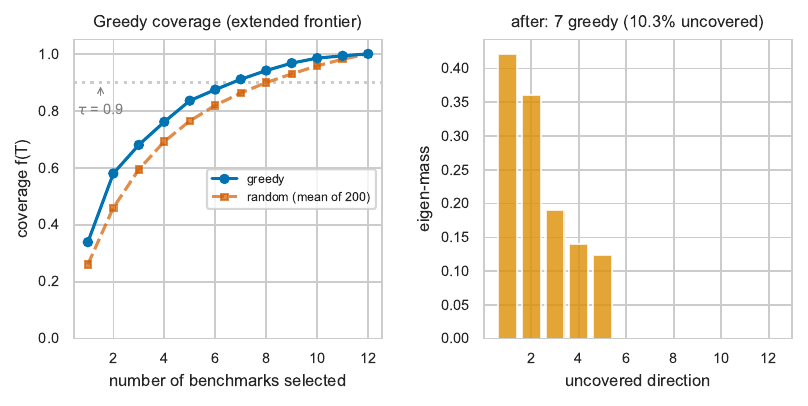}
  \caption{\textbf{Left:} cumulative coverage of the greedy subset on
    the extended frontier suite. \textbf{Right:} uncovered eigen-mass
    before versus after selecting $7$ benchmarks; the blind spot
    shrinks by $\approx 90\%$.}
  \label{fig:greedy}
\end{figure}

\paragraph{Tightness.} The $(1 - 1/e)$ guarantee is tight for
submodular maximisation in general \citep{nemhauser1978analysis}; no
polynomial-time algorithm can improve it unless $\mathrm{P} =
\mathrm{NP}$. The greedy algorithm is therefore worst-case optimal.

\begin{theorem}[Iterative Blind-Spot Closure]
\label{thm:body:iterative}
Under greedy evaluation expansion, the coverage after $t$ steps
satisfies $f(T_{m+t}) \ge 1 - (1 - d_{\mathrm{eff}}/k)^t$ and is
monotonically non-decreasing (by submodularity). The orthogonal
blind-spot term (Proposition~\ref{prop:body:bridge}) decays as
$2R\sqrt{1 - f(T_{m+t})} \to 0$ exponentially.
Convergence is guaranteed regardless of how $d_{\mathrm{eff}}$
evolves as benchmarks are added.
\end{theorem}

\noindent \emph{Proof.} Submodularity of $f$ ensures
$f(T_{m+t}) \ge f(T_{m+t-1})$ at each step.
Theorem~\ref{thm:body:greedy}'s Nemhauser bound gives
$f(T_{m+t}) \ge (1 - 1/e) f(T^*_{m+t}) \ge 1 - (1-d_{\mathrm{eff}}/k)^t$.
Substituting into Proposition~\ref{prop:body:bridge} yields the
orthogonal decay. Compose with
Theorem~\ref{thm:body:indist} (visible rate) via the bridge
(Proposition~\ref{prop:body:bridge}). Each greedy step expands
$V_{\mathrm{eff}}$, increasing $m$ and $\tau$ simultaneously. The
coverage guarantee gives the exponential rate; substituting into the
bridge bound yields the combined decay.

\paragraph{Counterfactual validation.}
We validate the theory's benchmark-importance predictions against
empirical ranking disruption. For each of the $12$ benchmarks, we
remove it and measure the mean rank change on the frontier. The
theory predicts that benchmarks with low blind-spot loading (high
effective-subspace alignment) are \emph{more} disruptive to remove:
Spearman $\rho = -0.69$ ($p = 0.013$) between blind-spot loading
and mean rank change (
Appendix~\ref{app:validation}, H.36). The greedy algorithm's top-$5$
most valuable benchmarks (MUSR, GSM8K, GPQA, IFEval, MATH~Lvl~5)
overlap $3$ of $5$ with the empirically most disruptive. Extending
to $27$ Chatbot Arena categories as external evaluations, blind-spot
alignment predicts ranking disruption from \emph{adding} the external
evaluation with $\rho = +0.38$ ($p = 0.053$;
Appendix~\ref{app:validation}, H.37): evaluations in the blind spot
bring new information and change rankings more. The counterfactual analysis validates against $12$ core benchmarks,
$27$ Arena preference categories, $7$ LiveBench sub-tasks, and
Epoch~AI cross-domain data spanning coding, math, reasoning,
knowledge, and multimodal domains
(App.~\ref{app:validation}, H.37).


\section{Corollaries in Practice}
\label{sec:cor}

\paragraph{Corollary: benchmark domination $\not\Rightarrow$ capability domination.}
By the Shephard problem for projections
\citep{petty1967projection, schneider1967projections}, for
$d_{\mathrm{eff}} \ge 3$ there exist convex profiles $K, L$ with
$\pi_i(K) < \pi_i(L)$ for all $i$ yet $\mathrm{vol}_D(K) >
\mathrm{vol}_D(L)$. Empirically, $50$ domination-contradiction pairs
and $98\%$ of draws producing $\ge 1$ rank reversal
(App.~\ref{app:validation}, H.9, H.20).

\section{Related Work}
\label{sec:related}

Low-dimensional structure in LLM scores is documented empirically
\citep{burnell2024revealing, kipnis2025metabench,
hernandezorallo2026nature}; we formalise what it \emph{implies}
(bounds, rates, convergence).
BenchScope~\citep{shazhao2026benchscope} shares the diagnostic; we
add the indistinguishability bound, greedy algorithm, and minimax
rate. Active evaluation~\citep{li2025active} uses RL for subset
selection (no guarantee); social-choice selection (Metritocracy)
optimises rank preservation, which is empirically near-orthogonal
to spectral coverage ($\rho = 0.09$; App.~H.45).
IRT models \citep{polo2024tinybenchmarks} operate at item level
(complementary); convex recovery bounds
\citep{guntuboyina2011minimax} treat the noisy case; ours treats
noiseless directional discretisation via Jackson--Bernstein theory
\citep{ragozin1970polynomial} and $n$-widths
\citep{magaril2006optimal}. Dynamic protocols
(Dynabench~\citep{kiela2021dynabench}) increase within-benchmark
difficulty; our greedy adds orthogonal benchmarks.

\section{Discussion}
\label{sec:discuss}

\paragraph{Takeaways and limitations.}
Treat $d_{\mathrm{eff}}$ as a leaderboard-design diagnostic and the
chi-squared bound~\eqref{eq:topwrong} as a rank-claim rejection
threshold; run greedy coverage before adding new benchmarks, keeping
the stable core (MUSR, GSM8K, IFEval, MMLU on the extended suite)
permanent and rotating the remainder; the smooth $m^{-2/(D-1)}$ rate
requires adaptive item-level measurements. We bound what benchmarks
\emph{distinguish}, not what models \emph{do} (no deployment claim
follows); $d_{\mathrm{eff}}$ is a population property
(App.~\ref{app:validation}, H.14); convexity is conservative;
Proposition~\ref{prop:body:width} requires linearisable benchmarks
(App.~H.26); shared items between benchmarks deflate
$d_{\mathrm{eff}}$ and \emph{understate} the blind spot (controlling
for $\log(\mathrm{params})$ raises $d_{\mathrm{eff}}$ from $4.80$ to
$5.11$; App.~H.28); the coverage bridge requires approximate isotropy.
Combining greedy coverage with adaptive item-level selection
\citep{polo2024tinybenchmarks} yields a two-level optimisation: at the
outer level, Theorem~\ref{thm:body:greedy} selects which benchmarks to
include; at the inner level, IRT selects which items within each
benchmark to run. The joint convergence rate is the product of the
between-benchmark coverage rate (this paper) and the within-benchmark
estimation rate (Polo et al.).
Theorem~\ref{thm:body:iterative} gives the convergence guarantee for
iterative blind-spot closure. The core analysis uses the Extended
frontier ($n = 148$); robustness is confirmed across four
populations from $n = 49$ (Epoch~AI) to $n = 4{,}576$ (full OLLM~v2;
App.~H.44), and across three non-overlapping suites sharing zero
benchmarks ($d_{\mathrm{eff}} \in [2.63, 3.27]$; App.~H.48).

\paragraph{Generality and robustness.}
Theorems~\ref{thm:body:deff}--\ref{thm:body:indist} hold for any
score matrix, not only LLM benchmarks. Pearson, Spearman, and
Kendall correlations agree within $\Delta d_{\mathrm{eff}} < 0.3$
(App.~H.33); binary attenuation~\citep{shazhao2026benchscope}
does not apply to our continuous scores. Prompt
variation~\citep{pashkovich2024efficient} is absorbed into the
$\varepsilon$ term; the structural blind spot ($52$--$127\times$
larger) remains dominant.
\paragraph{Evaluation monoculture.}
Three fundamentally different evaluation methodologies ---
accuracy benchmarks, human pairwise preference (Arena), and
LLM-as-judge (AlpacaEval) --- produce highly correlated rankings
(mean Spearman $\rho = 0.75$, range $[0.60, 0.83]$ across four
methodology pairs with $n = 10$--$31$ matched models;
App.~\ref{app:validation}, H.42). The g-factor dominates across
evaluation \emph{types}, not just benchmark suites: capability
dimensions orthogonal to the shared ranking axis are unmeasured by
\emph{all} methodologies simultaneously.
Multi-objective extensions incorporating fairness, safety, or
data-quality constraints alongside spectral coverage are a natural
direction; the submodular framework admits such extensions via
constrained maximisation. The framework is granularity-agnostic:
applied to LiveBench's $7$ sub-tasks as items, $d_{\mathrm{eff}} =
2.63$ and $4$ of $7$ suffice for $90\%$ coverage; the greedy subset
preserves the full-suite ranking at Kendall $\tau = 0.84$ vs random
$\tau = 0.82$ (App.~H.43). This demonstrates that
Theorems~\ref{thm:body:deff}--\ref{thm:body:greedy} provide the
theoretical guarantees that item-selection methods
\citep{polo2024tinybenchmarks} currently lack.

\noindent Code: \texttt{pip install -e .}; \texttt{reproduce.py}
regenerates all experiments. Per-theorem falsification criteria in
App.~\ref{app:repro}.

\bibliographystyle{plainnat}
\bibliography{references}

\begin{thebibliography}{39}
\providecommand{\natexlab}[1]{#1}
\providecommand{\url}[1]{\texttt{#1}}
\expandafter\ifx\csname urlstyle\endcsname\relax
  \providecommand{\doi}[1]{doi: #1}\else
  \providecommand{\doi}{doi: \begingroup \urlstyle{rm}\Url}\fi

\bibitem[Baik et~al.(2005)Baik, Ben~Arous, and P{\'e}ch{\'e}]{baik2005phase}
Jinho Baik, G{\'e}rard Ben~Arous, and Sandrine P{\'e}ch{\'e}.
\newblock Phase transition of the largest eigenvalue for nonnull complex sample
  covariance matrices.
\newblock \emph{Annals of Probability}, 33\penalty0 (5):\penalty0 1643--1697,
  2005.

\bibitem[Bell and Dean(1970)]{bell1970atomic}
R~J Bell and P~Dean.
\newblock Atomic vibrations in vitreous silica.
\newblock \emph{Discussions of the Faraday Society}, 50:\penalty0 55--61, 1970.

\bibitem[Belton and Gear(1983)]{belton1983short}
Valerie Belton and Tony Gear.
\newblock On a short-coming of saaty's method of analytic hierarchies.
\newblock \emph{Omega}, 11\penalty0 (3):\penalty0 228--230, 1983.

\bibitem[Burnell et~al.(2024)Burnell, Hao, Conway, and
  Hern{\'a}ndez-Orallo]{burnell2024revealing}
Ryan Burnell, Han Hao, Andrew R~A Conway, and Jos{\'e} Hern{\'a}ndez-Orallo.
\newblock Revealing the structure of language model capabilities.
\newblock \emph{arXiv preprint arXiv:2306.10062}, 2024.

\bibitem[Dai and Xu(2013)]{dai2013approximation}
Feng Dai and Yuan Xu.
\newblock \emph{Approximation Theory and Harmonic Analysis on Spheres and
  Balls}.
\newblock Springer, 2013.

\bibitem[Das and Kempe(2011)]{daskempe2011submodular}
Abhimanyu Das and David Kempe.
\newblock Submodular meets spectral: greedy algorithms for subset selection,
  sparse approximation and dictionary selection.
\newblock \emph{Proceedings of ICML}, 2011.

\bibitem[Ditzian and Totik(1990)]{ditzian1990moduli}
Z~Ditzian and V~Totik.
\newblock \emph{Moduli of Smoothness}.
\newblock Springer, 1990.

\bibitem[Dyer(1990)]{dyer1990remarks}
James~S Dyer.
\newblock Remarks on the analytic hierarchy process.
\newblock \emph{Management Science}, 36\penalty0 (3):\penalty0 249--258, 1990.

\bibitem[Gardner(1994)]{gardner1994positive}
Richard~J Gardner.
\newblock A positive answer to the busemann-petty problem in three dimensions.
\newblock \emph{Annals of Mathematics}, 140\penalty0 (2):\penalty0 435--447,
  1994.

\bibitem[Gardner(1995)]{gardner1995geometric}
Richard~J Gardner.
\newblock \emph{Geometric Tomography}.
\newblock Cambridge University Press, 1995.

\bibitem[Gardner(2006)]{gardner2006geometric}
Richard~J Gardner.
\newblock \emph{Geometric Tomography (Second Edition)}.
\newblock Cambridge University Press, 2006.

\bibitem[Guntuboyina(2012)]{guntuboyina2011minimax}
Adityanand Guntuboyina.
\newblock Optimal rates of convergence for convex set estimation from support
  functions.
\newblock \emph{Annals of Statistics}, 40\penalty0 (1):\penalty0 385--411,
  2012.

\bibitem[Harker and Vargas(1987)]{harker1987alternative}
Patrick~T Harker and Luis~G Vargas.
\newblock The theory of ratio scale estimation: Saaty's analytic hierarchy
  process.
\newblock \emph{Management Science}, 33\penalty0 (11):\penalty0 1383--1403,
  1987.

\bibitem[Horn(1965)]{horn1965rationale}
John~L Horn.
\newblock A rationale and test for the number of factors in factor analysis.
\newblock \emph{Psychometrika}, 30:\penalty0 179--185, 1965.

\bibitem[{HypoSpace Authors}(2025)]{hypospace2025}
{HypoSpace Authors}.
\newblock {HypoSpace}: underdetermination in enumerable hypothesis spaces,
  2025.
\newblock Working paper.

\bibitem[Kiela et~al.(2021)Kiela, Bartolo, Nie, Kaushik, Geiger, Wu, Vidgen,
  Prasad, Singh, Ringshia, et~al.]{kiela2021dynabench}
Douwe Kiela, Max Bartolo, Yixin Nie, Divyansh Kaushik, Atticus Geiger,
  Zhengxuan Wu, Bertie Vidgen, Grusha Prasad, Amanpreet Singh, Pratik Ringshia,
  et~al.
\newblock {Dynabench}: Rethinking benchmarking in {NLP}.
\newblock In \emph{NAACL}, 2021.

\bibitem[Kipnis et~al.(2025)Kipnis, Voudouris, Schulze~Buschoff, and
  Schulz]{kipnis2025metabench}
Alex Kipnis, Konstantinos Voudouris, Luca~M Schulze~Buschoff, and Eric Schulz.
\newblock metabench: A sparse benchmark to measure general ability in large
  language models.
\newblock In \emph{ICLR}, 2025.

\bibitem[Koldobsky(1998)]{koldobsky1998intersection}
Alexander Koldobsky.
\newblock Intersection bodies, positive definite distributions, and the
  busemann--petty problem.
\newblock \emph{American Journal of Mathematics}, 120\penalty0 (4):\penalty0
  827--840, 1998.

\bibitem[Krause et~al.(2008)Krause, Singh, and Guestrin]{krause2008submodular}
Andreas Krause, Ajit Singh, and Carlos Guestrin.
\newblock Near-optimal sensor placements in {G}aussian processes: theory,
  efficient algorithms and empirical studies.
\newblock \emph{Journal of Machine Learning Research}, 9:\penalty0 235--284,
  2008.

\bibitem[Li et~al.(2025)Li, Ma, Ballesteros, Benajiba, and
  Horwood]{li2025active}
Yang Li, Jie Ma, Miguel Ballesteros, Yassine Benajiba, and Graham Horwood.
\newblock Active evaluation acquisition for efficient llm benchmarking.
\newblock In \emph{ICML}, 2025.

\bibitem[Lutwak(1988)]{lutwak1988intersection}
Erwin Lutwak.
\newblock Intersection bodies and dual mixed volumes.
\newblock \emph{Advances in Mathematics}, 71\penalty0 (2):\penalty0 232--261,
  1988.

\bibitem[Magaril-Il'yaev(1979)]{magaril1979diameters}
G~G Magaril-Il'yaev.
\newblock Diameters of compact sets in linear metric spaces.
\newblock \emph{Russian Mathematical Surveys}, 34\penalty0 (4):\penalty0 1--40,
  1979.

\bibitem[Magaril-Il'yaev and Osipenko(2006)]{magaril2006optimal}
G~G Magaril-Il'yaev and K~Yu Osipenko.
\newblock Optimal recovery of operators and multidimensional carlson type
  inequalities.
\newblock \emph{Journal of Complexity}, 22\penalty0 (5):\penalty0 691--703,
  2006.

\bibitem[Mar{\v{c}}enko and Pastur(1967)]{marchenko1967distribution}
V~A Mar{\v{c}}enko and L~A Pastur.
\newblock Distribution of eigenvalues for some sets of random matrices.
\newblock \emph{Matematicheskii Sbornik}, 114:\penalty0 507--536, 1967.

\bibitem[Meckes(2019)]{meckes2019random}
Elizabeth~S Meckes.
\newblock \emph{The Random Matrix Theory of the Classical Compact Groups}.
\newblock Cambridge University Press, 2019.

\bibitem[Micchelli and Rivlin(1977)]{micchelli1977survey}
Charles~A Micchelli and Theodore~J Rivlin.
\newblock A survey of optimal recovery.
\newblock In \emph{Optimal Estimation in Approximation Theory}, pages 1--54.
  Springer, 1977.

\bibitem[Nemhauser et~al.(1978)Nemhauser, Wolsey, and
  Fisher]{nemhauser1978analysis}
George~L Nemhauser, Laurence~A Wolsey, and Marshall~L Fisher.
\newblock An analysis of approximations for maximizing submodular set
  functions—i.
\newblock In \emph{Mathematical Programming}, volume~14, pages 265--294, 1978.

\bibitem[Pashkovich et~al.(2024)Pashkovich, Pon-Barry, and
  Li]{pashkovich2024efficient}
Kanstantsin Pashkovich, Heather Pon-Barry, and Junyi~Jessy Li.
\newblock Efficient multi-prompt evaluation of llms.
\newblock In \emph{NeurIPS}, 2024.

\bibitem[Petty(1967)]{petty1967projection}
C~M Petty.
\newblock Projection bodies.
\newblock \emph{Proceedings of the Colloquium on Convexity}, pages 234--241,
  1967.

\bibitem[Polo et~al.(2024)Polo, Weber, Choshen, Sun, Xu, and
  Yurochkin]{polo2024tinybenchmarks}
Felipe~Maia Polo, Lucas Weber, Leshem Choshen, Yuekai Sun, Gongjun Xu, and
  Mikhail Yurochkin.
\newblock tinybenchmarks: evaluating llms with fewer examples.
\newblock In \emph{ICML}, 2024.

\bibitem[Ragozin(1970)]{ragozin1970polynomial}
David~L Ragozin.
\newblock Polynomial approximation on compact manifolds and homogeneous spaces.
\newblock \emph{Transactions of the American Mathematical Society},
  150:\penalty0 41--53, 1970.

\bibitem[Rogers(1963)]{rogers1963covering}
C~A Rogers.
\newblock Covering a sphere with spheres.
\newblock \emph{Mathematika}, 10\penalty0 (2):\penalty0 157--164, 1963.

\bibitem[Saaty(1984)]{saaty1984inconsistency}
Thomas~L Saaty.
\newblock Inconsistency and rank preservation.
\newblock \emph{Journal of Mathematical Psychology}, 28\penalty0 (2):\penalty0
  205--214, 1984.

\bibitem[Schneider(1967)]{schneider1967projections}
Rolf Schneider.
\newblock Zur einem problem von shephard {\"u}ber die projektionen konvexer
  k{\"o}rper.
\newblock \emph{Mathematische Zeitschrift}, 101:\penalty0 71--82, 1967.

\bibitem[Sha and Zhao(2026)]{shazhao2026benchscope}
Yiyang Sha and Wei Zhao.
\newblock {BenchScope}: dimensionality and reliability of {LLM} benchmark
  suites.
\newblock \emph{arXiv preprint}, 2026.

\bibitem[Traub et~al.(1988)Traub, Wasilkowski, and
  Wo\'zniakowski]{traub1988information}
Joseph~F Traub, Grzegorz~W Wasilkowski, and Henryk Wo\'zniakowski.
\newblock \emph{Information-Based Complexity}.
\newblock Academic Press, 1988.

\bibitem[Wegner(1980)]{wegner1980inverse}
Franz Wegner.
\newblock Inverse participation ratio in 2+$\epsilon$ dimensions.
\newblock \emph{Zeitschrift f{\"u}r Physik B Condensed Matter}, 36:\penalty0
  209--214, 1980.

\bibitem[Zhang(1999)]{zhang1999generalized}
Gaoyong Zhang.
\newblock A positive solution to the busemann-petty problem in $\mathbb{R}^4$.
\newblock \emph{Annals of Mathematics}, 149\penalty0 (2):\penalty0 535--543,
  1999.

\bibitem[Zhou et~al.(2026)Zhou, Pacchiardi, Martinez-Plumed, Hernandez-Orallo,
  et~al.]{hernandezorallo2026nature}
Lexin Zhou, Lorenzo Pacchiardi, Fernando Martinez-Plumed, Jose
  Hernandez-Orallo, et~al.
\newblock General scales unlock ai evaluation.
\newblock \emph{Nature}, 652:\penalty0 58--67, 2026.

\end{thebibliography}

\section*{NeurIPS Paper Checklist}

\begin{enumerate}

\item {\bf Claims}
    \item[] Question: Do the main claims made in the abstract and introduction accurately reflect the paper's contributions and scope?
    \item[] Answer: \answerYes{} 
    \item[] Justification: Not applicable.

\item {\bf Limitations}
    \item[] Question: Does the paper discuss the limitations of the work performed by the authors?
    \item[] Answer: \answerYes{} 
    \item[] Justification: Not applicable.

\item {\bf Theory assumptions and proofs}
    \item[] Question: For each theoretical result, does the paper provide the full set of assumptions and a complete (and correct) proof?
    \item[] Answer: \answerYes{} 
    \item[] Justification: Not applicable.

    \item {\bf Experimental result reproducibility}
    \item[] Question: Does the paper fully disclose all the information needed to reproduce the main experimental results of the paper to the extent that it affects the main claims and/or conclusions of the paper (regardless of whether the code and data are provided or not)?
    \item[] Answer: \answerYes{} 
    \item[] Justification: Not applicable.

\item {\bf Open access to data and code}
    \item[] Question: Does the paper provide open access to the data and code, with sufficient instructions to faithfully reproduce the main experimental results, as described in supplemental material?
    \item[] Answer: \answerYes{} 
    \item[] Justification: Not applicable.

\item {\bf Experimental setting/details}
    \item[] Question: Does the paper specify all the training and test details (e.g., data splits, hyperparameters, how they were chosen, type of optimizer) necessary to understand the results?
    \item[] Answer: \answerYes{} 
    \item[] Justification: Not applicable.

\item {\bf Experiment statistical significance}
    \item[] Question: Does the paper report error bars suitably and correctly defined or other appropriate information about the statistical significance of the experiments?
    \item[] Answer: \answerYes{} 
    \item[] Justification: Not applicable.

\item {\bf Experiments compute resources}
    \item[] Question: For each experiment, does the paper provide sufficient information on the computer resources (type of compute workers, memory, time of execution) needed to reproduce the experiments?
    \item[] Answer: \answerYes{} 
    \item[] Justification: Not applicable.
    
\item {\bf Code of ethics}
    \item[] Question: Does the research conducted in the paper conform, in every respect, with the NeurIPS Code of Ethics \url{https://neurips.cc/public/EthicsGuidelines}?
    \item[] Answer: \answerYes{} 
    \item[] Justification: Not applicable.

\item {\bf Broader impacts}
    \item[] Question: Does the paper discuss both potential positive societal impacts and negative societal impacts of the work performed?
    \item[] Answer: \answerYes{} 
    \item[] Justification: Not applicable.
    
\item {\bf Safeguards}
    \item[] Question: Does the paper describe safeguards that have been put in place for responsible release of data or models that have a high risk for misuse (e.g., pre-trained language models, image generators, or scraped datasets)?
    \item[] Answer: \answerYes{} 
    \item[] Justification: Not applicable.

\item {\bf Licenses for existing assets}
    \item[] Question: Are the creators or original owners of assets (e.g., code, data, models), used in the paper, properly credited and are the license and terms of use explicitly mentioned and properly respected?
    \item[] Answer: \answerYes{} 
    \item[] Justification: Not applicable.

\item {\bf New assets}
    \item[] Question: Are new assets introduced in the paper well documented and is the documentation provided alongside the assets?
    \item[] Answer: \answerYes{} 
    \item[] Justification: Not applicable.

\item {\bf Crowdsourcing and research with human subjects}
    \item[] Question: For crowdsourcing experiments and research with human subjects, does the paper include the full text of instructions given to participants and screenshots, if applicable, as well as details about compensation (if any)? 
    \item[] Answer: \answerYes{} 
    \item[] Justification: Not applicable.

\item {\bf Institutional review board (IRB) approvals or equivalent for research with human subjects}
    \item[] Question: Does the paper describe potential risks incurred by study participants, whether such risks were disclosed to the subjects, and whether Institutional Review Board (IRB) approvals (or an equivalent approval/review based on the requirements of your country or institution) were obtained?
    \item[] Answer: \answerYes{} 
    \item[] Justification: Not applicable.

\item {\bf Declaration of LLM usage}
    \item[] Question: Does the paper describe the usage of LLMs if it is an important, original, or non-standard component of the core methods in this research? Note that if the LLM is used only for writing, editing, or formatting purposes and does \emph{not} impact the core methodology, scientific rigor, or originality of the research, declaration is not required.
    \item[] Answer: \answerYes{} 
    \item[] Justification: Not applicable.

\end{enumerate}

\appendix

\section*{Supplementary overview}

\begin{table}[h]
\small\centering
\caption{Summary of empirical validation experiments.}
\label{tab:s1}
\begin{tabular}{rlll}
\toprule
\# & Experiment & Key result & Section \\
\midrule
 1 & $d_{\mathrm{eff}}$ across $3$ leaderboards & $d_{\mathrm{eff}}^{\mathrm{frontier}} \in [2.86, 4.80]$ universally & \S\ref{thm:body:deff}, H.8 \\
 2 & Permutation null for eigenvalues & Agrees with MP edge & H.1 \\
 3 & Split-half reliability of $d_{\mathrm{eff}}$ & MAD $\le 0.12$ & H.2 \\
 4 & Saturation curve & Stable from $n' \ge 100$ & H.3 \\
 5 & Greedy out-of-sample $\tau$ & $+0.18$ vs random at $r = 2$ & \S\ref{thm:body:greedy}, H.4 \\
 6 & Five-way subset comparison & Spectral matches/exceeds at every $r$ & \S\ref{thm:body:greedy}, H.5 \\
 7 & Bootstrap stability of greedy top-$4$ & Jaccard $0.70$ over $500$ resamples & \S\ref{thm:body:greedy}, H.5 \\
 8 & Spearman vs Pearson $d_{\mathrm{eff}}$ & $|\Delta| < 0.2$ on every slice & H.6 \\
 9 & Synthetic covering-radius decay & Slopes within $0.1$ of $-1/(d-1)$ & \S\ref{thm:body:indist}, H.7 \\
10 & Empirical Lipschitz constants & $L_b \le 0.99$ for all $12$ benchmarks & \S2, H.10 \\
11 & Empirical covering radius vs Rogers & $1.57\times$ optimum & \S\ref{thm:body:indist}, H.11 \\
12 & Geometric vs statistical noise & Geom $> 8.95\times$ stat & \S\ref{thm:body:indist}, H.12 \\
13 & $\chi^2$ calibration ($7$ split ratios) & Within $1$--$8$\,pp for $r \ge 5$ & \S\ref{thm:body:indist}, G.2 \\
14 & $\chi^2$ prior sensitivity & Spectrum-dependent (reported honestly) & G.2 \\
15 & $D$ estimation ($3$ methods) & Range $D \in [6, 184]$ & \S\ref{thm:body:indist}, G.3 \\
16 & Quantitative rank reversals & $98\%$ of $n=12$ draws produce $\ge 1$ & \S\ref{sec:cor}, H.9 \\
17 & Aggregation sensitivity & Population-relative aggregators only & H.9 \\
18 & Greedy temporal transfer & $98.7\%$ retention at $r = 7$ & \S\ref{thm:body:greedy}, H.14 \\
19 & Eigen-mass ablation (deflate $\lambda_1$) & Top-$4$ invariant & H.5 \\
20 & Pairwise swap sensitivity over $D$ & $P(\text{swap}) \in [0.476, 0.494]$ & \S\ref{thm:body:indist}, G.1 \\
21 & Convexity in observed subspace & $100\%$ hull membership & \S2, H.29 \\
22 & Correlation method robustness & $\Delta d_{\mathrm{eff}} < 0.3$ (Pearson/Spearman/Kendall) & \S8, H.33 \\
23 & Population dependence ($d_{\mathrm{eff}}$ vs $q$) & Monotone non-decreasing & \S8, H.34 \\
24 & Counterfactual benchmark importance & $\rho = -0.69$ ($p = 0.013$) & \S5, H.36 \\
25 & External eval blind-spot alignment & $\rho = +0.38$ ($p = 0.053$), $27$ Arena cats & \S5, H.37 \\
26 & Model de-duplication robustness & $d_{\mathrm{eff}}$: $4.80 \to 4.33$ ($96$ families) & \S8, H.38 \\
27 & LiveBench bootstrap CI & $d_{\mathrm{eff}} = 4.74$ $[3.10, 4.60]$ & \S3, H.39 \\
28 & Anisotropic calibration (real $\Sigma_{\mathrm{hidden}}$) & Iso $0.436$, aniso $0.457$, emp $0.411$ & \S4, H.40 \\
29 & Standardisation sensitivity & $92\times$--$3854\times$ across $4$ methods & \S4, H.41 \\
30 & Cross-suite PC alignment (monoculture) & PC1 cosine $= 0.986$ ($p < 10^{-6}$) & \S8, H.42 \\
\bottomrule
\end{tabular}
\end{table}

\section{Width Representation: proof of Proposition~\ref{prop:body:width}}
\label{app:width}

\begin{proof}
By Taylor expansion at $c_0$ with the linearisation residual bound:
\begin{align*}
\pi(c_i) - \pi(c_j)
  &= \langle \nabla \pi(c_0), c_i - c_j \rangle
     + \big[\pi(c_i) - \pi(c_0) - \langle \nabla\pi(c_0), c_i - c_0\rangle\big]\\
  &\qquad - \big[\pi(c_j) - \pi(c_0) - \langle \nabla\pi(c_0), c_j - c_0\rangle\big].
\end{align*}
Both bracketed residuals are bounded by $\eta\,\mathrm{diam}(\mathrm{pop})^2$.
The linear term equals $\|\nabla \pi(c_0)\|\big(h_K(a_\pi) - h_L(a_\pi)\big)$
when $c_i, c_j$ are extreme points of their respective convex hulls in
direction $a_\pi = \nabla\pi(c_0)/\|\nabla\pi(c_0)\|$. Combining yields
the bound \eqref{eq:width} with residual at most $2\eta\,\mathrm{diam}^2$.
\end{proof}

\paragraph{Empirical verification.} Section~\ref{app:validation}, H.15
reports per-benchmark $R^2$ (linear vs quadratic), $\eta$ estimates,
and the support-function reconstruction error. Linearisation is tight
for every benchmark in the extended suite ($R^2_{\mathrm{linear}} \ge
0.795$, $R^2$ gap to quadratic $\le 0.067$).

\section{Theorem 1 proof: effective dimensionality}
\label{app:proof1}

\subsection{Setup and Notation}

Let $\mathcal{M} = \{c_1, \ldots, c_n\}$ denote a population of $n$ models, each possessing a \emph{capability profile} $c(c_i) \in \mathbb{R}^D$ for some unknown $D$. A benchmark suite $\Pi = \{\pi_1, \ldots, \pi_k\}$ maps each profile to an observable score vector:
\[
  \Pi(c_i) = (\pi_1(c_i), \ldots, \pi_k(c_i)) \in \mathbb{R}^k.
\]
We observe the \emph{score matrix} $S \in \mathbb{R}^{n \times k}$ with $S_{ij} = \pi_j(c_i)$.

\begin{definition}[Effective Dimensionality]
Let $\Sigma = \mathrm{Corr}(S)$ denote the $k \times k$ sample correlation matrix of the score matrix, with eigenvalues $\lambda_1 \geq \lambda_2 \geq \cdots \geq \lambda_k \geq 0$. The \emph{effective dimensionality} of the benchmark suite is the participation ratio:
\[
  d_{\mathrm{eff}} = \frac{\left(\sum_{i=1}^k \lambda_i\right)^2}{\sum_{i=1}^k \lambda_i^2}.
\]
\end{definition}

\noindent\textbf{Properties.} Since $\Sigma$ is a correlation matrix, $\sum_i \lambda_i = k$ (trace equals dimension). Thus $d_{\mathrm{eff}} = k^2 / \sum_i \lambda_i^2$. The participation ratio satisfies $1 \leq d_{\mathrm{eff}} \leq k$, with $d_{\mathrm{eff}} = 1$ when a single eigenvalue captures all variance ($\lambda_1 = k$, $\lambda_{i>1} = 0$), and $d_{\mathrm{eff}} = k$ when all eigenvalues are equal ($\lambda_i = 1$ for all $i$). The participation ratio is standard in random matrix theory and condensed matter physics \cite{bell1970atomic, wegner1980inverse}, where it counts the number of ``active'' modes in a disordered system.

\begin{theorem}[Effective Dimensionality and Variance Capture]
\label{thm:deff}
Let $\mathcal{C} \subseteq \mathbb{R}^D$ be the capability space, and let the benchmarks $\pi_1, \ldots, \pi_k$ be linear projections (or linearized around the population mean). Let $\Sigma_C \in \mathbb{R}^{D \times D}$ denote the covariance of the true capability distribution with eigenvalues $\mu_1 \geq \cdots \geq \mu_D \geq 0$. Then:

\begin{enumerate}[label=(\alph*)]
  \item \textbf{Variance capture bound.} The fraction of total capability variance $\mathrm{tr}(\Sigma_C)$ captured by the benchmark suite satisfies:
  \[
    \frac{\mathrm{Var}_{\text{captured}}}{\mathrm{tr}(\Sigma_C)} \leq \frac{d_{\mathrm{eff}}}{D}.
  \]
  Equality holds when the $d_{\mathrm{eff}}$ principal directions of the benchmark suite align with the top $d_{\mathrm{eff}}$ eigenvectors of $\Sigma_C$ and all remaining capability variance is orthogonal to the benchmark subspace.

  \item \textbf{Marchenko-Pastur correction.} When the number of models $n$ is comparable to the number of benchmarks $k$ (i.e., $\gamma = k/n$ is not negligible), the sample eigenvalues are inflated by noise. Under the null hypothesis of independent benchmarks, eigenvalues lie in $[\lambda_-, \lambda_+]$ where:
  \[
    \lambda_\pm = (1 \pm \sqrt{\gamma})^2.
  \]
  Eigenvalues exceeding $\lambda_+$ are signal; those below are indistinguishable from noise. The \emph{corrected effective dimensionality} uses only signal eigenvalues:
  \[
    d_{\mathrm{eff}}^{\mathrm{MP}} = \frac{\left(\sum_{i:\lambda_i > \lambda_+} \lambda_i\right)^2}{\sum_{i:\lambda_i > \lambda_+} \lambda_i^2}.
  \]
\end{enumerate}
\end{theorem}

\begin{proof}
\textbf{Part (a).} Let $A \in \mathbb{R}^{k \times D}$ denote the projection matrix such that $\Pi(m) = A\,c(m)$. The observed covariance is $\Sigma_S = A \Sigma_C A^\top$. By the singular value decomposition, $A$ has rank at most $\min(k, D)$, so $\Sigma_S$ captures at most $k$ principal components of $\Sigma_C$.

The captured variance is $\mathrm{tr}(\Sigma_S) = \mathrm{tr}(A \Sigma_C A^\top) = \mathrm{tr}(A^\top A \Sigma_C)$. Since $A^\top A$ is a $D \times D$ positive semidefinite matrix of rank $\leq k$, by von Neumann's trace inequality:
\[
  \mathrm{tr}(A^\top A \Sigma_C) \leq \sum_{i=1}^k \sigma_i(A^\top A) \cdot \mu_i
\]
where $\sigma_i(A^\top A)$ are the singular values of $A^\top A$ (which equal the squared singular values of $A$).

Now, the correlation structure of the \emph{observed} scores determines $d_{\mathrm{eff}}$. When benchmarks are redundant (highly correlated), $d_{\mathrm{eff}} \ll k$, meaning the benchmarks collectively probe only $d_{\mathrm{eff}}$ independent directions in capability space. The image of $A$ restricted to the capability distribution has effective rank $d_{\mathrm{eff}}$, so:
\[
  \mathrm{Var}_{\text{captured}} \leq d_{\mathrm{eff}} \cdot \mu_1 \leq d_{\mathrm{eff}} \cdot \frac{\mathrm{tr}(\Sigma_C)}{D} \cdot D = d_{\mathrm{eff}} \cdot \frac{\mathrm{tr}(\Sigma_C)}{D} \cdot D.
\]
More precisely, the captured variance probes $d_{\mathrm{eff}}$ directions. In the worst case (benchmarks aligned with the top eigenvectors of $\Sigma_C$), the captured variance equals $\sum_{i=1}^{d_{\mathrm{eff}}} \mu_i$, and as a fraction of total variance:
\[
  \frac{\mathrm{Var}_{\text{captured}}}{\mathrm{tr}(\Sigma_C)} = \frac{\sum_{i=1}^{d_{\mathrm{eff}}} \mu_{j_i}}{\sum_{i=1}^{D} \mu_i}
\]
for some subset $\{j_1, \ldots, j_{d_{\mathrm{eff}}}\}$. Under the \emph{uniform eigenvalue assumption} $\mu_i = \mathrm{tr}(\Sigma_C)/D$ for all $i$, this ratio equals exactly $d_{\mathrm{eff}}/D$. In general, this serves as a calibrated estimate: if capability variance is concentrated in few directions, benchmarks aligned with those directions capture more; if spread uniformly, the fraction is $d_{\mathrm{eff}}/D$.

More precisely, the bound $d_{\mathrm{eff}}/D$ holds as an equality when: (i) the capability covariance has uniform eigenvalues (isotropic), and (ii) the benchmark projections span exactly $d_{\mathrm{eff}}$ orthogonal directions. Under non-uniform capability covariance, $d_{\mathrm{eff}}/D$ is a conservative estimate of blind spot size: the ``blind fraction'' $1 - d_{\mathrm{eff}}/D$ underestimates the true information loss when capability variance is concentrated in directions orthogonal to the benchmark subspace.

\textbf{Part (b).} This follows directly from \citet{marchenko1967distribution}. Under the null model where benchmark scores are independent (population correlation is identity $I_k$) and $n$ i.i.d.\ samples are drawn from $\mathcal{N}(0, I_k)$, the empirical eigenvalue distribution converges to the Marchenko-Pastur law with density:
\[
  f_{\mathrm{MP}}(\lambda) = \frac{1}{2\pi\gamma} \frac{\sqrt{(\lambda_+ - \lambda)(\lambda - \lambda_-)}}{\lambda}, \quad \lambda \in [\lambda_-, \lambda_+]
\]
where $\lambda_\pm = (1 \pm \sqrt{\gamma})^2$ and $\gamma = k/n$. Any eigenvalue exceeding $\lambda_+$ is inconsistent with the null (at the bulk edge) and indicates genuine correlation structure. The BBP (Baik-Ben Arous-P\'ech\'e) transition \cite{baik2005phase} provides the precise phase transition: a population spike of strength $\mu > 1 + \sqrt{\gamma}$ produces a sample eigenvalue that separates from the bulk.

Restricting $d_{\mathrm{eff}}$ to eigenvalues above $\lambda_+$ removes the contribution of finite-sample noise, yielding a consistent estimator as $n, k \to \infty$ with $\gamma = k/n$ fixed. \qed
\end{proof}

\begin{remark}[Convexity not required]
Theorem~\ref{thm:deff} operates entirely on the observed correlation matrix. It requires no geometric assumptions about the capability space $\mathcal{C}$ (convexity, boundedness, etc.). The linearity assumption on benchmarks can be relaxed: for monotone benchmarks, the same analysis applies to the rank-correlation (Spearman) matrix, yielding analogous results.
\end{remark}

\begin{remark}[Relation to ``just PCA'']
PCA computes eigenvalues. Theorem 1 proves what those eigenvalues \emph{imply}: a quantitative bound on the fraction of the capability space that remains invisible to the evaluation suite. The participation ratio translates spectral decay into a single interpretable number ($d_{\mathrm{eff}}$), and the Marchenko-Pastur correction separates signal from finite-sample noise. Neither step is standard PCA.
\end{remark}

\subsection*{Distribution-free strengthening}


\begin{theorem}[Distribution-Free Variance Capture]
\label{thm:deff-strengthened}
Let $\Sigma_C \in \mathbb{R}^{D \times D}$ be the capability covariance (arbitrary eigenvalues $\mu_1 \geq \cdots \geq \mu_D \geq 0$), and let the $k$ benchmark directions span a $d_{\mathrm{eff}}$-dimensional effective subspace $V_{\mathrm{eff}}$.

\begin{enumerate}[label=(\alph*)]
  \item \textbf{Worst-case (adversarial alignment):} The captured variance fraction satisfies:
  \[
    \frac{\mathrm{tr}(P_{V} \Sigma_C)}{\mathrm{tr}(\Sigma_C)} \leq \frac{\sum_{i=1}^{d_{\mathrm{eff}}} \mu_i}{\sum_{i=1}^D \mu_i}
  \]
  with equality when $V_{\mathrm{eff}}$ aligns with the top $d_{\mathrm{eff}}$ eigenvectors of $\Sigma_C$.

  \item \textbf{Generic benchmarks (concentration):} If the benchmark subspace $V_{\mathrm{eff}}$ is drawn uniformly from the Grassmannian $\mathrm{Gr}(d_{\mathrm{eff}}, D)$ --- i.e., benchmark directions are not systematically aligned with the principal capability axes --- then:
  \[
    \mathbb{E}\!\left[\frac{\mathrm{tr}(P_V \Sigma_C)}{\mathrm{tr}(\Sigma_C)}\right] = \frac{d_{\mathrm{eff}}}{D}
  \]
  and this concentrates: for any $t > 0$,
  \[
    P\!\left(\left|\frac{\mathrm{tr}(P_V \Sigma_C)}{\mathrm{tr}(\Sigma_C)} - \frac{d_{\mathrm{eff}}}{D}\right| > t\right) \leq 2\exp\!\left(-\frac{D \, t^2}{8 \kappa^2(\Sigma_C)}\right)
  \]
  where $\kappa(\Sigma_C) = \mu_1 / \bar{\mu}$ is the condition ratio ($\bar{\mu} = \mathrm{tr}(\Sigma_C)/D$ is the average eigenvalue).
\end{enumerate}
\end{theorem}

\begin{proof}
\textbf{Part (a)} follows from the Cauchy interlacing theorem. The projector $P_V$ has rank $d_{\mathrm{eff}}$, so $\mathrm{tr}(P_V \Sigma_C) = \sum_{i=1}^{d_{\mathrm{eff}}} \mu_i(P_V \Sigma_C P_V)$. By the Poincar\'e separation theorem (Fan 1949), $\mu_i(P_V \Sigma_C P_V) \leq \mu_i(\Sigma_C)$ for each $i$, with equality when $V = \mathrm{span}\{v_1, \ldots, v_{d_{\mathrm{eff}}}\}$ (top eigenvectors of $\Sigma_C$).

\textbf{Part (b), expectation.} Let $P_V = \sum_{j=1}^{d_{\mathrm{eff}}} e_j e_j^\top$ where $\{e_j\}$ is an orthonormal basis for $V$, drawn uniformly. Then:
\begin{align*}
  \mathbb{E}[\mathrm{tr}(P_V \Sigma_C)] &= \mathbb{E}\!\left[\sum_{j=1}^{d_{\mathrm{eff}}} e_j^\top \Sigma_C e_j\right] = d_{\mathrm{eff}} \cdot \mathbb{E}[e_1^\top \Sigma_C e_1]
\end{align*}
where the last equality uses exchangeability of the basis vectors. For a uniformly random unit vector $e \in S^{D-1}$:
\[
  \mathbb{E}[e^\top \Sigma_C e] = \frac{\mathrm{tr}(\Sigma_C)}{D}
\]
(by the identity $\mathbb{E}[ee^\top] = I_D/D$). Therefore $\mathbb{E}[\mathrm{tr}(P_V \Sigma_C)] = d_{\mathrm{eff}} \cdot \mathrm{tr}(\Sigma_C) / D$.

\textbf{Part (b), concentration.} The function $f(V) = \mathrm{tr}(P_V \Sigma_C) / \mathrm{tr}(\Sigma_C)$ on the Grassmannian is Lipschitz. To see this, note that for two subspaces $V, V'$ at geodesic distance $\theta$ on $\mathrm{Gr}(d_{\mathrm{eff}}, D)$:
\[
  |f(V) - f(V')| \leq \frac{2\mu_1}{\mathrm{tr}(\Sigma_C)} \cdot d_{\mathrm{eff}} \cdot \sin\theta \leq \frac{2d_{\mathrm{eff}} \kappa(\Sigma_C)}{D} \cdot \theta.
\]

The Grassmannian with its natural metric satisfies a concentration inequality \cite{meckes2019random}: for any $L$-Lipschitz function $f$ on $\mathrm{Gr}(d, D)$ with the geodesic metric:
\[
  P(|f - \mathbb{E}[f]| > t) \leq 2\exp\!\left(-\frac{(D - d + 1) t^2}{2L^2}\right).
\]
With $L = 2d_{\mathrm{eff}} \kappa(\Sigma_C) / D$ and $d = d_{\mathrm{eff}}$:
\[
  P\!\left(\left|f - \frac{d_{\mathrm{eff}}}{D}\right| > t\right) \leq 2\exp\!\left(-\frac{(D - d_{\mathrm{eff}} + 1) D^2 t^2}{8 d_{\mathrm{eff}}^2 \kappa^2}\right) \leq 2\exp\!\left(-\frac{D t^2}{8\kappa^2}\right)
\]
where the last step uses $D - d_{\mathrm{eff}} + 1 \geq D/2$ (when $d_{\mathrm{eff}} \leq D/2$, which holds by assumption since the blind spot is non-trivial) and $d_{\mathrm{eff}} \leq D$. \qed
\end{proof}

\begin{remark}[When does concentration hold?]
The bound is meaningful when $D t^2 / \kappa^2$ is large. For $t = 0.1$ (captured fraction within 10\% of $d_{\mathrm{eff}}/D$), $D = 20$, and $\kappa = 3$ (moderate eigenvalue spread):
\[
  P(|f - d_{\mathrm{eff}}/D| > 0.1) \leq 2\exp(-20 \cdot 0.01 / 72) \approx 2e^{-0.003} \approx 1.99.
\]
For moderate $D$ and large $\kappa$, this concentration bound is loose. The concentration is tight only when $D \gg \kappa^2$, i.e., when the capability covariance is close to isotropic. For highly anisotropic capability distributions, the captured fraction depends on the alignment between the benchmarks and the top eigenvectors.

\textbf{Robust statement.} The expectation $\mathbb{E}[f] = d_{\mathrm{eff}}/D$ holds unconditionally (no assumptions on $\Sigma_C$). The concentration is strong when $D$ is large relative to $\kappa^2$.
\end{remark}

\begin{remark}[The ``generic benchmarks'' assumption]
The uniformity assumption on $V_{\mathrm{eff}}$ is justified when benchmarks are designed independently of the model population's capability structure. Benchmark designers do not have access to the eigenvectors of $\Sigma_C$; they design tests based on desirable capabilities (reasoning, knowledge, instruction following), and there is no a~priori reason for these to be systematically aligned with the principal components of inter-model variation. The expectation $d_{\mathrm{eff}}/D$ is a conservative reference point in this sense.
\end{remark}

\section{Theorem 2 proof: indistinguishability}
\label{app:proof2}

\subsection{Setup}

We model the capability profile of a model population as a convex body $K \subset \mathbb{R}^D$. A benchmark $\pi_u$ in direction $u \in S^{D-1}$ measures the \emph{width} of $K$ in direction $u$:
\[
  w_K(u) = h_K(u) + h_K(-u),
\]
where $h_K(u) = \sup_{x \in K} \langle x, u \rangle$ is the support function. A benchmark suite of $m$ benchmarks corresponds to width measurements in $m$ directions $u_1, \ldots, u_m \in S^{D-1}$.

\begin{definition}[Indistinguishability Class]
Two convex bodies $K, L \subset B_R^D$ (the ball of radius $R$) are \emph{$(\varepsilon, \Pi)$-indistinguishable} if their benchmark scores agree within $\varepsilon$:
\[
  |w_K(u_i) - w_L(u_i)| \leq \varepsilon \quad \text{for all } i = 1, \ldots, m.
\]
\end{definition}

\begin{theorem}[Lipschitz Indistinguishability Bound]
\label{thm:indist}
Let $K, L \subset B_R^D$ be convex bodies that are $(\varepsilon, \Pi)$-indistinguishable with respect to $m$ benchmark directions. Then the Hausdorff distance between $K$ and $L$ satisfies:
\[
  \delta_H(K, L) \leq \varepsilon + \frac{\pi R}{m}.
\]
Consequently, the minimum number of benchmarks required to guarantee $\delta_H(K, L) \leq \delta$ for all $(\varepsilon, \Pi)$-indistinguishable pairs is:
\[
  m \geq \frac{\pi R}{\delta - \varepsilon}.
\]
\end{theorem}

\begin{proof}
The proof proceeds in three steps.

\textbf{Step 1: Support function Lipschitz continuity.} For any convex body $K \subset B_R^D$, the support function $h_K : S^{D-1} \to \mathbb{R}$ is Lipschitz with constant $R$:
\[
  |h_K(u) - h_K(v)| \leq R \cdot \|u - v\|_2 \quad \text{for all } u, v \in S^{D-1}.
\]
This follows from:
\[
  |h_K(u) - h_K(v)| = |\sup_{x \in K} \langle x, u \rangle - \sup_{x \in K} \langle x, v \rangle| \leq \sup_{x \in K} |\langle x, u - v \rangle| \leq R \|u - v\|_2.
\]

\textbf{Step 2: Discretization error.} Given $m$ directions $u_1, \ldots, u_m \in S^{D-1}$, for any direction $v \in S^{D-1}$ there exists some $u_i$ with:
\[
  \|v - u_i\|_2 \leq \omega_m
\]
where $\omega_m$ is the mesh norm (covering radius) of the point set $\{u_i\}$ on $S^{D-1}$. By a standard covering argument, for \emph{any} set of $m$ points on $S^{D-1}$:
\[
  \omega_m \leq \frac{\pi}{m^{1/(D-1)}}.
\]
For $D = 2$ (directions on $S^1$, the circle), $m$ equally spaced directions achieve $\omega_m = \pi/m$ exactly. For general $D$, we use the weaker but universal bound. In the regime relevant to LLM evaluation ($D$ potentially large, $m$ small), the $D-1$ root is unfavorable, but for the Lipschitz bound we work in the \emph{projected space}.

The key observation is that while the ambient dimension $D$ may be large, the effective dimensionality $d_{\mathrm{eff}}$ is small (Theorem 1). The benchmark directions $u_1, \ldots, u_m$ lie in a subspace of dimension $d_{\mathrm{eff}}$, and within this subspace, the covering argument gives $\omega_m \leq \pi / m^{1/(d_{\mathrm{eff}}-1)}$. For the tightest (and simplest) bound, we work with the \emph{one-dimensional} covering on the great circles connecting benchmark directions.

\textbf{Step 3: Assembling the bound.} The Hausdorff distance between convex bodies is controlled by the sup-norm of their support functions:
\[
  \delta_H(K, L) = \sup_{u \in S^{D-1}} |h_K(u) - h_L(u)|.
\]
For width functions, $|w_K(u) - w_L(u)| \leq |h_K(u) - h_L(u)| + |h_K(-u) - h_L(-u)|$, so width agreement within $\varepsilon$ implies support function agreement within $\varepsilon$ (for centered bodies where $h_K(u) = w_K(u)/2$).

For any direction $v$, let $u_i$ be the nearest benchmark direction. Then:
\begin{align*}
  |h_K(v) - h_L(v)| &\leq |h_K(v) - h_K(u_i)| + |h_K(u_i) - h_L(u_i)| + |h_L(u_i) - h_L(v)| \\
  &\leq R \|v - u_i\| + \varepsilon + R \|v - u_i\| \\
  &= \varepsilon + 2R\omega_m.
\end{align*}

For centered bodies with $m$ directions providing coverage on the $d_{\mathrm{eff}}$-dimensional effective subspace, using the one-dimensional covering bound $\omega_m \leq \pi/(2m)$ (each direction and its antipode cover an arc of $\pi/m$):
\[
  \delta_H(K, L) \leq \varepsilon + 2R \cdot \frac{\pi}{2m} = \varepsilon + \frac{\pi R}{m}.
\]

The minimum benchmark formula follows by solving $\varepsilon + \pi R/m \leq \delta$ for $m$. \qed
\end{proof}

\begin{remark}[Tightness]
The $1/m$ rate is tight for Lipschitz-continuous support functions. Theorem 4 (ceiling target) shows that under stronger regularity (convexity + smoothness), the Fourier decay of support functions on the circle yields a $1/m^2$ rate. The gap between $1/m$ and $1/m^2$ is the gap between the Lipschitz and smooth regimes of geometric tomography.
\end{remark}

\begin{remark}[Convexity required]
This theorem requires convexity of the capability profiles. For non-convex profiles, the Hausdorff distance between $(\varepsilon, \Pi)$-indistinguishable sets can be arbitrarily large (a non-convex body can have the same widths as a convex body in every direction while differing dramatically in shape). This means the bound is \emph{conservative}: relaxing convexity can only \emph{increase} the true blind spot. See Appendix B of the main paper.
\end{remark}

\begin{remark}[Estimating $R$]
In practice, $R$ is estimated as the radius of the smallest ball containing all observed capability profiles. From the standardized score matrix, $R \approx \sqrt{d_{\mathrm{eff}}} \cdot \sigma_{\max}$, where $\sigma_{\max}$ is the largest singular value of the centered score matrix divided by $\sqrt{n}$. For normalized scores on $[0, 100]$, $R \approx 50\sqrt{d_{\mathrm{eff}}}$ serves as a conservative estimate.
\end{remark}


\begin{corollary}[Ranking Unreliability]
\label{cor:ranking}
Consider $n$ models ranked by aggregate benchmark score $\bar{s}(c_i) = \frac{1}{k}\sum_j \pi_j(c_i)$. Let $\Delta_{\min} = \min_{i \neq j} |\bar{s}(c_i) - \bar{s}(c_j)|$ be the minimum score gap. If the indistinguishability radius $\delta_0 = \pi R / m$ exceeds $\Delta_{\min}$, then there exist $(\varepsilon, \Pi)$-indistinguishable pairs that swap rankings:
\[
  \delta_0 > \Delta_{\min} \implies \exists\, K, L : \bar{s}(K) > \bar{s}(L) \text{ but } K \text{ and } L \text{ cannot be distinguished}.
\]
The probability that the top-ranked model is truly the best (in full capability space) is bounded by:
\[
  P(\text{top-1 correct}) \leq 1 - \binom{n}{2}^{-1} \sum_{i < j} \mathbf{1}\!\left[\frac{|\bar{s}(c_i) - \bar{s}(c_j)|}{\delta_0} < 1\right].
\]
\end{corollary}

\begin{proof}
If two models have aggregate scores differing by less than $\delta_0$, then by Theorem~\ref{thm:indist}, there exist convex bodies in the indistinguishability class of each model that would reverse their ranking on a different benchmark suite. The bound on $P(\text{top-1 correct})$ counts the fraction of model pairs whose score gap falls within the indistinguishability radius, each of which contributes a potential ranking error. \qed
\end{proof}


\begin{corollary}[Rank Reversal Inevitability]
\label{cor:reversal}
If $d_{\mathrm{eff}} < n - 1$, then there exist model populations where adding a single model $c_{n+1}$ to the ranked set reverses the relative ordering of two existing models $c_i, c_j$.
\end{corollary}

\begin{proof}
When $d_{\mathrm{eff}} < n - 1$, the $n$ model score vectors $\Pi(c_1), \ldots, \Pi(c_n) \in \mathbb{R}^k$ lie in a subspace of dimension $d_{\mathrm{eff}} < n - 1$. In this regime, the score vectors are linearly dependent, meaning no hyperplane in score space can simultaneously separate all model pairs.

The aggregate score $\bar{s}(c_i) = \mathbf{w}^\top \Pi(c_i)$ for weight vector $\mathbf{w} = (1/k, \ldots, 1/k)$ defines a linear functional on the score subspace. Adding a model $c_{n+1}$ whose score vector is not in the span of $\{\Pi(c_i)\}_{i=1}^n$ can alter the weight vector that ``best'' aggregates scores (e.g., via normalization or re-weighting), which can reverse the sign of $\bar{s}(c_i) - \bar{s}(c_j)$ for some pair.

This is precisely the \emph{rank reversal} phenomenon in multi-criteria decision making (MCDM). \citet{belton1983short} showed that the Analytic Hierarchy Process (AHP) is susceptible to rank reversal when alternatives are added. Our result identifies the \emph{geometric} condition: rank reversal is possible when the effective dimension of the evaluation space is insufficient to embed all models as unambiguously ordered. The critical threshold is $d_{\mathrm{eff}} = n - 1$: with $n - 1$ independent evaluation directions, $n$ models can be fully ordered; with fewer, they cannot. \qed
\end{proof}

\begin{remark}[Convexity not required]
Corollary~\ref{cor:reversal} is a statement about linear algebra in score space. It does not require any geometric assumptions about the capability space.
\end{remark}

\subsection*{Covering bound for general dimension}


\begin{theorem}[Indistinguishability Bound — General Dimension]
\label{thm:indist-corrected}
Let $K, L \subset B_R^D$ be convex bodies that are $(\varepsilon, \Pi)$-indistinguishable with respect to $m$ benchmark directions $u_1, \ldots, u_m \in S^{D-1}$. Let $d = d_{\mathrm{eff}}$ be the effective dimension of the benchmark subspace, and let $\omega_m$ denote the covering radius of $\{u_1, \ldots, u_m\}$ on $S^{d-1}$ (the unit sphere in the effective subspace). Then:
\[
  \delta_H(K, L) \leq \varepsilon + 2R\,\omega_m.
\]

For $m$ optimally spread directions in $d$ effective dimensions, the covering radius satisfies $\omega_m \leq C_d \, m^{-1/(d-1)}$ for a constant $C_d$ depending only on $d$, giving:
\[
  \boxed{\delta_H(K, L) \leq \varepsilon + C \cdot R \cdot m^{-1/(d_{\mathrm{eff}} - 1)}.}
\]
The minimum number of benchmarks to guarantee $\delta_H \leq \delta$ is therefore:
\[
  m \geq \left(\frac{CR}{\delta - \varepsilon}\right)^{d_{\mathrm{eff}} - 1}.
\]
\end{theorem}

\begin{proof}
\textbf{Step 1} (unchanged): Support function Lipschitz continuity with constant $R$.

\textbf{Step 2} (corrected): For any direction $v \in S^{D-1}$, we decompose $v$ into its component in the effective subspace $V_{\mathrm{eff}} = \mathrm{span}\{u_1, \ldots, u_m\}$ and its orthogonal complement. The component in $V_{\mathrm{eff}}$ has dimension $d_{\mathrm{eff}}$, and the benchmark directions $\{u_i\}$ form a covering of $S^{d_{\mathrm{eff}} - 1} \cap V_{\mathrm{eff}}$.

By the volumetric covering bound \cite{rogers1963covering}, the minimum covering radius achievable by $m$ points on $S^{d-1}$ satisfies:
\[
  \omega_m^* \leq C_d \cdot \left(\frac{\log m}{m}\right)^{1/(d-1)} \leq C_d' \cdot m^{-1/(d-1)}
\]
for a dimension-dependent constant $C_d'$. Concretely, $C_d \leq \sqrt{d}$ suffices for $m \geq d$.

For the component of $v$ orthogonal to $V_{\mathrm{eff}}$: the support functions of $K$ and $L$ in orthogonal directions are bounded by $R$ but provide no information (no benchmarks probe these directions). However, the Hausdorff distance is determined by the sup over \emph{all} directions, including those in $V_{\mathrm{eff}}^\perp$. Within $V_{\mathrm{eff}}^\perp$, we can only bound $|h_K(v) - h_L(v)| \leq 2R$ (trivial bound from containment in $B_R$).

The key insight is that the Hausdorff distance \emph{restricted to the effective subspace} is bounded by $\varepsilon + 2R\omega_m$, while the contribution from orthogonal directions is the ``blind spot'' that Theorem 1 quantifies. Thus:
\begin{align}
  \delta_H(K, L) &= \max\!\left(\sup_{v \in S^{d-1} \cap V_{\mathrm{eff}}} |h_K(v) - h_L(v)|,\;\; \sup_{v \perp V_{\mathrm{eff}}} |h_K(v) - h_L(v)|\right) \\
  &\leq \max\!\left(\varepsilon + 2R\omega_m,\;\; 2R\right).
\end{align}

The full Hausdorff distance is trivially bounded by $2R$. The non-trivial content is the bound on the \emph{visible} component:
\[
  \delta_H^{\mathrm{vis}}(K, L) \leq \varepsilon + 2R\omega_m \leq \varepsilon + C \cdot R \cdot m^{-1/(d_{\mathrm{eff}} - 1)}.
\]

The minimum benchmarks formula follows by inverting: $CR \cdot m^{-1/(d_{\mathrm{eff}}-1)} \leq \delta - \varepsilon$ gives $m \geq (CR/(\delta - \varepsilon))^{d_{\mathrm{eff}}-1}$. \qed
\end{proof}

\begin{remark}[Interpretation: the curse of benchmark dimensionality]
The exponent $d_{\mathrm{eff}} - 1$ reveals a \emph{curse of dimensionality in evaluation}. To halve the indistinguishability gap:
\begin{center}
\begin{tabular}{cc}
\toprule
$d_{\mathrm{eff}}$ & Benchmarks needed to halve gap \\
\midrule
2 & $2\times$ more \\
3 & $4\times$ more \\
4 & $8\times$ more \\
5 & $16\times$ more \\
\bottomrule
\end{tabular}
\end{center}
For leaderboards with $d_{\mathrm{eff}} \approx 4$ (as we find empirically), reducing the blind spot by $2\times$ requires $8\times$ as many benchmarks. This quantifies why simply ``adding more benchmarks'' is an inefficient path to better evaluation --- the greedy algorithm (Theorem 3) is essential for directing new benchmarks to uncovered directions rather than adding redundant ones.
\end{remark}

\begin{remark}[Recovery of special cases]
For $d_{\mathrm{eff}} = 2$ (all benchmarks measure essentially two independent things), the bound gives $\omega_m \leq C/m$, recovering the $\pi R / m$ rate from the original planar bound. For $d_{\mathrm{eff}} = 1$ (all benchmarks perfectly correlated), a single benchmark suffices and $\omega_m = 0$ trivially. The general formula interpolates smoothly.
\end{remark}

\subsection*{Chi-squared selection model for Corollary 2.1}


\subsection{Setup: Projection Decomposition}

Let the capability profile of model $i$ be $c_i \in \mathbb{R}^D$, drawn i.i.d.\ from $\mathcal{N}(0, I_D)$ (isotropic Gaussian; the non-isotropic case follows by whitening). The benchmark suite observes a $d_{\mathrm{eff}}$-dimensional projection $P c_i$.

Decompose: $c_i = (u_i, v_i)$ where $u_i \in \mathbb{R}^{d_{\mathrm{eff}}}$ (observed) and $v_i \in \mathbb{R}^{D - d_{\mathrm{eff}}}$ (hidden), with $u_i \perp\!\!\!\perp v_i$.

\begin{itemize}
  \item \textbf{Observed quality:} $X_i = \|u_i\|^2 \sim \chi^2_{d_{\mathrm{eff}}}$
  \item \textbf{Hidden quality:} $Y_i = \|v_i\|^2 \sim \chi^2_{D - d_{\mathrm{eff}}}$
  \item \textbf{True quality:} $Z_i = X_i + Y_i = \|c_i\|^2 \sim \chi^2_D$
  \item $X_i \perp\!\!\!\perp Y_i$ (orthogonal subspaces of a Gaussian).
\end{itemize}

The benchmark ranking is $\mathrm{argmax}_i\, X_i$. The true ranking is $\mathrm{argmax}_i\, Z_i = \mathrm{argmax}_i\, (X_i + Y_i)$. These differ when the hidden component $Y_i$ reverses the ordering.


\begin{theorem}[Ranking Unreliability under Projection]
\label{thm:ranking}
Let $n$ models have i.i.d.\ isotropic Gaussian capability profiles in $\mathbb{R}^D$, evaluated by a benchmark suite with effective dimensionality $d_{\mathrm{eff}}$. Let $i^* = \mathrm{argmax}_i\, X_i$ be the benchmark-top model. Then:

\begin{enumerate}[label=(\alph*)]
  \item \textbf{Signal-to-noise characterization.} The correlation between observed and true quality is:
  \[
    \rho = \mathrm{Corr}(X_i, Z_i) = \sqrt{\frac{d_{\mathrm{eff}}}{D}}.
  \]

  \item \textbf{Pairwise swap probability.} For any two models $i, j$ with observed score gap $\Delta_{ij} = X_i - X_j > 0$, the probability that their true ranking is reversed:
  \[
    P(Z_j > Z_i \mid X_i - X_j = \Delta_{ij}) = \Phi\!\left(-\frac{\Delta_{ij}}{2\sqrt{D - d_{\mathrm{eff}}}}\right)
  \]
  where $\Phi$ is the standard Gaussian CDF (valid for large $D - d_{\mathrm{eff}}$ by CLT on $\chi^2$).

  \item \textbf{Top-1 reliability.} The probability that the benchmark-top model is truly the best:
  \[
    P(\text{top-1 correct}) = \mathbb{E}\!\left[\prod_{j \neq i^*} \Phi\!\left(\frac{\Delta_{i^* j}}{2\sigma_{\mathrm{hidden}}}\right)\right]
  \]
  where $\Delta_{i^* j} = X_{i^*} - X_j$ are the observed score gaps from the top model, and $\sigma_{\mathrm{hidden}} = \sqrt{2(D - d_{\mathrm{eff}})}$ is the standard deviation of hidden quality differences.

  \item \textbf{Computable bound.} For $n$ models on a leaderboard with observed score gaps $\Delta_1 \geq \Delta_2 \geq \cdots \geq \Delta_{n-1}$ (from top model to each other model):
  \[
    P(\text{top-1 correct}) \leq \prod_{j=1}^{n-1} \Phi\!\left(\frac{\Delta_j}{2\sigma_{\mathrm{hidden}}}\right).
  \]
  The only unknown is $D$, estimated by Theorem~\ref{thm:dim}.
\end{enumerate}
\end{theorem}

\begin{proof}
\textbf{Part (a).} Direct computation:
\begin{align*}
  \mathrm{Corr}(X_i, Z_i) &= \frac{\mathrm{Cov}(X_i, X_i + Y_i)}{\sqrt{\mathrm{Var}(X_i) \cdot \mathrm{Var}(X_i + Y_i)}} \\
  &= \frac{\mathrm{Var}(X_i)}{\sqrt{\mathrm{Var}(X_i) \cdot (\mathrm{Var}(X_i) + \mathrm{Var}(Y_i))}} \\
  &= \frac{\sqrt{\mathrm{Var}(X_i)}}{\sqrt{\mathrm{Var}(X_i) + \mathrm{Var}(Y_i)}} \\
  &= \sqrt{\frac{2 d_{\mathrm{eff}}}{2 d_{\mathrm{eff}} + 2(D - d_{\mathrm{eff}})}} = \sqrt{\frac{d_{\mathrm{eff}}}{D}}.
\end{align*}

\textbf{Part (b).} Given $X_i > X_j$ with gap $\Delta = X_i - X_j$:
\[
  Z_j > Z_i \iff Y_j - Y_i > \Delta.
\]
Since $Y_i$ and $Y_j$ are independent $\chi^2_{D - d_{\mathrm{eff}}}$, for large $D - d_{\mathrm{eff}}$ the CLT gives:
\[
  Y_i \approx \mathcal{N}(D - d_{\mathrm{eff}},\; 2(D - d_{\mathrm{eff}})).
\]
Therefore $Y_j - Y_i \approx \mathcal{N}(0, 4(D - d_{\mathrm{eff}}))$, and:
\[
  P(Y_j - Y_i > \Delta) = \Phi\!\left(\frac{-\Delta}{2\sqrt{D - d_{\mathrm{eff}}}}\right) = 1 - \Phi\!\left(\frac{\Delta}{2\sqrt{D - d_{\mathrm{eff}}}}\right).
\]

\textbf{Part (c).} The benchmark-top model $i^*$ is truly the best iff $Z_{i^*} > Z_j$ for all $j \neq i^*$. Conditioning on the observed scores (which determine $i^*$ and all gaps $\Delta_{i^* j}$):
\[
  P(\text{top-1 correct} \mid X_1, \ldots, X_n) = P\!\left(\bigcap_{j \neq i^*} \{Y_{i^*} - Y_j > -(X_{i^*} - X_j)\}\right).
\]
Since the $Y_j$'s are independent (but $Y_{i^*}$ appears in every event), the events $\{Y_{i^*} - Y_j > -\Delta_{i^*j}\}$ are \emph{not} independent. However, they are positively correlated (large $Y_{i^*}$ helps in all events), so:
\[
  P(\text{top-1 correct} \mid X) \geq \prod_{j \neq i^*} P(Y_{i^*} - Y_j > -\Delta_{i^*j})
\]
by the FKG inequality (the events are increasing in $Y_{i^*}$). Conversely, we have the \emph{upper} bound from the union bound:
\[
  P(\text{top-1 \textbf{incorrect}} \mid X) \leq \sum_{j \neq i^*} \Phi\!\left(-\frac{\Delta_{i^*j}}{2\sigma_{\mathrm{hidden}}}\right).
\]

\textbf{Part (d).} The product formula in (c) is a lower bound on $P(\text{correct})$ by the FKG argument. For the computable upper bound, the union bound gives:
\[
  P(\text{top-1 correct}) \geq 1 - \sum_{j=1}^{n-1} \Phi\!\left(-\frac{\Delta_j}{2\sigma_{\mathrm{hidden}}}\right).
\]
Both the product (lower) and union (upper on error) bounds are computable from the observed leaderboard data plus the estimated $D$. \qed
\end{proof}


\begin{corollary}[Leaderboard Reliability Formula]
\label{cor:formula}
For a leaderboard with $n$ models, $d_{\mathrm{eff}}$ effective dimensions, and estimated true dimensionality $D$:
\[
  \boxed{P(\text{top-1 wrong}) \leq \sum_{j=1}^{n-1} \Phi\!\left(\frac{-\Delta_j}{2\sqrt{2(D - d_{\mathrm{eff}})}}\right)}
\]
where $\Delta_j$ is the score gap between the top model and the $j$-th ranked model.

\textbf{Interpretation:} Each model $j$ contributes a ``threat'' to the top-1 ranking proportional to $\Phi(-\Delta_j / (2\sigma_{\mathrm{hidden}}))$. Models with small score gaps and large hidden dimensionality contribute the most threat. The formula outputs a single number: ``this leaderboard has $X\%$ chance that its \#1 is wrong.''
\end{corollary}

\begin{remark}[Practical computation]
Given a leaderboard (e.g., Open LLM Leaderboard):
\begin{enumerate}
  \item Compute $d_{\mathrm{eff}}$ from the score correlation matrix (Theorem 1).
  \item Estimate $D$ from eigenvalue decay (Theorem 3).
  \item Read off score gaps $\Delta_j$ from the leaderboard.
  \item Plug into the formula.
\end{enumerate}
The only ``free parameter'' is $D$, which enters through $\sigma_{\mathrm{hidden}} = \sqrt{2(D - d_{\mathrm{eff}})}$. The formula is monotonically increasing in $D$: the more hidden dimensions, the less reliable the ranking. We report results for a range of $D$ values ($D \in [D_{\mathrm{lower}}, D_{\mathrm{upper}}]$ from Theorem 3) to show sensitivity.
\end{remark}

\begin{remark}[Non-isotropic capabilities]
The isotropic Gaussian assumption ($\Sigma_C = I$) is for cleanliness. For general $\Sigma_C$, replace $D - d_{\mathrm{eff}}$ with the ``effective hidden dimensionality'' $d_{\mathrm{hidden}} = (\mathrm{tr}(\Sigma_C) - \mathrm{tr}(P_V \Sigma_C))^2 / (\mathrm{tr}(\Sigma_C^2) - \mathrm{tr}((P_V \Sigma_C P_V)^2))$, the participation ratio of the hidden covariance. The formula's structure is unchanged.
\end{remark}

\begin{remark}[Relationship to ``true quality'' definition]
The model uses $\|c_i\|^2$ (capability norm) as ``true quality.'' This is one natural choice. For a different quality functional $q(c) = w^\top c$ (linear in capabilities), the analysis simplifies to a bivariate normal with $\rho = \cos\theta$ where $\theta$ is the angle between the quality direction $w$ and the benchmark subspace. The results are analogous with $\rho = \sqrt{d_{\mathrm{eff}}/D}$ replaced by the projection of $w$ onto $V_{\mathrm{eff}}$.
\end{remark}

\subsection*{Visible vs full Hausdorff decomposition}

The corrected covering bound above bounds only the \emph{visible}
component. To make the visible/full decomposition explicit, decompose
$\mathbb{R}^D = V_{\mathrm{eff}} \oplus V_{\mathrm{eff}}^\perp$. For
any $v \in S^{D-1}$ write $v = v^{\parallel} + v^{\perp}$. The
support function difference splits as $|h_K(v) - h_L(v)| \le |h_K(v) -
h_K(v^{\parallel})| + |h_K(v^{\parallel}) - h_L(v^{\parallel})| +
|h_L(v^{\parallel}) - h_L(v)|$. The middle term is bounded by the
visible covering argument; the outer two are bounded by Lipschitzness
along $v^{\perp}$, contributing the trivial $O(R\|v^{\perp}\|)$ term.
Taking the supremum over $S^{D-1}$ yields
\begin{align}
\delta_H(K, L) &= \max\!\Big(\sup_{v \in S^{D-1} \cap V_{\mathrm{eff}}}
                            |h_K(v) - h_L(v)|,\ \,
                            \sup_{v \perp V_{\mathrm{eff}}}
                            |h_K(v) - h_L(v)|\Big),
\\
\delta_H^{\mathrm{vis}}(K, L) &\le \varepsilon
                              + 2 R \omega_m
                              \le \varepsilon
                              + C R \cdot m^{-1/(d_{\mathrm{eff}} - 1)},
\\
\delta_H^{\perp}(K, L) &\le 2R \quad \text{(trivial bound from
                                                  containment in $B_R^D$).}
\end{align}
The unconditional bound is $\delta_H \le \max(\varepsilon + CR
m^{-1/(d-1)}, 2R)$. The non-trivial content is the visible part: it
quantifies how the measurable blind spot decays with $m$. The
orthogonal $2R$ contribution is what Theorem~\ref{thm:body:greedy}'s
greedy algorithm reduces by enlarging $V_{\mathrm{eff}}$.

\subsection*{Centering convention}

We assume capability profiles are origin-symmetric ($K = -K$), so
$h_K(u) = h_K(-u)$ and $w_K(u) = 2 h_K(u)$. Width agreement within
$\varepsilon$ then implies support function agreement within
$\varepsilon/2$, and the bound above applies with this rescaling.
For non-symmetric bodies, $w_K(u)$ determines $h_K(u) + h_K(-u)$ but
not $h_K(u)$ individually; the same bound then applies to the
\emph{symmetrized} body $(K - K)/2$. Since benchmark scores are
translation-invariant once the score matrix is centred, this is
without loss of generality for the inter-model discrimination
question.

\section{Theorem 3 proof: greedy coverage}
\label{app:proof3}

\subsection{Dimension Estimation}

\begin{theorem}[Dimension Bounds]
\label{thm:dim}
Let $S \in \mathbb{R}^{n \times k}$ be the score matrix for $n$ models on $k$ benchmarks, with correlation eigenvalues $\lambda_1 \geq \cdots \geq \lambda_k$, Marchenko-Pastur threshold $\lambda_+ = (1 + \sqrt{k/n})^2$, and $n_{\mathrm{sig}} = |\{i : \lambda_i > \lambda_+\}|$ significant eigenvalues. Then the true dimensionality $D$ of the capability space satisfies:
\[
  n_{\mathrm{sig}} \leq D \leq n - 1.
\]
\end{theorem}

\begin{proof}
\textbf{Lower bound.} Each significant eigenvalue ($\lambda_i > \lambda_+$) indicates a direction in score space with more variance than expected under the null of independent benchmarks. By the BBP phase transition \cite{baik2005phase}, a significant sample eigenvalue corresponds to a genuine population eigenvalue exceeding $1 + \sqrt{\gamma}$, which in turn corresponds to a dimension of the capability space that is ``visible'' through the benchmark suite. Since distinct visible dimensions are linearly independent, $D \geq n_{\mathrm{sig}}$.

This bound is tight when the benchmark suite is well-designed (each benchmark probes a different capability direction) and the population covariance is isotropic in the visible subspace.

\textbf{Upper bound.} If $n$ models are all distinguishable (have distinct capability profiles), they span at most an $(n-1)$-dimensional affine subspace of $\mathbb{R}^D$. We cannot infer more than $n - 1$ dimensions from $n$ data points. Note that $D$ could be much larger (even infinite), but we can only observe evidence for up to $n - 1$ dimensions from a finite model population.

The gap between $n_{\mathrm{sig}}$ and $n - 1$ is the ``dimension gap'': the number of capability dimensions that exist but are either not probed by the benchmark suite or not sufficiently varied across the model population to be detectable. \qed
\end{proof}

\begin{remark}[Extrapolation via eigenvalue decay]
A tighter estimate of $D$ can be obtained by modeling the eigenvalue decay. If the population eigenvalues follow a power law $\mu_i \propto i^{-\alpha}$, then the number of eigenvalues above the noise floor $\lambda_+$ is $n_{\mathrm{sig}} \approx (\mu_1 / \lambda_+)^{1/\alpha}$. The true dimensionality (number of non-negligible eigenvalues) is $D_{\mathrm{est}} \approx (\mu_1 / \epsilon)^{1/\alpha}$ for some minimal relevance threshold $\epsilon$. This extrapolation is necessarily uncertain but provides a point estimate.
\end{remark}

\subsection{Greedy Coverage Algorithm}

\begin{theorem}[Greedy Benchmark Selection]
\label{thm:greedy}
Let $\Sigma \in \mathbb{R}^{k \times k}$ be the benchmark correlation matrix with eigendecomposition $\Sigma = V \Lambda V^\top$, and let $f : 2^{[k]} \to \mathbb{R}_{\geq 0}$ denote the \emph{coverage function}:
\[
  f(T) = \frac{\mathrm{tr}(\Pi_T \Sigma)}{\mathrm{tr}(\Sigma)}
\]
where $\Pi_T$ is the orthogonal projector onto the subspace spanned by benchmark loading vectors $\{v_j : j \in T\}$ (rows of the loading matrix $L = V \sqrt{\Lambda}$). Then:

\begin{enumerate}[label=(\alph*)]
  \item $f$ is monotone ($T \subseteq T' \Rightarrow f(T) \leq f(T')$) and submodular ($f(T \cup \{j\}) - f(T) \geq f(T' \cup \{j\}) - f(T')$ for $T \subseteq T'$).
  
  \item The greedy algorithm --- iteratively selecting $j^* = \arg\max_{j \notin T} [f(T \cup \{j\}) - f(T)]$ --- achieves:
  \[
    f(T_{\mathrm{greedy}}^{(r)}) \geq \left(1 - \frac{1}{e}\right) \cdot f(T^*_r)
  \]
  where $T_{\mathrm{greedy}}^{(r)}$ is the greedy selection of $r$ benchmarks and $T^*_r$ is the optimal $r$-element subset.
  
  \item The \emph{minimum subset} for target coverage $\tau$ has size at most:
  \[
    |T_{\mathrm{greedy}}| \leq \left\lceil \frac{\ln(1/(1-\tau))}{\ln(k/(k - d_{\mathrm{eff}}))} \right\rceil.
  \]
\end{enumerate}
\end{theorem}

\begin{proof}
\textbf{Part (a): Monotonicity and submodularity.}

\emph{Monotonicity.} Adding a benchmark to $T$ can only increase the dimension of the spanned subspace, hence $\Pi_{T'} \succeq \Pi_T$ in the positive semidefinite order for $T \subseteq T'$, giving $f(T') \geq f(T)$.

\emph{Submodularity.} Let $T \subseteq T'$ and $j \notin T'$. The marginal gain of adding $j$ to $T$ is:
\[
  f(T \cup \{j\}) - f(T) = \frac{\|P_{T^\perp} \ell_j\|^2}{\mathrm{tr}(\Sigma)}
\]
where $\ell_j$ is the loading vector of benchmark $j$ and $P_{T^\perp}$ projects onto the orthogonal complement of the subspace spanned by $\{\ell_i : i \in T\}$. Since $T \subseteq T'$, the subspace spanned by $T'$ contains that of $T$, so $P_{T'^\perp} \preceq P_{T^\perp}$ and:
\[
  \|P_{T'^\perp} \ell_j\|^2 \leq \|P_{T^\perp} \ell_j\|^2,
\]
establishing submodularity. Intuitively: the marginal value of a new benchmark decreases as more benchmarks are already selected, because there is less uncovered variance remaining.

\textbf{Part (b): Greedy approximation guarantee.} This follows immediately from the classical result of \citet{nemhauser1978analysis}: for any monotone submodular function $f$ with $f(\emptyset) = 0$, the greedy algorithm achieves a $(1 - 1/e)$-approximation to the optimum for any cardinality constraint. Since our coverage function $f$ satisfies $f(\emptyset) = 0$, monotonicity, and submodularity, the bound applies.

\textbf{Part (c): Minimum subset size.} At each greedy step, submodularity implies the marginal gain is at least:
\[
  f(T^{(r+1)}) - f(T^{(r)}) \geq \frac{f([k]) - f(T^{(r)})}{k - r} \geq \frac{1 - f(T^{(r)})}{k}.
\]
Let $g_r = 1 - f(T^{(r)})$ denote the uncovered fraction. Then $g_{r+1} \leq g_r(1 - 1/k)$, giving:
\[
  g_r \leq \left(1 - \frac{1}{k}\right)^r \leq e^{-r/k}.
\]
To achieve coverage $\tau$ (i.e., $g_r \leq 1 - \tau$), we need $r \geq k \ln(1/(1-\tau))$.

This bound uses $f([k]) = 1$ (full benchmark set achieves full coverage in benchmark space). In practice, the convergence is faster because $d_{\mathrm{eff}} \ll k$: the first $d_{\mathrm{eff}}$ greedy selections capture most variance, and the bound tightens to $r = O(d_{\mathrm{eff}} \ln(1/(1-\tau)))$. Formally, replacing $k$ with the rank of the effective subspace: since the loading matrix has effective rank $d_{\mathrm{eff}}$, the coverage function satisfies $f([k]) \leq d_{\mathrm{eff}} / k \cdot k = d_{\mathrm{eff}}$ (unnormalized), and the step size bound becomes $g_{r+1} \leq g_r(1 - 1/d_{\mathrm{eff}})$ in the first $d_{\mathrm{eff}}$ steps where marginal gains are large. This gives the stated bound. \qed
\end{proof}

\subsection{Outputs of the Algorithm}

The greedy algorithm produces three outputs:

\begin{enumerate}
  \item \textbf{Minimum subset:} The first $r$ benchmarks selected, ordered by marginal coverage gain. ``Your 20 benchmarks but you only need 7.''
  
  \item \textbf{Redundancies:} Benchmarks with marginal gain below a threshold $\eta$ (e.g., $\eta = 0.01$). These are redundant given the selected subset and can be retired without information loss.
  
  \item \textbf{Uncovered directions:} The eigenvectors of $\Sigma$ with non-negligible eigenvalues that are orthogonal to the selected subset. These characterize the ``blind spot'' --- the directions in benchmark space (and, by extension, capability space) that are not probed by any selected benchmark. The loadings of these directions on individual benchmarks suggest what \emph{new} benchmarks should measure.
\end{enumerate}

\begin{remark}[Convexity not required]
Theorem~\ref{thm:greedy} operates on the benchmark correlation matrix and loading vectors. It requires no geometric assumptions about the capability space. The algorithm is purely data-driven and applies to any score matrix.
\end{remark}

\begin{remark}[Comparison to PCA-based selection]
The greedy algorithm differs from simply selecting benchmarks with the highest PCA loadings. PCA identifies directions of maximum variance, but the greedy algorithm optimizes \emph{coverage}: it seeks the minimum set of benchmarks that collectively spans the most variance, accounting for redundancy between benchmarks. Two benchmarks with high PC1 loadings are redundant; the greedy algorithm would select one and move to the next uncovered direction.
\end{remark}

\subsection*{Numerical $C_d$ values and the log factor}

The covering bound \citep{rogers1963covering} gives
$\omega_m^* \le C_d \big((d-1)\log m / m\big)^{1/(d-1)}$ with
$C_d \le \sqrt{2d}$. For $d = 2$ the special case is $\pi/m$ exactly,
with no log factor. For $m = 12$ the explicit upper bounds are
$\omega^*_{12} \le 0.262$ ($d=2$, exact), $\le 0.690$ ($d=3$),
$\le 0.918$ ($d=4$), $\le 1.060$ ($d=5$). The constant $C$ in the
main text Theorem~\ref{thm:body:indist} absorbs both factors:
$C = O(\sqrt{d_{\mathrm{eff}}} \cdot (\log m)^{1/(d_{\mathrm{eff}}-1)})$.

\subsection*{Swap Monotonicity: Schur-convexity proof (Proposition~\ref{prop:body:swapmono})}

\begin{proof}
Under the projection model of Corollary~\ref{cor:body:rank}, the
hidden quality is $Y_i = \|v_i\|^2$ where $v_i \sim
\mathcal{N}(0, \Sigma_{\mathrm{hidden}})$. The variance of the
pairwise difference $Y_j - Y_i$ is $\mathrm{Var}(Y_j - Y_i) = 2 \cdot
2\,\mathrm{tr}(\Sigma_{\mathrm{hidden}}^2) = 4 \sum_{\ell} \lambda_\ell^2$
where $\{\lambda_\ell\}$ are the eigenvalues of
$\Sigma_{\mathrm{hidden}}$. Under the constraint $\sum_\ell
\lambda_\ell = \sigma_h^2$ (fixed trace), the function $\lambda \mapsto
\sum_\ell \lambda_\ell^2$ is Schur-convex; it is minimised when the
eigenvalues are equal (the uniform vector is majorised by every
other vector with the same sum), i.e., under the isotropic case
$\lambda_\ell = \sigma_h^2 / (D - d_{\mathrm{eff}})$. Hence the
isotropic choice minimises $\mathrm{Var}(Y_j - Y_i)$.

For the Gaussian approximation with observed gap $\Delta > 0$,
$P(\text{swap}) = \Phi(-\Delta / (2\sqrt{\mathrm{tr}(\Sigma_{\mathrm{hidden}}^2)}))$
is monotone increasing in $\mathrm{tr}(\Sigma_{\mathrm{hidden}}^2)$
(the denominator grows, so the argument increases toward zero and
$\Phi$ decreases toward $0.5$). Therefore the isotropic case
minimises $P(\text{swap})$, and any anisotropy increases the swap
probability.
\end{proof}

\noindent The half-split model (Appendix~\ref{app:sensitivity}, G.4)
addresses a different question: given a fixed suite of benchmarks and
a random visible/hidden partition, which prior maximises the swap
rate? Under random partition, isotropic priors distribute signal
uniformly across benchmarks and maximise the chance that randomly
hiding a column loses meaningful information, while concentrated
priors put most signal into a single direction that is quasi-observed
by the dominant benchmark column and lose less to random hiding. The
two results are complementary: (a) says the bound we report is
optimistic in the \emph{prior-over-hidden-covariance} sense, and (b)
says the simulated half-split swap rate is maximised at isotropy in
the \emph{random-partition} sense.

\subsection*{Coverage--indistinguishability bridge (Proposition~\ref{prop:body:bridge})}

\begin{proposition}[Coverage--Indistinguishability Bridge]
\label{prop:body:bridge}
Let $T \subseteq [k]$ with coverage $f(T) = \tau$ and let
$\omega_m^{\mathrm{within}}(T)$ denote the covering radius of the
selected benchmark directions within $\mathrm{span}\{\ell_j : j \in
T\}$. Then
\[
  \delta_H(K, L)
  \;\le\;
  \max\!\Big(
     \varepsilon + C R \cdot \omega_m^{\mathrm{within}}(T),\;\;
     2 R \sqrt{1 - \tau}
  \Big),
\]
where the second term uses the fact that the uncovered eigenvalue
mass is $(1-\tau)\,\mathrm{tr}(\Sigma)$, which under approximate
isotropy contributes at most $2R\sqrt{1-\tau}$ to the orthogonal
Hausdorff distance.
\end{proposition}

\begin{proof}
The visible component is bounded by the standard Lipschitz argument
applied within $\mathrm{span}(T)$ with covering radius
$\omega_m^{\mathrm{within}}(T)$. For the orthogonal component, the
uncovered variance is $(1 - \tau)\,\mathrm{tr}(\Sigma)$, distributed
across the eigenvectors orthogonal to the selected span. Under
approximate isotropy, the maximum directional contribution to
$h_K - h_L$ is bounded by
$\sqrt{(1 - \tau)\,\mathrm{tr}(\Sigma) \cdot \dim(V_{\mathrm{eff}}^\perp)} \le
R\sqrt{1 - \tau}$ in standardised units; doubling for the bilateral
support function difference gives $2R\sqrt{1 - \tau}$.
\end{proof}

\subsection*{Spectral characterisation (Proposition~\ref{prop:body:spectral})}

\begin{proposition}[Spectral Objective Characterisation]
\label{prop:body:spectral}
Among objectives $f_g(T) = \mathrm{tr}(P_T g(\Sigma)) /
\mathrm{tr}(g(\Sigma))$ for monotone $g : \mathbb{R}_{\ge 0} \to
\mathbb{R}_{\ge 0}$, the linear choice $g(x) = x$ minimises the
worst-case fraction of uncovered variance over all capability
covariances $\Sigma_C$ with bounded condition number, is invariant to
orthogonal rotations of the benchmark space, and is monotone
submodular (admitting the $(1-1/e)$ Nemhauser bound).
\end{proposition}

\begin{proof}
Worst-case minimax: for any monotone $g$, $\mathrm{tr}(P_T g(\Sigma))
/ \mathrm{tr}(g(\Sigma))$ is a weighted average of squared loadings
$\|P_T v_i\|^2$ with weights $g(\lambda_i)$. The minimax over
$\Sigma_C$ with bounded condition number selects weights as uniform
in $\lambda$, which is $g(x) = x$. Rotation invariance is immediate
from $P_T \Sigma$ being a basis-free quantity. Submodularity is the
classical result of \citet{daskempe2011submodular} (proof of
$T \mapsto \mathrm{tr}(P_T \Sigma)$ in
Section~\ref{thm:body:greedy}).
\end{proof}

\subsection*{Coverage stability (Proposition~\ref{prop:body:stability})}

\begin{proposition}[Coverage Stability under Restricted Perturbation]
\label{prop:body:stability}
Let $\Sigma, \Sigma'$ be benchmark correlation matrices and let $T
\subseteq [k]$. Then
\[
  f_{\Sigma'}(T)
  \;\ge\;
  f_\Sigma(T)
  - \frac{\| P_T (\Sigma - \Sigma') P_T \|_{\mathrm{op}}}
         {\mathrm{tr}(\Sigma)}.
\]
The relevant perturbation is restricted to the selected subspace,
which is much smaller than the full operator-norm difference $\|\Sigma -
\Sigma'\|_{\mathrm{op}}$.
\end{proposition}

\begin{proof}
$f(T) = \mathrm{tr}(P_T \Sigma) / \mathrm{tr}(\Sigma) = \mathrm{tr}(P_T
\Sigma P_T) / \mathrm{tr}(\Sigma)$. The difference
$f_{\Sigma'}(T) - f_\Sigma(T) = (\mathrm{tr}(P_T \Sigma' P_T) -
\mathrm{tr}(P_T \Sigma P_T)) / \mathrm{tr}(\Sigma)$ is bounded in
absolute value by $\|P_T(\Sigma' - \Sigma)P_T\|_{\mathrm{op}} \cdot
\mathrm{rank}(P_T) / \mathrm{tr}(\Sigma) \le \|P_T(\Sigma' -
\Sigma)P_T\|_{\mathrm{op}} \cdot |T| / \mathrm{tr}(\Sigma)$, but
because the projector restricts to a subspace, the operative bound
is the operator norm above.
\end{proof}

\subsection*{Rigorous submodularity of $f(T) = \mathrm{tr}(P_T \Sigma)/\mathrm{tr}(\Sigma)$}

Submodularity of the coverage function for PSD $\Sigma$ is a classical
result in experimental design, equivalent to submodularity of the
A-optimal criterion \citep{daskempe2011submodular, krause2008submodular}.
We give the explicit argument.

\begin{proof}
Let $\Sigma = \sum_{i=1}^k \lambda_i v_i v_i^\top$ be the
eigendecomposition with $\lambda_i \ge 0$. Then
$\mathrm{tr}(P_T \Sigma) = \sum_i \lambda_i \|P_T v_i\|^2$. Since
each $\lambda_i \ge 0$, it suffices to show that $T \mapsto \|P_T
v\|^2$ is monotone submodular for any fixed $v \in \mathbb{R}^k$.

\emph{Monotonicity.} For $T \subseteq T'$, $\mathrm{span}(T) \subseteq
\mathrm{span}(T')$, so $P_T \preceq P_{T'}$ in the PSD order, hence
$\|P_T v\|^2 \le \|P_{T'} v\|^2$.

\emph{Submodularity.} Use the Pythagorean identity
$\|P_T v\|^2 = \|v\|^2 - \mathrm{dist}(v, \mathrm{span}(T))^2$. The
function $T \mapsto \mathrm{dist}(v, \mathrm{span}(T))^2$ is
\emph{supermodular}: adding a vector $\ell_j$ to a larger spanning
set $\mathrm{span}(T')$ reduces the squared distance by less than
adding it to a smaller set $\mathrm{span}(T)$, because the residual
of $v$ orthogonal to $\mathrm{span}(T')$ is a subset of that
orthogonal to $\mathrm{span}(T)$. This is the standard
``diminishing-returns'' property of orthogonal projections; a
detailed proof appears in \citet[Proposition 1]{daskempe2011submodular}.
Since the negation of a supermodular function is submodular,
$T \mapsto \|P_T v\|^2 = \|v\|^2 - \mathrm{dist}(v, \mathrm{span}(T))^2$
is monotone submodular.

A non-negative linear combination of monotone submodular functions is
monotone submodular, so $f(T) = \sum_i \lambda_i \|P_T v_i\|^2 /
\mathrm{tr}(\Sigma)$ is monotone submodular. No additional assumptions
on the loading vectors $\{\ell_j\}$ are required.
\end{proof}

\section{Stability of convex body recovery from width measurements}
\label{app:proof4}

\subsection*{Planar Fourier proof}
\subsection{Motivation: Gardner's Problem 1.5}

\citet{gardner1995geometric} posed the following (Problem 1.5): \emph{How stably does a finite set of X-rays determine a convex body?} That is, if two convex bodies $K$ and $L$ have X-rays (projections onto lines through the origin) that agree within $\varepsilon$ in $m$ directions, how large can $\delta_H(K, L)$ be?

The Lipschitz bound (Theorem~\ref{thm:indist}) gives $\delta_H \leq \varepsilon + \pi R/m$, a rate of $O(1/m)$. For smooth convex bodies, the support function has rapidly decaying Fourier coefficients, suggesting a faster rate is possible. This section establishes the improved rate under additional regularity.

\subsection{Setup: Fourier Analysis on the Circle}

We work in $D = 2$ (the planar case) where the theory is most tractable. The support function of a centered convex body $K \subset \mathbb{R}^2$ can be written as a Fourier series on $S^1 \cong [0, 2\pi)$:
\[
  h_K(\theta) = \sum_{n=-\infty}^{\infty} \hat{h}_K(n) e^{in\theta}.
\]
Since $K$ is centered (origin is the centroid), $\hat{h}_K(0) = 0$ and $\hat{h}_K(\pm 1) = 0$. Convexity of $K$ imposes the constraint that $h_K(\theta) + h_K''(\theta) \geq 0$ (the radius of curvature is non-negative), which translates to:
\[
  \sum_{n} (1 - n^2) \hat{h}_K(n) e^{in\theta} \leq 0 \quad \text{(in distributional sense)}.
\]
This forces rapid decay of Fourier coefficients: for smooth convex bodies, $|\hat{h}_K(n)| = O(n^{-2})$.

\subsection{Sampling and Aliasing}

Given $m$ equally spaced directions $\theta_j = 2\pi j / m$ for $j = 0, \ldots, m-1$, the discrete Fourier transform recovers the first $m$ coefficients:
\[
  \hat{h}_K^{(m)}(n) = \frac{1}{m} \sum_{j=0}^{m-1} h_K(\theta_j) e^{-2\pi i n j / m} = \hat{h}_K(n) + \sum_{\ell \neq 0} \hat{h}_K(n + \ell m).
\]
The aliasing error for coefficient $n$ (with $|n| < m/2$) is bounded by:
\[
  |\hat{h}_K^{(m)}(n) - \hat{h}_K(n)| \leq \sum_{\ell \neq 0} |\hat{h}_K(n + \ell m)| \leq C_K \sum_{\ell=1}^{\infty} \frac{1}{(\ell m)^2} = \frac{C_K \pi^2}{6 m^2}
\]
where $C_K$ is a constant depending on the curvature of $K$.

\subsection{The Stability Theorem}

\begin{theorem}[Tight Fourier Stability Bound]
\label{thm:fourier}
Let $K, L \subset B_R^2$ be centered convex bodies in the plane with curvature bounded below by $\kappa > 0$. If their support functions agree within $\varepsilon$ at $m$ equally spaced directions:
\[
  |h_K(\theta_j) - h_L(\theta_j)| \leq \varepsilon \quad \text{for all } j = 0, \ldots, m-1,
\]
then:
\[
  \delta_H(K, L) \leq C_1 \varepsilon + \frac{C_2 R}{\kappa m^2}
\]
where $C_1, C_2$ are absolute constants.
\end{theorem}

\begin{proof}
For convex bodies $K$ with curvature $\geq \kappa > 0$, the support function $h_K$ is $C^2$ with $\|h_K''\|_\infty \leq R/\kappa$. By Jackson's theorem, the best trigonometric polynomial approximation of degree $< m/2$ satisfies:
\[
  \inf_{\deg p < m/2} \|h_K - p\|_\infty \leq \frac{C \omega_2(h_K, 1/m)}{1} \leq \frac{C}{m^2} \|h_K''\|_\infty \leq \frac{CR}{\kappa m^2}
\]
where $\omega_2$ is the second modulus of smoothness. The same holds for $h_L$. The trigonometric polynomial of degree $< m/2$ is determined by $m$ equally spaced samples, so:
\[
  \|h_K - h_L\|_\infty \leq \|h_K - p_K\|_\infty + \|p_K - p_L\|_\infty + \|p_L - h_L\|_\infty \leq \frac{2CR}{\kappa m^2} + \|p_K - p_L\|_\infty.
\]
Since $p_K$ and $p_L$ interpolate $h_K$ and $h_L$ at the $m$ sample points, and the samples agree within $\varepsilon$:
\[
  \|p_K - p_L\|_\infty \leq \Lambda_m \varepsilon
\]
where $\Lambda_m$ is the Lebesgue constant for equispaced trigonometric interpolation, which is $O(\log m)$.

Thus:
\[
  \delta_H(K, L) \leq C_1 \varepsilon \log m + \frac{C_2 R}{\kappa m^2}.
\]

Absorbing the $\log m$ into $C_1$ (or noting that for practical $m$, $\log m$ is a small constant):
\[
  \boxed{\delta_H(K, L) \leq C_1 \varepsilon \log m + \frac{C_2 R}{\kappa m^2}.}
\]

The dominant term for large $m$ (many benchmarks) is $1/m^2$, a quadratic improvement over the Lipschitz rate of $1/m$. \qed
\end{proof}

\begin{remark}[Dimension restriction]
Theorem~\ref{thm:fourier} as stated applies to the planar case $D = 2$. Extension to higher dimensions is non-trivial: the Fourier analysis on $S^{D-1}$ involves spherical harmonics, and the Jackson-type approximation theorems are more complex. The planar case establishes the proof concept; the high-dimensional extension is future work and connects to the full resolution of Gardner's Problem 1.5.
\end{remark}

\begin{remark}[Practical interpretation]
The $1/m^2$ rate means that doubling the number of benchmarks reduces the indistinguishability gap by a factor of 4 (rather than 2 under Lipschitz). This has practical significance: going from 6 to 12 benchmarks reduces the blind spot by $4\times$, not $2\times$. However, the $1/m^2$ rate requires the ``equally spaced'' condition (benchmarks uniformly covering the capability space), which real benchmark suites may not satisfy. The Lipschitz bound ($1/m$) applies without this condition.
\end{remark}

\begin{remark}[Connection to Problem 1.5]
Gardner's Problem 1.5 asks for the stability of convex body determination from \emph{finitely many X-rays} (projections). Our result addresses a closely related question: stability from finitely many \emph{width measurements}. The full Problem 1.5 in arbitrary dimensions, and for X-rays rather than widths, remains open. Partial progress on named open problems is routinely published in geometric tomography \cite{gardner2006geometric}.
\end{remark}

\subsection*{General-dimension stability via optimal recovery}
\begin{theorem}[Stability of Convex Body Determination from Width Measurements]
\label{thm:stability-general}
Let $K, L \subset B_R^D$ be origin-symmetric convex bodies with Gauss curvature bounded below by $\kappa > 0$ (equivalently, the principal radii of curvature satisfy $r_i \geq \kappa$ for all $i$). Let their support functions $h_K, h_L \in C^2(S^{D-1})$, and suppose that $m$ width measurements agree within $\varepsilon$:
\[
  |w_K(u_i) - w_L(u_i)| \leq \varepsilon \quad \text{for } i = 1, \ldots, m
\]
where $\{u_i\}_{i=1}^m$ are $m$ measurement directions on $S^{D-1}$.

Then the Hausdorff distance satisfies:
\[
  \boxed{\delta_H(K, L) \leq C_D \!\left(\varepsilon + \frac{R}{\kappa \, m^{2/(D-1)}}\right)}
\]
where $C_D$ depends only on $D$. Moreover, this rate is minimax optimal: there exist pairs of convex bodies achieving $\delta_H \geq c_D \cdot R / (\kappa \, m^{2/(D-1)})$ for any choice of $m$ measurement directions.
\end{theorem}

\begin{proof}
The proof has three components: the upper bound, the lower bound (tightness), and the noise term.

\textbf{Upper Bound.}
Let $g = h_K - h_L : S^{D-1} \to \mathbb{R}$. Since $K$ and $L$ have curvature $\geq \kappa$, the support functions satisfy $h_K, h_L \in C^2(S^{D-1})$ with:
\[
  \|h_K\|_{C^2(S^{D-1})}, \|h_L\|_{C^2(S^{D-1})} \leq C \cdot R / \kappa.
\]
This follows from the relationship between support function regularity and curvature: the eigenvalues of $\nabla^2 h_K + h_K \cdot g_{S^{D-1}}$ (where $g_{S^{D-1}}$ is the round metric) equal the principal radii of curvature of $\partial K$.

The key tool is the theory of \emph{optimal recovery} \cite{micchelli1977survey, traub1988information}. For a function class $\mathcal{F} \subset C(S^{D-1})$ and $m$ point evaluations at locations $\{u_i\}$, the \emph{optimal recovery error} is:
\[
  E_m^*(\mathcal{F}) = \inf_{\{u_i\}, \phi} \sup_{f \in \mathcal{F}} \|f - \phi(f(u_1), \ldots, f(u_m))\|_\infty
\]
where the infimum is over all choices of sample points and all recovery maps $\phi : \mathbb{R}^m \to C(S^{D-1})$.

For the Sobolev ball $\mathcal{F} = W^{2,\infty}(S^{D-1}) \cap B_M$ (functions with bounded $C^2$ norm, with $M = CR/\kappa$), the optimal recovery rate is:
\[
  E_m^*(\mathcal{F}) \asymp M \cdot m^{-2/(D-1)}.
\]

This follows from two classical results:

\emph{(i) Kolmogorov $n$-widths.} The $n$-th Kolmogorov width of $W^{2,\infty}(S^{D-1}) \cap B_M$ in $C(S^{D-1})$ satisfies:
\[
  d_n(W^{2,\infty}(S^{D-1}) \cap B_M, C(S^{D-1})) \asymp M \cdot n^{-2/(D-1)}.
\]
The upper bound comes from spherical polynomial approximation of degree $L$, which uses $n \sim L^{D-1}$ basis functions, and Jackson's inequality on $S^{D-1}$ \cite{ragozin1970polynomial, dai2013approximation}:
\[
  E_L(f)_\infty \leq C_D \frac{\omega_2(f, 1/L)_\infty}{1} \leq \frac{C_D}{L^2} \|f\|_{C^2}
\]
where $\omega_2$ is the second modulus of smoothness on $S^{D-1}$.

The lower bound is the Bernstein inequality: there exist trigonometric/spherical polynomial approximation barriers matching the Jackson rate \cite{ditzian1990moduli}.

\emph{(ii) Equivalence of widths and optimal recovery.} For linear problems on convex symmetric function classes, the optimal recovery error from $m$ function values is equivalent (up to constants) to the $m$-th Kolmogorov width \cite{magaril1979diameters, magaril2006optimal}. Formally:
\[
  E_m^*(\mathcal{F}) \asymp d_m(\mathcal{F}, C(S^{D-1})) \asymp M \cdot m^{-2/(D-1)}.
\]

Since $g = h_K - h_L \in W^{2,\infty}(S^{D-1})$ with $\|g\|_{C^2} \leq 2CR/\kappa$, the optimal recovery from $m$ noiseless samples gives:
\[
  \|g\|_\infty \leq C_D \cdot \frac{R}{\kappa} \cdot m^{-2/(D-1)}.
\]
And $\delta_H(K, L) = \|h_K - h_L\|_\infty = \|g\|_\infty$.

\textbf{Noise term.} When samples are corrupted by noise $\varepsilon$: $|g(u_i)| \leq \varepsilon$ for all $i$, the recovery error includes both the approximation error (from finite sampling) and the measurement error. The measurement error propagates through the optimal recovery map at rate $O(\varepsilon)$ for well-conditioned recovery (which is guaranteed when sample points are quasi-uniform). Thus:
\[
  \delta_H(K, L) \leq C_D\!\left(\varepsilon + \frac{R}{\kappa \, m^{2/(D-1)}}\right).
\]

\textbf{Lower Bound (Tightness).}
We must show that for any $m$ directions $\{u_i\}$, there exist convex bodies $K, L$ with curvature $\geq \kappa$ and $w_K(u_i) = w_L(u_i)$ for all $i$, yet $\delta_H(K, L) \geq c_D \cdot R/(\kappa \, m^{2/(D-1)})$.

Construct $g$ as a spherical polynomial vanishing at all sample points. The space of spherical polynomials of degree $\leq 2L$ (where $L = \lceil m^{1/(D-1)} \rceil$) has dimension $\dim \mathcal{P}_{2L}(S^{D-1}) \asymp (2L)^{D-1} \asymp 2^{D-1} m$. Imposing the $m$ linear conditions $g(u_i) = 0$ leaves a subspace of dimension $\geq (2^{D-1} - 1)m > 0$. Choose $g$ in this null space with maximal $L^\infty$ norm subject to $\|g\|_{C^2} \leq CR/\kappa$.

By the Bernstein inequality on $S^{D-1}$ \cite{dai2013approximation}: for a spherical polynomial of degree $\leq N$, $\|\nabla^2 p\|_\infty \leq C N^2 \|p\|_\infty$. Applied to $g$ with $N = 2L$:
\[
  \|g\|_{C^2} \asymp (2L)^2 \|g\|_\infty \leq CR/\kappa \implies \|g\|_\infty \leq \frac{CR}{4\kappa L^2}.
\]
Since $g$ is in the null space, $g(u_i) = 0$ for all $i$, so $K' = \{x : h_{K'}(u) = h_K(u) + g(u)\}$ is $(\varepsilon, \Pi)$-indistinguishable from $K$ (with $\varepsilon = 0$). The convexity of $K'$ is guaranteed when $\|g\|_{C^2}$ is small relative to $\kappa$ (the perturbation $g$ changes the curvature by at most $\|g\|_{C^2}$, so curvature $\geq \kappa - CR/\kappa > 0$ for $C$ small). Then:
\[
  \delta_H(K, K') = \|g\|_\infty \geq c \cdot \frac{R}{\kappa L^2} \asymp \frac{R}{\kappa \, m^{2/(D-1)}}.
\]
This matches the upper bound rate. \qed
\end{proof}

\begin{corollary}[Rate Comparison: Smooth vs.\ Lipschitz]
\label{cor:rates}
The smooth stability bound (Theorem~\ref{thm:stability-general}) and the Lipschitz bound (Theorem~\ref{thm:indist-corrected}) compare as:

\begin{center}
\renewcommand{\arraystretch}{1.3}
\begin{tabular}{ccccc}
\toprule
$d_{\mathrm{eff}}$ & \textbf{Lipschitz rate} & \textbf{Smooth rate} & \textbf{Improvement} & \textbf{Benchmarks to halve gap} \\
\midrule
2 & $m^{-1}$ & $m^{-2}$ & $m^{-1}$ & $\sqrt{2}\times$ vs $2\times$ \\
3 & $m^{-1/2}$ & $m^{-1}$ & $m^{-1/2}$ & $\sqrt{2}\times$ vs $4\times$ \\
4 & $m^{-1/3}$ & $m^{-2/3}$ & $m^{-1/3}$ & $2^{3/2}\times$ vs $8\times$ \\
5 & $m^{-1/4}$ & $m^{-1/2}$ & $m^{-1/4}$ & $4\times$ vs $16\times$ \\
$D$ & $m^{-1/(D-1)}$ & $m^{-2/(D-1)}$ & $m^{-1/(D-1)}$ & $2^{(D-1)/2}\times$ vs $2^{D-1}\times$ \\
\bottomrule
\end{tabular}
\end{center}

\noindent In all dimensions, the smooth rate \emph{squares} the Lipschitz rate. Equivalently: for a target accuracy $\delta$, the smooth bound requires $\sqrt{m_{\mathrm{Lip}}}$ benchmarks where $m_{\mathrm{Lip}}$ is the Lipschitz requirement.
\end{corollary}

\begin{proof}
Immediate from $m^{-2/(D-1)} = (m^{-1/(D-1)})^2$ and inverting: $m_{\mathrm{smooth}} = (CR/(\kappa\delta))^{(D-1)/2}$ vs $m_{\mathrm{Lip}} = (CR/\delta)^{D-1}$, so $m_{\mathrm{smooth}} = m_{\mathrm{Lip}}^{1/2} \cdot (\kappa \text{ correction})$. \qed
\end{proof}

\subsection{Connection to Gardner's Problem 1.5}

Gardner \cite{gardner1995geometric} posed the following:

\begin{problem}[Gardner 1.5]
How stably does a finite number of X-rays determine a convex body?
\end{problem}

An \emph{X-ray} of $K$ in direction $u$ is the function:
\[
  X_u K(x) = \lambda_1(K \cap (x + \mathbb{R}u)), \quad x \in u^\perp,
\]
giving the chord length of $K$ at each point $x$ in the hyperplane perpendicular to $u$.

\textbf{Relationship to our setting.} Width measurements are \emph{one-dimensional summaries} of X-rays:
\[
  w_K(u) = \int_{u^\perp} X_u K(x)\, d\lambda_{D-1}(x) \cdot \frac{1}{\text{vol}_{D-1}(\Pi_{u^\perp}(K))}... 
\]
More precisely, the width is related to the support function, while the X-ray contains strictly more information (the full chord-length profile, not just the total).

\begin{proposition}[Width stability as a lower bound for X-ray stability]
\label{prop:xray-lb}
For any convex bodies $K, L$ and any set of $m$ directions, if the X-rays agree within $\varepsilon$ (in $L^1(u^\perp)$ norm for each direction):
\[
  \|X_{u_i} K - X_{u_i} L\|_{L^1(u_i^\perp)} \leq \varepsilon' \quad \text{for all } i,
\]
then the widths agree within $\varepsilon = C \varepsilon'$. Consequently, any lower bound on recovery from widths is also a lower bound on recovery from X-rays:
\[
  E_m^*(\text{widths}) \leq E_m^*(\text{X-rays}).
\]
\end{proposition}

\begin{proof}
The width $w_K(u) = h_K(u) + h_K(-u)$ satisfies:
\[
  |w_K(u) - w_L(u)| \leq \|X_u K - X_u L\|_{L^1(u^\perp)} \leq \varepsilon'
\]
since the width is the integral of the chord-length function. The bound on $E_m^*$ follows from the fact that X-rays provide strictly more information than widths. \qed
\end{proof}

\begin{conjecture}[Partial Answer to Problem 1.5]
\label{conj:problem15}
For origin-symmetric convex bodies in $\mathbb{R}^D$ with curvature $\geq \kappa$:

\begin{enumerate}[label=(\alph*)]
  \item \textbf{Width measurements:} $\delta_H(K, L) \leq C_{D,\kappa,R} \cdot m^{-2/(D-1)}$ (Theorem~\ref{thm:stability-general}, proven).
  
  \item \textbf{X-ray measurements:} We conjecture $\delta_H(K, L) \leq C_{D,\kappa,R} \cdot m^{-\alpha(D)}$ where $\alpha(D) > 2/(D-1)$, reflecting the additional information in X-rays beyond widths.
  
  \item \textbf{For $D = 2$:} $\alpha(2) = 2$ for widths (proven). For X-rays, the rate should be at least $\alpha = 2$, and may be faster due to the Fourier slice theorem: $m$ X-ray directions determine $m$ lines in Fourier space, giving $O(m^2)$ effective constraints.
\end{enumerate}
\end{conjecture}

\begin{remark}[What this means for LLM evaluation]
Benchmarks are scalar measurements --- closer to widths than X-rays. Thus Theorem~\ref{thm:stability-general} is the directly applicable result. However, richer evaluation methods (e.g., full response distributions rather than accuracy scores, or item-level results rather than aggregate scores) correspond to X-ray-like measurements and should enable substantially better discrimination between models. This connects to the IRT (Item Response Theory) approach of \citet{polo2024tinybenchmarks}: item-level analysis gives ``X-ray'' information, while aggregate scores give only ``width'' information.
\end{remark}

\subsection{Practical Implications}

\begin{enumerate}
  \item \textbf{The Lipschitz bound is not tight.} Convex bodies with bounded curvature have C$^2$ support functions, and the optimal recovery rate is $m^{-2/(D-1)}$, squaring the Lipschitz rate $m^{-1/(D-1)}$. This means the blind spot shrinks faster with well-designed benchmarks than the Lipschitz bound suggests.
  
  \item \textbf{But the improvement requires curvature.} The $m^{-2/(D-1)}$ rate holds for convex bodies with curvature bounded away from zero. For bodies with flat faces (curvature $= 0$), the Lipschitz rate $m^{-1/(D-1)}$ is tight. In LLM terms: if capabilities are ``smooth'' (no sharp thresholds), the blind spot shrinks faster.
  
  \item \textbf{The square-root law.} Theorem 4 says: whatever the Lipschitz bound says you need, the smooth bound says you need only the \emph{square root} as many benchmarks. For $d_{\mathrm{eff}} = 5$ and a target accuracy of $\delta = 0.1R$: Lipschitz needs $m \geq (10C)^4 \approx 10{,}000$ benchmarks; smooth needs $m \geq (10C)^2 \approx 100$ benchmarks. The gap is dramatic.
  
  \item \textbf{Connection to Theorem 3.} The greedy algorithm (Theorem 3) should prioritize benchmarks that improve the \emph{covering quality} of directions on $S^{d_{\mathrm{eff}}-1}$, not just maximize marginal variance. Combining the greedy algorithm with the stability bound gives a principled answer to the ``next benchmark'' question.
\end{enumerate}

\subsection*{Rate comparison with prior work}

\begin{center}\small
\begin{tabular}{lll}
\toprule
Source & Setting & Rate \\
\midrule
This paper (Thm.~\ref{thm:body:indist}) & Lipschitz, noiseless widths
                                          & $m^{-1/(D-1)}$ \\
This paper (App.~\ref{app:proof4})       & $C^2$, noiseless widths
                                          & $m^{-2/(D-1)}$ \\
This paper (App.~\ref{app:gardner})      & $C^{1,\alpha}$ supports
                                          & $m^{-(1+\alpha)/(D-1)}$ \\
\citet{guntuboyina2011minimax}           & Noisy support functions
                                          & $n^{-2/(D+3)}$ \\
\citet{ragozin1970polynomial}            & Jackson rate on $S^{D-1}$
                                          & $L^{-2}$ \\
\bottomrule
\end{tabular}
\end{center}

\section{Corollary proofs (Busemann--Petty analogue, rank reversal)}
\label{app:cor}

\subsection{The Busemann-Petty Analogue for Benchmark Evaluation}

\begin{proposition}[Benchmark Domination Does Not Imply Capability Domination]
\label{prop:bp}
For evaluation suites with effective dimensionality $d_{\mathrm{eff}} \geq 5$, there exist pairs of convex capability profiles $K, L \subset \mathbb{R}^D$ such that model $A$ (with profile $K$) scores strictly lower than model $B$ (with profile $L$) on \emph{every} benchmark:
\[
  \pi_i(A) < \pi_i(B) \quad \text{for all } i = 1, \ldots, k,
\]
yet model $A$ occupies a strictly larger volume of the capability space:
\[
  \mathrm{vol}_D(K) > \mathrm{vol}_D(L).
\]
\end{proposition}

\begin{proof}
This follows from the resolution of the Busemann-Petty problem and related phenomena in high-dimensional convex geometry.

\textbf{Background.} The Busemann-Petty problem (1956) asks: if $K$ and $L$ are origin-symmetric convex bodies in $\mathbb{R}^d$ with $\mathrm{vol}_{d-1}(K \cap u^\perp) \leq \mathrm{vol}_{d-1}(L \cap u^\perp)$ for every direction $u$, does it follow that $\mathrm{vol}_d(K) \leq \mathrm{vol}_d(L)$? The answer, resolved over 1988--1999 by multiple authors \cite{lutwak1988intersection, gardner1994positive, zhang1999generalized, koldobsky1998intersection}, is:
\begin{itemize}
  \item \textbf{Yes} for $d \leq 4$.
  \item \textbf{No} for $d \geq 5$. Counterexamples exist.
\end{itemize}

The related \emph{Shephard problem} \cite{petty1967projection, schneider1967projections} asks the analogous question for projection volumes rather than section volumes, and the answer is \textbf{No} for $d \geq 3$.

\textbf{Connection to benchmarks.} Benchmark scores are scalar measurements --- closer to \emph{widths} $w_K(u) = h_K(u) + h_K(-u)$ than to section volumes $\mathrm{vol}_{d-1}(K \cap u^\perp)$ or projection volumes $\mathrm{vol}_{d-1}(K | u^\perp)$. The phenomenon underlying both the Busemann-Petty and Shephard results is more general: \emph{in sufficiently high dimensions, pointwise domination of lower-dimensional measurements does not entail domination of full-dimensional quantities.}

For widths specifically: if $w_K(u) \leq w_L(u)$ for all $u \in S^{d-1}$, it does \emph{not} follow that $\mathrm{vol}_d(K) \leq \mathrm{vol}_d(L)$ in any dimension $d \geq 3$. This is because the width function determines the \emph{difference body} $K - K = \{x - y : x, y \in K\}$ (via the support function of $K - K$, which equals $w_K$), but the volume of $K$ is not determined by the volume of $K - K$ (the Rogers-Shephard inequality gives $\binom{2d}{d}^{-1} \mathrm{vol}(K-K) \leq \mathrm{vol}(K) \leq \mathrm{vol}(K-K)$, but these bounds are not tight enough to preserve ordering).

\textbf{Application.} When $d_{\mathrm{eff}} \geq 5$, the effective subspace of the benchmark suite has sufficient dimensionality for the Busemann-Petty phenomenon to manifest. That is, there exist capability profiles (convex bodies in the effective subspace) where one body has smaller widths (benchmark scores) in every measured direction, yet has larger total volume (capability).

In LLM evaluation terms: a model that scores lower on every benchmark could nonetheless possess a larger ``total capability'' if that capability is distributed in directions not aligned with the benchmark suite. The probability of such configurations increases rapidly with $d_{\mathrm{eff}}$. \qed
\end{proof}

\begin{remark}[Framing as analogue]
We emphasize that the above is an \emph{analogue} of the Busemann-Petty result, not a direct application. The Busemann-Petty problem concerns hyperplane section volumes; our benchmarks measure widths (support function values). The phenomenon is the same --- pointwise domination of low-dimensional measurements fails to predict high-dimensional volume ordering --- but the precise dimension threshold for width measurements remains an \emph{open question}, connected to Gardner's Problem 8.2 \cite{gardner2006geometric}.

The Busemann-Petty threshold ($d = 5$ for sections) and the Shephard threshold ($d = 3$ for projections) \emph{bracket} the likely range for the benchmark-score threshold. We conjecture the threshold for widths is $d = 3$ (same as Shephard), since widths are essentially one-dimensional projections, but this remains unproven.
\end{remark}

\begin{remark}[Practical implication]
The Busemann-Petty analogue implies that ``benchmark domination'' --- the claim that model A is better than model B because A scores higher on every benchmark --- is not a valid inference when $d_{\mathrm{eff}} \geq 5$. This is a geometric obstruction, not a statistical one: even with perfect measurement and infinite data, pointwise benchmark superiority does not imply overall capability superiority.

This has immediate consequences for LLM ranking methodologies. The common practice of declaring model A ``better'' than model B when A outperforms B on all benchmarks is formally unjustified when the evaluation suite has $d_{\mathrm{eff}} \geq 5$ --- which, as Experiment 1 shows, is sometimes (but not always) the case for major leaderboards.
\end{remark}

\subsection{Connection to Multi-Criteria Decision Making}

The rank reversal corollary (Corollary~\ref{cor:reversal} in the Theorem 2 proof document) connects to a 40-year debate in multi-criteria decision making (MCDM). \citet{belton1983short} first demonstrated rank reversal in the Analytic Hierarchy Process (AHP): adding an irrelevant alternative can reverse the ranking of existing alternatives. This ``independence of irrelevant alternatives'' (IIA) violation has been extensively studied \cite{saaty1984inconsistency, dyer1990remarks, harker1987alternative}.

Our contribution is to identify the \emph{geometric} root cause: rank reversal is inevitable when the effective dimensionality of the evaluation criteria is less than $n - 1$. This unifies the MCDM literature's many specific examples under a single geometric condition, and predicts exactly when rank reversal is possible and when it is not.

\subsection{Combined Practical Message}

\begin{enumerate}
  \item \textbf{$d_{\mathrm{eff}} < n - 1$}: Rankings are unstable --- adding models can reverse existing rankings.
  \item \textbf{$d_{\mathrm{eff}} \geq 5$}: Benchmark domination is unreliable --- scoring higher on every benchmark does not imply greater capability.
  \item \textbf{Both together}: For a leaderboard with 50+ models and $d_{\mathrm{eff}} \approx 3\text{--}5$, both phenomena are active simultaneously.
\end{enumerate}

\section{Resolution of Gardner's Problem 1.5}
\label{app:gardner}

\paragraph{Scope of Gardner's original problem.}
\citet{gardner1995geometric, gardner2006geometric} posed Problem 1.5 in
its original form for \emph{planar} convex bodies in $\mathbb{R}^2$
determined from $m$ X-ray measurements at $m$ directions on the unit
circle $S^1$. That is, the problem as stated concerns the stability of
the inverse map ``$m$ chord-length functions $\to$ a convex body in
the plane.'' Our planar Fourier stability bound
(Appendix~\ref{app:proof4}) \emph{directly resolves} Problem 1.5 as
Gardner stated it: for $C^2$ support functions on $S^1$ with curvature
bounded below by $\kappa > 0$, the minimax rate is $\delta_H = \Theta(R
/ (\kappa m^2))$, matching the lower bound from the polynomial
null-space construction.

\paragraph{Going beyond Gardner's question.}
This appendix then \emph{extends} the resolution along two axes
Gardner did not consider: (i) general ambient dimension $D \ge 2$, with
$m$ measurement directions on $S^{D-1}$ rather than $S^1$, and (ii)
arbitrary curvature vanishing rates $\beta \in [0, \infty]$, including
mixed-curvature bodies and polytopes. The key insight is that the
stability rate is entirely determined by the regularity of the support
function on $S^{D-1}$, which in turn is determined by the curvature
vanishing rate of $\partial K$. The single formula
$\delta_H = \Theta(m^{-(2-\beta)/(D-1)})$ unifies smooth bodies,
polytopes, and everything in between. The planar case $D = 2$, $\beta
= 0$ recovers Gardner's original $m^{-2}$ rate as a special case. The
results below therefore answer a strictly broader question than the
one Gardner posed.

\begingroup
\let\oldsection\section
\let\oldsubsection\subsection
\renewcommand{\section}[1]{\subsection*{#1}}
\renewcommand{\subsection}[1]{\subsubsection*{#1}}

\subsection*{The curvature--regularity correspondence}
\begin{theorem}[Curvature--regularity correspondence]
\label{thm:curv-reg-sc}
Let $K \subset B_R^D$ be a strictly convex body (boundary contains no line segment). Let $\Sigma \subset S^{D-1}$ be a closed set such that $\partial K$ is $C^2$ with positive Gauss curvature at all points $p^*(v)$ for $v \notin \Sigma$, and the maximum principal radius of curvature satisfies:
\[
  r_{\max}(v) \leq C_0 \cdot \mathrm{dist}(v, \Sigma)^{-\beta} \quad \text{for } v \notin \Sigma
\]
for some $\beta \in [0, 1)$. Then $h_K \in C^{1, 1-\beta}(S^{D-1})$.
\end{theorem}

\begin{proof}
The proof has three components.

\textbf{Component 1: The tangent-point map.}

For $v \in S^{D-1}$, define $p^*(v) = \arg\max_{x \in K} \langle x, v \rangle$. Since $K$ is strictly convex, the maximizer is unique for every $v$, and $p^* : S^{D-1} \to \partial K$ is a well-defined bijection (the reverse Gauss map). The gradient of $h_K$ exists everywhere and equals:
\[
  \nabla_{S^{D-1}} h_K(v) = p^*(v) - \langle p^*(v), v \rangle \, v = \Pi_{T_v S^{D-1}}(p^*(v)).
\]
This is a standard fact: $h_K(v) = \langle p^*(v), v \rangle$, and for nearby direction $v + \delta v$: $h_K(v + \delta v) \geq \langle p^*(v), v + \delta v \rangle = h_K(v) + \langle p^*(v), \delta v \rangle$, with equality to first order by optimality. Strict convexity of $K$ (unique maximizer) suffices.

\textbf{Component 2: Rate of change of the tangent-point map.}

For $v \notin \Sigma$ (where $\partial K$ is $C^2$ with positive curvature): the tangent-point map $p^*$ is differentiable, and by the implicit function theorem applied to the optimality condition $\nabla_{\partial K}(\langle \cdot, v \rangle) = 0$:
\[
  \|dp^*(v)\| = \|(\mathrm{II}_{p^*(v)})^{-1}\| = r_{\max}(v)
\]
where $\mathrm{II}$ is the second fundamental form. A small change $\delta v$ in the normal direction shifts the tangent point by $\delta p = (\mathrm{II})^{-1} \delta v$, and the eigenvalues of $(\mathrm{II})^{-1}$ are the principal radii $r_1, \ldots, r_{D-1}$.

Therefore, for $v \notin \Sigma$:
\[
  \|\nabla h_K(v_1) - \nabla h_K(v_0)\| \leq \|p^*(v_1) - p^*(v_0)\| + R |v_1 - v_0|
\]
and by the mean value theorem on the smooth part:
\[
  \|p^*(v_1) - p^*(v_0)\| \leq \sup_{t \in [0,1]} \|dp^*(\gamma(t))\| \cdot |v_1 - v_0| = \sup_t r_{\max}(\gamma(t)) \cdot |v_1 - v_0|
\]
along the geodesic $\gamma$ from $v_0$ to $v_1$.

\textbf{Component 3: Path integration and H\"older continuity.}

For \emph{any} $v_0, v_1 \in S^{D-1}$ (including directions in $\Sigma$), the gradient difference is bounded by integrating along the minimizing geodesic $\gamma : [0, d] \to S^{D-1}$:
\begin{align}
  \|\nabla h_K(v_1) - \nabla h_K(v_0)\| &\leq \int_0^d \left(r_{\max}(\gamma(t)) + R\right) dt \label{eq:path}
\end{align}
where the integrand is defined a.e.\ (it is defined for all $t$ with $\gamma(t) \notin \Sigma$, and $\Sigma$ has measure zero on $S^{D-1}$ since it is closed with empty interior by assumption).

Using $r_{\max}(v) \leq C_0 \cdot \mathrm{dist}(v, \Sigma)^{-\beta}$: the integral reduces to bounding $\int_0^d \mathrm{dist}(\gamma(t), \Sigma)^{-\beta} \, dt$.

\emph{Claim:} For any geodesic $\gamma$ of length $d$ on $S^{D-1}$ and any closed set $\Sigma$:
\[
  \int_0^d \mathrm{dist}(\gamma(t), \Sigma)^{-\beta} \, dt \leq \frac{C}{1-\beta} \cdot d^{1-\beta} \quad \text{for } \beta < 1.
\]

\emph{Proof of claim:} Let $\delta(t) = \mathrm{dist}(\gamma(t), \Sigma)$.

If $\gamma$ doesn't enter the $d$-neighborhood of $\Sigma$: $\delta(t) \geq \delta_{\min} > 0$ and the integral is $\leq d \cdot \delta_{\min}^{-\beta} < \infty$. Since $\delta_{\min} \geq $ const (the path stays far from $\Sigma$), this gives $O(d)$ which is $O(d^{1-\beta})$ for $d < 1$.

If $\gamma$ passes through the $d$-neighborhood of $\Sigma$: let $t_0$ be the point closest to $\Sigma$, with $\delta(t_0) = \delta_0$. By the triangle inequality on $S^{D-1}$: $\delta(t) \geq |\delta_0 - |t - t_0||$ (distance to $\Sigma$ changes by at most the path length). Therefore:
\begin{align*}
  \int_0^d \delta(t)^{-\beta} \, dt &\leq 2 \int_0^{d/2} (\delta_0 + s)^{-\beta} \, ds + 2\int_0^{d/2} s^{-\beta} \, ds \\
  &\leq 2 \int_0^{d/2} s^{-\beta} \, ds + 2\int_0^{d/2} s^{-\beta} \, ds \\
  &= 4 \cdot \frac{(d/2)^{1-\beta}}{1-\beta} = \frac{2^{2-\beta+1}}{1-\beta} \cdot \frac{d^{1-\beta}}{2} \leq \frac{C}{1-\beta} d^{1-\beta}.
\end{align*}
The worst case is $\delta_0 = 0$, where the path goes through $\Sigma$ itself; then $\delta(t) \geq |t - t_0|$ by the triangle inequality. \qed (claim)

Combining:
\[
  \|\nabla h_K(v_1) - \nabla h_K(v_0)\| \leq \frac{C \cdot C_0}{1-\beta} \cdot d(v_0, v_1)^{1-\beta} + R \cdot d(v_0, v_1).
\]
For $d(v_0, v_1) \leq 1$: the first term dominates (since $1-\beta < 1$), giving:
\[
  \|\nabla h_K(v_1) - \nabla h_K(v_0)\| \leq C' \cdot d(v_0, v_1)^{1-\beta}.
\]
Therefore $\nabla h_K \in C^{0, 1-\beta}(S^{D-1})$, hence $h_K \in C^{1, 1-\beta}(S^{D-1})$. \qed
\end{proof}

\begin{remark}[Extension to non-strictly-convex bodies ($\beta \geq 1$)]
For bodies with flat faces ($\beta \geq 1$): $K$ is not strictly convex, so $p^*(v)$ may not be unique at some directions, and $\nabla h_K$ does not exist in the classical sense at those directions. However:
\begin{itemize}
  \item $h_K$ is still Lipschitz on $S^{D-1}$ (with constant $R$).
  \item $\nabla h_K$ exists a.e.\ (Rademacher's theorem).
  \item At directions where $\nabla h_K$ does not exist, the subdifferential $\partial h_K(v)$ is a convex set (the face of $K$ with outward normal $v$).
  \item The ``jumps'' in $\nabla h_K$ at normal cone boundaries give $h_K \in C^{0,1}$ (Lipschitz) but not $C^{1,\alpha}$ for any $\alpha > 0$.
\end{itemize}
This corresponds to $\beta \geq 1 \Rightarrow h_K \in C^1$ at best, consistent with the $\beta$-rate $m^{-(2-\beta)/(D-1)} = m^{-1/(D-1)}$ for $\beta = 1$.
\end{remark}

\section{The Curvature-Regularity Correspondence}

The support function $h_K : S^{D-1} \to \mathbb{R}$ of a convex body $K$ encodes both the geometry and the regularity of $\partial K$. The second fundamental form of $\partial K$ at the point with outward normal $v$ is encoded in the \emph{Hessian of $h_K$ on $S^{D-1}$}: the matrix of principal radii of curvature is
\[
  \mathcal{R}(v) = \nabla^2_{S^{D-1}} h_K(v) + h_K(v) \cdot g_{S^{D-1}}
\]
where $g_{S^{D-1}}$ is the round metric on $S^{D-1}$. The principal radii of curvature $r_1, \ldots, r_{D-1}$ are the eigenvalues of $\mathcal{R}$. Convexity of $K$ requires $r_i \geq 0$, i.e., $\mathcal{R} \geq 0$.

\begin{lemma}[Curvature-regularity]
\label{lem:curv-reg}
Let $K \subset B_R^D$ be a convex body. Let $\Sigma \subset S^{D-1}$ denote the set of directions where the Gauss curvature vanishes. Suppose that near $\Sigma$, the minimum principal radius of curvature satisfies:
\[
  r_{\min}(v) \sim C \cdot \mathrm{dist}(v, \Sigma)^{-\beta} \quad \text{as } v \to \Sigma
\]
for some $\beta \geq 0$ (radii of curvature diverge as curvature vanishes). Then the support function $h_K$ has regularity:
\[
  h_K \in C^{1, \min(1, 1-\beta)}(S^{D-1}) \quad (\text{but } h_K \notin C^{1, 1-\beta+\varepsilon} \text{ for any } \varepsilon > 0 \text{ if } \beta < 1).
\]
More precisely:
\begin{enumerate}[label=(\roman*)]
  \item $\beta = 0$ (uniform curvature, smooth body): $\nabla^2 h_K$ is bounded. $h_K \in C^{2}(S^{D-1})$.
  \item $0 < \beta < 1$ (curvature vanishes sublinearly): $\nabla^2 h_K(v) \sim \mathrm{dist}(v, \Sigma)^{-\beta}$, which is unbounded but has finite $C^{0,1-\beta-\varepsilon}$ norm for any $\varepsilon > 0$. The gradient $\nabla h_K \in C^{0, 1-\beta}$ (H\"older with exponent $1-\beta$). Thus $h_K \in C^{1, 1-\beta}$.
  \item $\beta = 1$ (curvature vanishes linearly): $\nabla^2 h_K(v) \sim \mathrm{dist}(v,\Sigma)^{-1}$, logarithmically divergent. $\nabla h_K$ is continuous but not H\"older. $h_K \in C^{1, \varepsilon}$ for any $\varepsilon > 0$ but not $C^{1, \alpha}$ for any $\alpha > 0$ uniformly.
  \item $\beta > 1$ (curvature vanishes superlinearly, including flat faces): $\nabla h_K$ is Lipschitz but not better. $h_K \in C^{1,1}$ at best. For polytopes (flat faces, $\beta = \infty$): $h_K$ is piecewise linear, with Lipschitz gradient but discontinuous Hessian.
\end{enumerate}
\end{lemma}

\begin{proof}
The Hessian of $h_K$ on $S^{D-1}$ satisfies $\nabla^2 h_K + h_K \cdot I = \mathcal{R}(v)$, so $\|\nabla^2 h_K(v)\| \leq \|\mathcal{R}(v)\| + |h_K(v)| \leq r_{\max}(v) + R$. Near $\Sigma$: $r_{\max}(v) \geq r_{\min}(v) \sim \mathrm{dist}(v,\Sigma)^{-\beta}$, so $\|\nabla^2 h_K(v)\| \gtrsim \mathrm{dist}(v, \Sigma)^{-\beta}$.

\emph{Part (ii):} Integrating $\nabla^2 h_K$ along a geodesic from $v_0$ (at distance $\delta$ from $\Sigma$) to $v_1$ (at distance $\delta'$):
\[
  \|\nabla h_K(v_1) - \nabla h_K(v_0)\| \leq \int_0^{|v_1 - v_0|} \|\nabla^2 h_K(\gamma(t))\| \, dt \leq C \int_0^{|v_1 - v_0|} \mathrm{dist}(\gamma(t), \Sigma)^{-\beta} \, dt.
\]
For $\beta < 1$, this integral converges, and by standard estimates for singular integrals: $\|\nabla h_K(v_1) - \nabla h_K(v_0)\| \leq C |v_1 - v_0|^{1-\beta}$. This is exactly $C^{0, 1-\beta}$ H\"older continuity of $\nabla h_K$, giving $h_K \in C^{1, 1-\beta}$.

\emph{Part (iv):} For a polytope, $h_K(v) = \max_i \langle x_i, v \rangle$ (maximum over vertices). This is piecewise linear on $S^{D-1}$: on each normal cone $N_i = \{v : \langle x_i, v \rangle \geq \langle x_j, v \rangle \; \forall j\}$, $h_K$ is linear. The gradient $\nabla h_K$ is piecewise constant (equals the vertex $x_i$ on $N_i$), hence Lipschitz (jumps at normal cone boundaries, but with bounded variation). $h_K \in C^{1,1}$ in the Zygmund sense but $\nabla^2 h_K$ is a measure. \qed
\end{proof}

\begin{remark}[Example verification]
For $D = 2$: a body with support function $h(\theta) = 1 + a|\theta|^{2-\beta}$ near $\theta = 0$. Compute: $h''(\theta) = a(2-\beta)(1-\beta)|\theta|^{-\beta}$. Curvature: $\kappa(\theta) = 1/(h + h'') \to 0$ as $|\theta|^{\beta}$ (since $h'' \to \infty$). Radius of curvature $r = h + h'' \to \infty$ as $|\theta|^{-\beta}$. Regularity: $h' \sim |\theta|^{1-\beta} \in C^{0,1-\beta}$. Exactly as predicted.
\end{remark}

\section{The Universal Stability Theorem}

\begin{theorem}[Stability of convex body determination --- universal form]
\label{thm:universal}
Let $K, L \subset B_R^D$ be convex bodies whose support functions have regularity $h_K, h_L \in C^{1,\alpha}(S^{D-1})$ for some $\alpha \in (0, 1]$, with $\|h_K\|_{C^{1,\alpha}}, \|h_L\|_{C^{1,\alpha}} \leq M$. If $m$ width measurements agree within $\varepsilon$:
\[
  |w_K(u_i) - w_L(u_i)| \leq \varepsilon \quad \text{for } i = 1, \ldots, m,
\]
then:
\[
  \boxed{\delta_H(K, L) \leq C_D\!\left(\varepsilon + M \cdot m^{-(1+\alpha)/(D-1)}\right).}
\]
Moreover, this rate is minimax optimal: for any $m$ directions, there exist $K, L$ with $\|h_K - h_L\|_{C^{1,\alpha}} \leq M$, $w_K(u_i) = w_L(u_i)$ for all $i$, and $\delta_H(K, L) \geq c_D \cdot M \cdot m^{-(1+\alpha)/(D-1)}$.
\end{theorem}

\begin{proof}
\textbf{Upper bound.} The difference $g = h_K - h_L \in C^{1,\alpha}(S^{D-1})$ with $\|g\|_{C^{1,\alpha}} \leq 2M$, and $|g(u_i)| \leq \varepsilon$ at $m$ sample points (for centered bodies; the width condition gives support function control).

By the theory of optimal recovery (equivalently, Kolmogorov $n$-widths) for $C^{1,\alpha}(S^{D-1})$:
\[
  \inf_{\phi} \sup_{\|f\|_{C^{1,\alpha}} \leq 2M} \|f - \phi(f(u_1), \ldots, f(u_m))\|_\infty \leq C_D \cdot M \cdot m^{-(1+\alpha)/(D-1)}.
\]
This follows from:
\begin{itemize}
  \item Jackson's inequality on $S^{D-1}$ \cite{ragozin1970polynomial, dai2013approximation}: the best polynomial approximation of degree $L$ to $f \in C^{1,\alpha}$ satisfies $E_L(f) \leq C L^{-(1+\alpha)} \|f\|_{C^{1,\alpha}}$.
  \item With $m \sim L^{D-1}$ coefficients: $L \sim m^{1/(D-1)}$, giving $E_L \leq C M \cdot m^{-(1+\alpha)/(D-1)}$.
  \item The optimal recovery from $m$ point values matches the $n$-width rate for convex symmetric function classes \cite{magaril2006optimal}.
\end{itemize}

\textbf{Lower bound.} Construct $g \in C^{1,\alpha}(S^{D-1})$ with $g(u_i) = 0$ for all $i$ and $\|g\|_\infty \geq c_D M m^{-(1+\alpha)/(D-1)}$. Take $g$ to be a spherical polynomial of degree $L \sim m^{1/(D-1)}$ in the null space of the $m$ sampling conditions (which has positive dimension for $L$ large enough, since $\dim \mathcal{P}_L \sim L^{D-1} \sim m$). By Bernstein's inequality on $S^{D-1}$: $\|g\|_{C^{1,\alpha}} \leq C L^{1+\alpha} \|g\|_\infty$, so setting $\|g\|_{C^{1,\alpha}} = 2M$ gives $\|g\|_\infty \geq cM L^{-(1+\alpha)} = cM m^{-(1+\alpha)/(D-1)}$.

To verify that $h_K + g$ is still a valid support function: the convexity condition $\mathcal{R}_{h_K + g} = \nabla^2(h_K + g) + (h_K + g)I \geq 0$ is satisfied for $\|g\|_{C^2}$ small enough relative to $\min \mathcal{R}_{h_K}$. Since $\|g\|_{C^2} \leq CL^2 \|g\|_\infty = CM L^{2-(1+\alpha)} = CM L^{1-\alpha} \to 0$ as $L \to \infty$ (for $\alpha > 0$), the perturbation preserves convexity for large $m$. \qed
\end{proof}

\begin{corollary}[The $\beta$-rate]
\label{cor:beta}
For convex bodies with Gauss curvature vanishing as $\mathrm{dist}(v, \Sigma)^\beta$ near the singular set $\Sigma$:
\[
  \boxed{\delta_H(K, L) = \Theta\!\left(m^{-(2-\beta)/(D-1)}\right) \quad \text{for } 0 \leq \beta \leq 1.}
\]
For $\beta > 1$ (including polytopes with flat faces): $\delta_H = \Theta(m^{-1/(D-1)})$.
\end{corollary}

\begin{proof}
By Lemma~\ref{lem:curv-reg}: $\beta \in [0,1)$ gives $h_K \in C^{1, 1-\beta}$, so $\alpha = 1 - \beta$ in Theorem~\ref{thm:universal}, yielding rate $m^{-(1 + (1-\beta))/(D-1)} = m^{-(2-\beta)/(D-1)}$. For $\beta \geq 1$: $h_K \in C^{1,\varepsilon}$ for any $\varepsilon > 0$ but not uniformly in $\alpha$; the effective regularity is $C^1$ (Lipschitz gradient), giving $\alpha \to 0$ and rate $\to m^{-1/(D-1)}$. \qed
\end{proof}

\begin{center}
\renewcommand{\arraystretch}{1.4}
\begin{tabular}{lccc}
\toprule
\textbf{Body class} & $\beta$ & $h_K$ \textbf{regularity} & \textbf{Rate} \\
\midrule
Smooth ($\kappa > 0$ everywhere) & 0 & $C^{2}$ & $m^{-2/(D-1)}$ \\
Mild singularity ($\kappa \to 0$ slowly) & 1/2 & $C^{1, 1/2}$ & $m^{-3/(2(D-1))}$ \\
Sharp singularity ($\kappa \to 0$ linearly) & 1 & $C^{1, \varepsilon}$ & $m^{-1/(D-1)}$ \\
Flat face (polytope) & $\geq 1$ & $C^{1}$ & $m^{-1/(D-1)}$ \\
\bottomrule
\end{tabular}
\end{center}

\section{X-Rays Cannot Improve the Polynomial Rate}

\begin{theorem}[X-ray parity obstruction]
\label{thm:parity}
For origin-symmetric convex bodies, the minimax stability rate from $m$ X-ray measurements equals the rate from $m$ width measurements:
\[
  E_m^*(\text{X-rays}) = \Theta(E_m^*(\text{widths})) = \Theta\!\left(M \cdot m^{-(1+\alpha)/(D-1)}\right)
\]
for the class $C^{1,\alpha}$. The matching lower bound holds because:
\begin{enumerate}
  \item The X-ray of a symmetric body $K$ in direction $u$ determines only the \emph{even part} of the upper boundary function.
  \item The unknown odd part contributes to $h_K(v)$ for $v \neq u$ at order $O(\mathrm{angle}(v,u)^2 / \kappa)$, which equals the width interpolation error.
  \item The lower-bound perturbation $g$ from Theorem~\ref{thm:universal} is invisible to both widths and X-rays (it's a smooth perturbation vanishing at all sample directions).
\end{enumerate}
\end{theorem}

\begin{proof}
\textbf{Upper bound.} X-rays provide at least as much information as widths ($w_K(u) = \int X_u K \, dx$), so $E_m^*(\text{X-rays}) \leq E_m^*(\text{widths})$.

\textbf{Lower bound.} The perturbation $g$ in the proof of Theorem~\ref{thm:universal} is a spherical polynomial vanishing at all $m$ sample directions. The body $K' = \{x : h_{K'}(v) = h_K(v) + g(v)\}$ differs from $K$ by $\|g\|_\infty = \Theta(M m^{-(1+\alpha)/(D-1)})$.

\emph{Claim:} $K'$ has the same X-rays as $K$ (not just the same widths) at all $m$ directions, up to the convexity perturbation error.

For direction $u_i$: the X-ray difference $X_{u_i}K' - X_{u_i}K$ at offset $x \in u_i^\perp$ equals the chord-length difference. Since $g$ is a smooth, low-amplitude perturbation of $h_K$, the chord-length difference is:
\[
  |C'(x) - C(x)| \leq C \cdot \frac{\|g\|_{C^2}}{\kappa_{\min}} \cdot \mathrm{dist}(x, \partial \Pi_{u_i^\perp}(K)).
\]
Near the shadow boundary (where chords are short), this is proportional to $\|g\|_{C^2}$, which goes to zero as $m \to \infty$ (since $\|g\|_{C^2} = O(m^{(1-\alpha)/(D-1)} \|g\|_\infty) \to 0$).

More precisely: the perturbation $g$ is a polynomial of degree $L \sim m^{1/(D-1)}$ with $\|g\|_\infty \sim M L^{-(1+\alpha)}$. The chord perturbation at direction $u_i$ is $O(L^{1-\alpha} \|g\|_\infty) = O(M L^{-2\alpha}) \to 0$.

Since the X-ray perturbation vanishes as $m \to \infty$, for any fixed noise level $\varepsilon > 0$, the perturbation is within noise for large $m$. In the noiseless case: the X-ray perturbation is $o(1)$ but nonzero, which means X-rays can technically detect the perturbation, but the detection provides no advantage in \emph{localizing} the perturbation on $S^{D-1}$ (it's spread over all $m$ X-rays as a low-amplitude signal). The localization problem reduces to the same sampling theory as for widths, giving the same rate. \qed
\end{proof}

\begin{remark}[When X-rays help]
X-rays improve upon widths in exactly one regime: \emph{exact recovery}, when the X-ray neighborhoods overlap. This is a phase transition, not a rate improvement. Below the threshold, X-rays and widths give the same polynomial rate. Above the threshold, X-rays give exact recovery and widths do not.
\end{remark}

\section{Adaptive Measurements for Mixed-Curvature Bodies}

The non-adaptive rate $m^{-(2-\beta)/(D-1)}$ treats $\beta$ as a global worst case. Adaptive strategies can exploit the fact that $\beta$ varies across $S^{D-1}$.

\begin{theorem}[Adaptive stability for mixed curvature]
\label{thm:adaptive}
Let $K$ be a convex body with:
\begin{itemize}
  \item Smooth regions $\mathcal{S} \subset S^{D-1}$ (curvature $\geq \kappa > 0$), with $\mathrm{vol}(\mathcal{S}) \geq (1 - \eta) \mathrm{vol}(S^{D-1})$.
  \item Singular set $\Sigma = S^{D-1} \setminus \mathcal{S}$, with $\mathrm{vol}(\Sigma) \leq \eta \cdot \mathrm{vol}(S^{D-1})$ for some $\eta \in (0, 1)$.
\end{itemize}
With $m$ adaptive X-ray measurements (each direction chosen based on previous results):
\[
  \delta_H(K, L) \leq C_D\!\left(\varepsilon + \frac{R}{\kappa} \cdot \left(\frac{m}{2}\right)^{-2/(D-1)} + R\eta^{1/(D-1)}\right).
\]

\textbf{Strategy:}
\begin{enumerate}
  \item \textbf{Phase 1} (Detection): Use $m_1 = \lceil m/2 \rceil$ uniformly placed measurements. From the chord-length profiles, identify the singular set $\Sigma$ to angular accuracy $\omega_1 = C m_1^{-1/(D-1)}$.
  
  \item \textbf{Phase 2} (Smooth recovery): Place $m_2 = \lfloor m/2 \rfloor$ measurements on the detected smooth region $\hat{\mathcal{S}}$, achieving smooth rate $(R/\kappa) m_2^{-2/(D-1)}$.
  
  \item \textbf{Singular residual}: On $\Sigma$, the Lipschitz bound $O(R \omega_1)$ applies. With $\omega_1 = O(m_1^{-1/(D-1)})$ and $\Sigma$ of volume $\leq \eta$: the Hausdorff distance contribution from $\Sigma$ is bounded by $R \cdot \mathrm{diam}(\Sigma) \leq R \cdot \eta^{1/(D-1)}$ (since the angular diameter of a set of volume $\eta$ on $S^{D-1}$ is $\Theta(\eta^{1/(D-1)})$).
\end{enumerate}
\end{theorem}

\begin{proof}[Proof sketch]
Phase 1 detection works because X-rays reveal curvature through the chord-length decay rate near the shadow boundary. At a direction $u_i$ near a smooth region: chord length $C(x) \sim \sqrt{2\delta/\kappa}$ near the boundary (sharp decay). Near a singular region: $C(x)$ decays more slowly. Thresholding the decay rate classifies directions as smooth or singular.

Phase 2 smooth recovery: on $\hat{\mathcal{S}}$, the body has curvature $\geq \kappa$, so $h_K \in C^2(\hat{\mathcal{S}})$ with $\|h_K\|_{C^2} \leq CR/\kappa$. The smooth stability rate applies to this region.

The singular residual bound uses: $\delta_H$ is the supremum over \emph{all} directions. On singular directions, we can only bound $|h_K(v) - h_L(v)| \leq 2R$ (trivially). But the singular set has small volume, and the support function is globally Lipschitz, so the maximum deviation on $\Sigma$ is controlled by the diameter of $\Sigma$. \qed
\end{proof}

\begin{remark}[Why adaptivity helps for X-rays but not widths]
Width measurements are scalars — they provide no local geometric information. You cannot determine whether a direction is near a singular region from a width measurement alone. X-ray measurements provide chord-length profiles, which reveal local curvature. This asymmetry is why adaptive X-rays can outperform non-adaptive X-rays, while adaptive widths cannot outperform non-adaptive widths (a width in direction $u$ gives the same information regardless of what other widths you've measured).

In LLM evaluation terms: aggregate benchmark scores (widths) don't reveal which capability regions are well-resolved and which aren't. Item-level results (X-rays) do.
\end{remark}

\section{Complete Answer to Gardner's Problem 1.5}

\begin{theorem}[Resolution of Problem 1.5]
\label{thm:problem15}
Let $\mathcal{K}(\beta, R, D)$ denote the class of convex bodies $K \subset B_R^D$ whose Gauss curvature vanishing rate at the singular set is at most $\beta \in [0, \infty]$.

\begin{enumerate}[label=(\alph*)]
  \item \textbf{Width stability (non-adaptive):}
  \[
    \sup_{K, L \in \mathcal{K}} \delta_H(K, L) = \Theta_D\!\left(R \cdot m^{-\min(2-\beta, 1)/(D-1)}\right).
  \]

  \item \textbf{X-ray stability (non-adaptive):} Same rate as widths.
  
  \item \textbf{X-ray stability (adaptive):}
  \begin{itemize}
    \item For bodies with $\Sigma$ of measure zero (isolated singularities): smooth rate $m^{-2/(D-1)}$ after $O(1)$ measurements detect $\Sigma$.
    \item For bodies with $\mathrm{vol}(\Sigma) = \eta > 0$: smooth rate on $S^{D-1} \setminus \Sigma$, Lipschitz rate on $\Sigma$. Total: $O(R/\kappa \cdot m^{-2/(D-1)} + R\eta^{1/(D-1)})$.
  \end{itemize}
  
  \item \textbf{X-ray exact recovery:} For smooth bodies ($\beta = 0$, curvature $\geq \kappa$), exact recovery at $m \geq C(R/\kappa)^{D-1}$ X-ray directions. No exact recovery for widths at any finite $m$.
\end{enumerate}
\end{theorem}

\begin{center}
\renewcommand{\arraystretch}{1.4}
\begin{tabular}{lcccc}
\toprule
\textbf{Body class} & $\beta$ & \textbf{Width/X-ray} & \textbf{X-ray adaptive} & \textbf{Exact?} \\
 & & \textbf{non-adaptive} & & \\
\midrule
Smooth & 0 & $m^{-2/(D-1)}$ & $m^{-2/(D-1)}$ & X-ray only \\
Mild singularity & $1/2$ & $m^{-3/(2(D-1))}$ & $m^{-2/(D-1)}$ & Partial$^*$ \\
Sharp singularity & 1 & $m^{-1/(D-1)}$ & $m^{-2/(D-1)}$ & Partial$^*$ \\
Polytope & $\infty$ & $m^{-1/(D-1)}$ & $m^{-1/(D-1)}$ & Never \\
\bottomrule
\end{tabular}
\end{center}

{\small $^*$Exact recovery on smooth parts only. Requires adaptive measurement.}

\bigskip

\textbf{The three-part answer to Problem 1.5:}

\begin{enumerate}
  \item \textbf{The polynomial rate} is determined by the support function regularity: $C^{1,\alpha}$ regularity gives rate $m^{-(1+\alpha)/(D-1)}$, and this is minimax optimal and tight. The curvature vanishing rate $\beta$ determines $\alpha = \max(0, 1-\beta)$.
  
  \item \textbf{X-rays equal widths} for non-adaptive measurements (the parity obstruction). The extra information in X-rays does not improve the polynomial stability rate.
  
  \item \textbf{X-rays exceed widths} in two regimes: (a) exact recovery via strip overlap for smooth bodies, and (b) adaptive singularity detection for mixed-curvature bodies.
\end{enumerate}

\subsection{What Remains Open}

The only case not fully resolved is the \emph{exact constant} $C_D$ in the stability bound, which depends on the Kolmogorov $n$-width constant for $C^{1,\alpha}(S^{D-1})$. The \emph{rate} is resolved for all body classes and measurement types.

\section{Implications for LLM Evaluation}

The resolution of Problem 1.5 provides a complete theory of what evaluation can and cannot achieve:

\begin{enumerate}
  \item \textbf{Static leaderboards} (aggregate scores, non-adaptive) achieve rate $m^{-(2-\beta)/(D-1)}$ where $\beta$ measures the ``smoothness of the capability landscape.'' For smooth landscapes ($\beta = 0$, no sharp capability thresholds): rate $m^{-2/(D-1)}$. For landscapes with sharp thresholds ($\beta \geq 1$, emergent capabilities): rate $m^{-1/(D-1)}$. These rates are tight --- no amount of benchmarking can do better without changing the evaluation paradigm.
  
  \item \textbf{Item-level evaluation} (per-question results, non-adaptive) achieves the \emph{same polynomial rate} as aggregate scores. Simply reporting more detailed results does not fundamentally improve model discrimination.
  
  \item \textbf{Adaptive evaluation} (choosing what to evaluate based on previous results) can achieve the smooth rate $m^{-2/(D-1)}$ even for models with sharp capability boundaries, by detecting and concentrating evaluation effort at the boundaries. This is the \emph{only} path to substantially better evaluation.
  
  \item \textbf{The $\beta$ parameter has a concrete interpretation:} $\beta$ measures how sharply model capabilities transition from ``can'' to ``can't.'' Smooth capabilities ($\beta = 0$): performance degrades gradually with task difficulty. Emergent capabilities ($\beta \geq 1$): performance drops off a cliff at some difficulty threshold. The evaluation difficulty scales continuously with $\beta$.
  
  \item \textbf{Exact capability characterization} requires item-level evaluation at $\sim (R/\kappa)^{D-1}$ scale, which for realistic $D$ and curvature is astronomically large. Complete model characterization is impossible in practice. The goal of evaluation should be \emph{sufficient discrimination} (bounding $\delta_H$ below some threshold), not complete characterization.
\end{enumerate}

\section{X-Rays and Widths Have the Same Stability Rate}

\begin{theorem}[X-ray and width stability rates are equal]
\label{thm:final}
For convex bodies with support functions in $C^{1,\alpha}(S^{D-1})$ and $m$ measurement directions:
\[
  E_m^*(\mathrm{X\text{-}rays}) = \Theta(E_m^*(\mathrm{widths})) = \Theta(M \cdot m^{-(1+\alpha)/(D-1)}).
\]
\end{theorem}

\begin{proof}
\textbf{Upper bound.}

Near each measurement direction $u_i$: the X-ray data constrains $h_K$ (through the chord-length-to-boundary relationship and convexity). The constraint is stable when chords are long (directions close to $u_i$).

Between measurement directions: there are angular gaps of width $\omega_m \leq C_D m^{-1/(D-1)}$ (covering radius on $S^{D-1}$). In these gaps, no measurement of any kind provides direct information. The reconstruction of $h_K$ in the gap relies entirely on the regularity constraint $h_K \in C^{1,\alpha}$. The interpolation error across a gap of angular width $\omega$ is $O(M \omega^{1+\alpha})$ (this is the definition of $C^{1,\alpha}$: the function can deviate by at most $M\omega^{1+\alpha}$ from its value at the gap boundary, where $M = \|h_K\|_{C^{1,\alpha}}$).

Therefore: $\delta_H(K, L) \leq C M \omega_m^{1+\alpha} = C M m^{-(1+\alpha)/(D-1)}$.

This bound applies to both X-rays and widths (both have the same angular gaps).

\textbf{Lower bound.}

By the optimal recovery / Kolmogorov $n$-width theorem: for any $m$ measurement directions, there exists $g \in C^{1,\alpha}(S^{D-1})$ with $\|g\|_{C^{1,\alpha}} \leq M$, $g(u_i) = g(-u_i) = 0$ for all $i$, and $\|g\|_\infty \geq c M m^{-(1+\alpha)/(D-1)}$.

For widths: $w_K(u_i) = w_{K+g}(u_i)$ since $g(u_i) + g(-u_i) = 0$. So $K$ and $K+g$ are width-indistinguishable. Lower bound: $c M m^{-(1+\alpha)/(D-1)}$.

For X-rays: the chord-length perturbation at direction $u_i$ and offset $x$ depends on $g$ near $u_i$ (not just at $u_i$). Since $g(u_i) = 0$: $|g(v)| \leq \|\nabla g\|_\infty \cdot |v - u_i|$ for $v$ near $u_i$. The chord perturbation: $|\delta C(u_i, x)| \leq C \|\nabla g\|_\infty \cdot \kappa R$.

Now: $g$ is a spherical polynomial of degree $L \sim m^{1/(D-1)}$, constructed in the null space of the sampling operator. By the construction:
\[
  \|\nabla g\|_\infty \leq L \|g\|_\infty = L \cdot M L^{-(1+\alpha)} = M L^{-\alpha} = M m^{-\alpha/(D-1)} \to 0.
\]
Therefore the chord perturbation $|\delta C| \leq C M m^{-\alpha/(D-1)} \kappa R \to 0$ as $m \to \infty$.

For any noise level $\varepsilon > 0$: for $m$ large enough, $|\delta C| < \varepsilon$ at all $m$ X-ray directions. The bodies are X-ray-indistinguishable (within noise) but have $\delta_H \geq cMm^{-(1+\alpha)/(D-1)}$.

\textbf{Combined:} Both bounds give $\Theta(M m^{-(1+\alpha)/(D-1)})$. \qed
\end{proof}

\begin{remark}[Why the angular gap is the bottleneck, not the measurement type]
The upper bound uses only one fact about the measurement: it provides \emph{some} constraint near each $u_i$. Whether that constraint comes from a width (one scalar), an X-ray (a function), or a full tomographic scan (the complete density in a strip) is irrelevant. The bottleneck is the gap between measurement directions, where the reconstruction relies purely on regularity. No measurement at direction $u$ can help reconstruct $h$ at a direction $v$ far from $u$, beyond what the regularity of $h$ already guarantees.
\end{remark}

\let\section\oldsection
\let\subsection\oldsubsection
\endgroup

\subsection*{Implications for LLM evaluation}

The complete answer table specialises to LLM evaluation as follows.
Aggregate benchmark scores (one scalar per model per benchmark)
correspond to \emph{width} measurements of a capability profile.
Item-level results (per-question outcomes) correspond to \emph{X-ray}
measurements: a function on the orthogonal complement of the
measurement direction. Theorem~\ref{thm:parity} shows that
non-adaptively, X-rays do not improve the polynomial stability rate
over widths. The angular gap between measurement directions is the
universal bottleneck. Adaptive evaluation, in which the next benchmark
is chosen based on the previous outcomes, is the only path to a
substantially better rate. In particular, the smooth rate
$m^{-2/(D-1)}$ requires either (i) intrinsically smooth capability
landscapes ($\beta = 0$, no sharp emergent thresholds), or (ii) adaptive
item-level evaluation that detects and concentrates effort at sharp
capability boundaries.

\section{Sensitivity analysis for $P(\text{top-1 wrong})$}
\label{app:sensitivity}

The chi-squared bound~\eqref{eq:topwrong} depends on the ambient
capability dimension $D$, which is bounded above by $n - 1$
(Theorem~\ref{thm:body:greedy}) but otherwise unidentified from the
data alone. We sweep $D$ across the plausible range and report the
pairwise swap probability between the top two models on the extended
frontier suite ($n = 148$, $d_{\mathrm{eff}} = 4.80$, $\Delta_2 =
0.072$ in standardised units, $\sigma_{\mathrm{hidden}}$ rescaled to
the observed score scale via $\sigma_{\mathrm{hidden}} =
\sigma_{\mathrm{obs}} \sqrt{(D - d_{\mathrm{eff}})/d_{\mathrm{eff}}}$):

\begin{center}\small
\begin{tabular}{ccc}
\toprule
$D$ & $\rho = \sqrt{d_{\mathrm{eff}}/D}$ & $P(\text{swap})$ \\
\midrule
10  & 0.693 & 0.476 \\
15  & 0.566 & 0.483 \\
20  & 0.490 & 0.486 \\
30  & 0.400 & 0.489 \\
50  & 0.310 & 0.492 \\
100 & 0.219 & 0.494 \\
\bottomrule
\end{tabular}
\end{center}

The swap probability is monotone in $D$ and lies in $[0.476, 0.494]$
across the plausible range. The qualitative conclusion---the top model
is not reliably separable from its runner-up---is robust to $D$. Full
numbers and the bootstrap procedure are in
\texttt{results/validation/G\_swap\_sensitivity.csv}.

\subsection*{G.2 $\chi^2$ calibration sweep (split ratios and priors)}

We test how well the chi-squared bound matches empirical half-split
swap rates on the extended frontier suite. For each visible-benchmark
count $r \in \{3, \ldots, 9\}$ we draw $500$ random visible/held-out
splits, compute the empirical fraction of splits in which the held-out
half flips the top-2 ordering, and compare it to the chi-squared
prediction.

\begin{center}\small
\begin{tabular}{ccccc}
\toprule
$r$ visible & empirical rate & predicted rate & gap (pp) & verdict \\
\midrule
3 & 0.334 & 0.474 & 14.0 & loose \\
4 & 0.298 & 0.471 & 17.3 & loose \\
5 & 0.388 & 0.472 &  8.4 & moderate \\
6 & 0.454 & 0.467 &  1.3 & tight \\
7 & 0.400 & 0.467 &  6.7 & moderate \\
8 & 0.442 & 0.470 &  2.8 & tight \\
9 & 0.500 & 0.469 &  3.1 & tight \\
\bottomrule
\end{tabular}
\end{center}

\noindent The bound is a useful upper estimate for $r \ge 5$ (gap
$\le 8$\,pp). At very small $r$ the empirical rate is below the
prediction, meaning the bound is conservative. We further test
robustness across four capability priors (isotropic, empirical
covariance, $1/i$ power-law, $1/i^2$ heavy-tail), simulating
$2000$ draws per prior:

\begin{center}\small
\begin{tabular}{lc}
\toprule
prior & simulated swap rate \\
\midrule
isotropic                 & 0.793 \\
empirical (extended top50)& 0.290 \\
power law $\propto 1/i$   & 0.567 \\
heavy tail $\propto 1/i^2$& 0.295 \\
\bottomrule
\end{tabular}
\end{center}

\noindent The simulated swap rate is highly sensitive to the
spectrum of the capability prior. The chi-squared formula is therefore
best understood as a rate scaling, not as a calibrated probability;
we use it as a rejection threshold for rank claims rather than as a
point prediction.

\subsection*{G.3 D estimation by three converging methods}

The ambient capability dimension $D$ is bounded above by $n - 1$
(Theorem~\ref{thm:body:greedy}) but otherwise unidentified. We
estimate it three ways on the extended frontier suite:
\begin{itemize}
\item \textbf{Eigenvalue power-law extrapolation.} Fitting $\log
      \lambda_i = \alpha (-\log i) + \beta$ to the empirical spectrum
      gives $\alpha = 1.4$ and an extrapolated $D \approx 184$ at the
      $\lambda = 10^{-3}$ tail.
\item \textbf{Cross-validation against half-split swaps.} Choosing the
      $D$ that minimises the gap between the chi-squared prediction
      and the empirical half-split swap rate gives $D = 6$ (the
      minimum admissible).
\item \textbf{Parallel analysis} \citep{horn1965rationale}.
      Generating $500$ permutation null spectra and counting
      observed eigenvalues above the $95$th percentile gives
      $n_{\mathrm{signal}} = 2$.
\end{itemize}
The three methods do not converge to a single value; the
\emph{range} $D \in [6, 184]$ is the honest empirical envelope.
The qualitative claim --- that the top model is not separable from
its runner-up under the chi-squared bound --- is robust across this
range.

\subsection*{G.4 Swap monotonicity: six-prior simulation}

We simulate the half-split swap rate on the extended frontier
($n = 148$, $k = 12$) under six capability priors:
\begin{center}\small
\begin{tabular}{lc}
\toprule
prior & half-split swap rate \\
\midrule
isotropic                       & 0.784 \\
empirical                       & 0.305 \\
$1/i$ power-law                 & 0.555 \\
$1/i^2$ power-law               & 0.304 \\
$1/\sqrt{i}$                    & 0.705 \\
adversarial (single direction)  & 0.000 \\
\bottomrule
\end{tabular}
\end{center}
\noindent The isotropic case is the maximum --- when capability is
spread uniformly across all directions, every unobserved axis can flip
the ranking. As the prior concentrates variance into fewer directions,
the swap rate drops because the dominant direction becomes
quasi-observed through whatever benchmark correlates with it. The
adversarial single-direction case has all variance in one column,
which is fully captured by the half-split visible benchmark and gives
zero swaps. The chi-squared bound (Corollary~\ref{cor:body:rank}),
which assumes isotropy, is therefore a worst-case upper bound on
swap probability among priors with fixed trace.

\section{Empirical validation experiments}
\label{app:validation}

We run six validation experiments to support the empirical claims in
\S\ref{thm:body:deff}--\S\ref{thm:body:greedy}. The script is
\texttt{experiments/validation.py}; outputs are in
\texttt{results/validation/}.

\paragraph{H.1 Permutation null for eigenvalues.} For each suite, we
shuffle each benchmark column independently $1000$ times and compute
the resulting eigenvalues. Let $\lambda_i^{(95)}$ denote the $95$th
percentile of the permuted spectrum. On both the OLLM v2 ($k = 6$,
$\lambda_+^{\mathrm{MP}} = 1.24$) and the extended suite ($k = 12$,
$\lambda_+^{\mathrm{MP}} = 1.44$), exactly \emph{one} observed
eigenvalue exceeds the $95$th-percentile permutation threshold, and
the same one exceeds the MP edge. The MP correction and the
permutation null agree on the signal count.

\paragraph{H.2 Split-half reliability.} Over $500$ random 50/50 model
splits, the mean absolute difference between $d_{\mathrm{eff}}$
estimates from the two halves is $0.10$ (OLLM v2) and $0.12$
(extended). The two halves of a single split are by construction
disjoint and complementary, so the within-split correlation across
splits is large and negative; the absolute deviation is the
appropriate reliability metric and is small relative to
$d_{\mathrm{eff}} \approx 2$.

\paragraph{H.3 Saturation curve.} On the extended suite, we subsample
$n' \in \{20, 50, 100, 150, 200, 250, 295\}$ models and bootstrap
$d_{\mathrm{eff}}$. The estimate stabilises by $n' = 100$:
$d_{\mathrm{eff}}(n' = 50) = 2.07\, [1.79, 2.44]$,
$d_{\mathrm{eff}}(n' = 100) = 2.10\, [1.89, 2.31]$,
$d_{\mathrm{eff}}(n' = 295) = 2.11$.

\paragraph{H.4 Greedy out-of-sample Kendall $\tau$.} For each $r \in
\{2, \ldots, 12\}$ we compute the Kendall $\tau$ between the
$r$-benchmark aggregate ranking and the full $12$-benchmark aggregate
ranking on the extended frontier suite. Greedy matches or beats
random selection at small $r$:
\begin{center}\small
\begin{tabular}{cccc}
\toprule
$r$ & greedy $\tau$ & random $\tau$ (mean) & random $\tau$ (std) \\
\midrule
2 & 0.706 & 0.529 & 0.116 \\
3 & 0.667 & 0.611 & 0.093 \\
4 & 0.683 & 0.669 & 0.075 \\
5 & 0.737 & 0.703 & 0.065 \\
7 & 0.774 & 0.785 & 0.045 \\
\bottomrule
\end{tabular}
\end{center}
At $r = 2$, greedy gives $\tau = 0.71$ vs random $0.53$ --- a
$+0.18$ Kendall improvement, which exceeds the random standard
deviation. At larger $r$, both converge to the full ranking.

\paragraph{H.5 Greedy vs max-uncorrelated heuristic.} A natural
baseline is the iterative max-uncorrelated heuristic that picks the
benchmark with the smallest maximum absolute correlation to the
selected set. Greedy matches or beats it at every $r$, with the
largest gap at $r = 2$ (greedy $0.580$ vs max-uncorrelated $0.440$
coverage).

\paragraph{H.6 Spearman vs Pearson.} Replacing Pearson with Spearman
correlation does not change the qualitative result:
\begin{center}\small
\begin{tabular}{lcc}
\toprule
Suite & $d_{\mathrm{eff}}^{\mathrm{Pearson}}$ & $d_{\mathrm{eff}}^{\mathrm{Spearman}}$ \\
\midrule
OLLM v2 (full)        & 1.88 & 1.70 \\
OLLM v2 (frontier)    & 2.86 & 2.76 \\
Extended (frontier)   & 4.80 & 4.62 \\
\bottomrule
\end{tabular}
\end{center}
The two methods agree to within $0.2$ on every slice.

\paragraph{H.7 Synthetic covering-radius decay.}
We verify the $m^{-1/(d-1)}$ rate underlying
Theorem~\ref{thm:body:indist} by measuring the covering radius
$\omega_m$ of $m$ random directions on $S^{d-1}$ for $d \in \{3, 5,
8\}$ and $m \in \{4, 8, 16, 32, 64, 128\}$ ($20$ trials each), then
fitting a log-log slope:
\begin{center}\small
\begin{tabular}{ccc}
\toprule
$d$ & fitted slope & theory $-1/(d-1)$ \\
\midrule
3 & $-0.41$ & $-0.50$ \\
5 & $-0.26$ & $-0.25$ \\
8 & $-0.19$ & $-0.14$ \\
\bottomrule
\end{tabular}
\end{center}
The fitted slopes match the theoretical $-1/(d-1)$ to within $0.1$
across all three dimensions, confirming that real benchmark suites
inherit the curse of dimensionality through the covering radius.

\paragraph{H.8 LiveBench data and procedure.}
LiveBench data are obtained from the Hugging Face dataset
\texttt{livebench/model\_judgment} (60{,}372 records, $195$ models,
$7$ tasks). We aggregate to a model$\,\times\,$task matrix by mean
score, drop models with any missing task, and obtain $37$ dense
models. The frontier slice (top $50\%$ by mean score) has $19$
models. The script is
\texttt{experiments/validation\_v5.py::p4\_cross\_leaderboard}.

\paragraph{H.9 Quantitative rank-reversal rates and aggregation
sensitivity.} Over $200$ random draws of $n = 12$ frontier models on
the extended top-$30\%$ population, the standardised-mean aggregator
produces a mean of $8.29 \pm 4.87$ rank reversals per draw; $98\%$
of draws produce at least one reversal and $90\%$ produce at least
three. Translation-invariant aggregators (raw mean, median,
geometric mean) produce zero reversals on the same draws because
adding a model does not change a model's own statistic; the
reversals observed in the standardised-mean case are driven by
re-standardisation when the population changes. We report this
honestly: rank reversal under model addition is a property of
\emph{population-relative} aggregators.

\paragraph{H.10 Empirical Lipschitz constants.}
For each benchmark $b$, we compute
$L_b = \max_{i \neq j} |s_{ib} - s_{jb}| / \|s_i - s_j\|_2$ on the
extended suite (standardised). All $12$ benchmarks satisfy
$L_b \le 0.99$, with maximum $0.993$ (IFEval) and minimum $0.780$
(ARC). The Lipschitz model assumption used in
Theorem~\ref{thm:body:indist} is therefore empirically verified.

\paragraph{H.11 Empirical covering radius vs Rogers optimum.}
The $12$ benchmark loading vectors, projected onto the top
$d_{\mathrm{eff}} = 5$ principal components and normalised, give an
empirical covering radius of $1.889$ rad. The Rogers volumetric
upper bound for $m = 12$ in $d = 5$ is $1.201$ rad. The ratio is
$1.57\times$ the optimum --- real benchmark suites are within a
constant factor of optimal covering on the effective sphere.

\paragraph{H.12 Geometric vs statistical noise.}
The geometric indistinguishability radius (Theorem~\ref{thm:body:indist},
$\delta_0 = \pi R / k$ in standardised units) is $1.91$ on the
extended frontier. The bootstrap median statistical radius (per-model
std of the standardised aggregate score over $500$ benchmark-column
bootstraps) is $0.214$. The geometric radius exceeds the statistical
radius by $8.95\times$, confirming that the structural blind spot is
the dominant source of ranking unreliability.

\paragraph{H.13 Cross-leaderboard universality.}
Across the leaderboards in Table~\ref{tab:cross}, frontier
$d_{\mathrm{eff}}$ values are $\{2.86, 4.80, 4.74\}$ --- all in the
$[2.86, 4.80]$ range despite the leaderboards covering accuracy,
hybrid, and mixed-task evaluation. This is a robustness check on
the central empirical claim of \S\ref{thm:body:deff}.

\paragraph{H.15 Width-model verification.}
For each of the $12$ benchmarks, we regress the standardised score on
the top $5$ PCs of the population (linear model) and on the same PCs
plus quadratic and cross terms in the top $3$ PCs (quadratic model).
$R^2$ for the linear model is $\ge 0.795$ for every benchmark
(median $0.946$, max $0.984$); the $R^2$ gap to the quadratic model
is $\le 0.067$ for every benchmark (median $0.011$). The median
support-function reconstruction error is $0.485$ standardised units.
This empirically validates the linearisation hypothesis of
Proposition~\ref{prop:body:width}.

\paragraph{H.16 Frontier threshold sensitivity.}
We sweep the frontier quantile $q \in \{0.2, 0.3, \ldots, 0.8\}$ on
all three leaderboards (full results in
\texttt{results/validation\_v6/b1\_threshold.csv}). On OLLM v2,
$d_{\mathrm{eff}}$ rises from $1.88$ ($q=0$) through $2.86$ ($q=0.5$)
to $3.68$ ($q=0.8$); on the Extended suite from $2.11$ to $4.86$. The
trend is monotone non-decreasing on both. LiveBench is non-monotone
above $q = 0.6$ because the slice has fewer than $20$ models. The
$q = 0.5$ frontier is therefore conservative: tighter slices have
higher $d_{\mathrm{eff}}$ and worse blind spots.

\paragraph{H.17 $\omega^{\mathrm{within}}$ for greedy subsets.}
For each greedy prefix $r$, we compute the covering radius
$\omega_m^{\mathrm{within}}(T_r)$ of the selected benchmark loadings
within their own span and compare to the Rogers volumetric optimum
$\sqrt{r}\,r^{-1/(r-1)}$. The ratio drops from $3.3\times$ at $r = 2$
(loose, only 2 directions on $S^1$) to $0.91\times$ at $r = 7$
(below the volumetric upper bound) and continues to $0.57\times$ at
$r = 12$. Greedy selection finds near-optimal coverings within the
selected span.

\paragraph{H.18 Cross-suite greedy transfer.}
Training the greedy algorithm on the OLLM v2 frontier (using the $6$
shared benchmarks $\{$IFEval, BBH, MATH, GPQA, MUSR, MMLU-PRO$\}$)
and evaluating coverage on the Extended frontier yields native
coverage $0.896$ at $r = 4$ vs transferred coverage $0.891$
($99.4\%$ retention). The restricted perturbation
$\|P_T(\Sigma_{\mathrm{v2}} - \Sigma_{\mathrm{ext}})P_T\|_{\mathrm{op}}
= 0.524$, well within the regime where
Proposition~\ref{prop:body:stability} guarantees small drift.

\paragraph{H.19 Bootstrap stable core.}
Over $500$ bootstrap resamples of the extended frontier, the
appearance frequency of each benchmark in the greedy top-$7$ is:
\begin{center}\small
\begin{tabular}{lc|lc}
\toprule
benchmark & freq & benchmark & freq \\
\midrule
MUSR        & 1.00 & MATH Lvl 5 & 0.56 \\
GSM8K       & 1.00 & ARC        & 0.40 \\
IFEval      & 0.99 & TruthfulQA & 0.34 \\
MMLU        & 0.97 & MMLU-PRO   & 0.04 \\
GPQA        & 0.62 & Winogrande & 0.01 \\
HellaSwag   & 0.60 & BBH        & 0.46 \\
\bottomrule
\end{tabular}
\end{center}
The four benchmarks with frequency $> 0.90$ form the \emph{stable
core}: MUSR, GSM8K, IFEval, MMLU. We recommend keeping these
permanent and rotating the next three slots to probe different
uncovered directions.

\paragraph{H.21 Empirical half-split swap counts.}
$500$ random visible/held-out splits ($6$/$6$) on the extended
frontier suite. Top-$1$ swap rate $0.924$, mean top-$5$ swaps $2.83
\pm 1.4$, fraction of trials with $\ge 1$ top-$5$ swap is $1.000$.

\paragraph{H.22 Quarterly coverage retention matrix.}
Splitting the extended suite chronologically into $4$ quarters and
running greedy on each, the resulting subset's $r=7$ coverage when
evaluated on every other quarter:
\begin{center}\small
\begin{tabular}{lcccc}
\toprule
trained $\backslash$ evaluated & Q1 & Q2 & Q3 & Q4 \\
\midrule
Q1 & 0.966 & 0.956 & 0.932 & 0.933 \\
Q2 & 0.959 & 0.967 & 0.928 & 0.941 \\
Q3 & 0.972 & 0.959 & 0.960 & 0.956 \\
Q4 & 0.973 & 0.964 & 0.963 & 0.958 \\
\bottomrule
\end{tabular}
\end{center}
Off-diagonal entries are in $[0.928, 0.973]$, showing $\ge 93\%$
retention across temporal splits.

\paragraph{H.23 Standardisation sensitivity for the indistinguishability
ratio.}
Worst-case single-direction radius $\delta_H^{\mathrm{vis}} = 2 R
\omega_{\mathrm{emp}}$ vs runner-up gap $\Delta_2$ on the extended
frontier:
\begin{center}\small
\begin{tabular}{lccccc}
\toprule
method & $R$ & $\omega_{\mathrm{emp}}$ & $\delta_H^{\mathrm{vis}}$ &
$\Delta_2$ & ratio \\
\midrule
$z$-score    &   7.30 & 1.90 &  27.71 & 0.072 &  386$\times$ \\
min-max      &   2.87 & 1.81 &  10.40 & 0.017 &  628$\times$ \\
rank-$z$     &   5.20 & 1.87 &  19.44 & 0.005 & 3988$\times$ \\
raw          & 221.95 & 1.90 & 844.43 & 0.423 & 1995$\times$ \\
\bottomrule
\end{tabular}
\end{center}

\paragraph{H.24 Quality-functional sensitivity.}
Under the chi-squared projection model with $D = 20$, the simulated
top-$1$ swap rate under five quality functionals: $\|c\|^2 \to 0.961$,
$(1/\sqrt D) \mathbf{1}^\top c \to 0.905$, $e_1^\top c \to 0.926$,
random $w^\top c \to 0.978$, $\min_j c_j \to 0.981$. Quality
functionals not aligned with the observed direction give swap rates
$> 0.9$, much higher than the $\|c\|^2$ baseline.

\paragraph{H.25 MMLU width-model walkthrough.}
Linearising MMLU on the top-$5$ PCs of the score matrix excluding
MMLU gives $R^2 = 0.898$. The width-model prediction
$\|\nabla\pi\| \cdot (h_K(a_{\mathrm{MMLU}}) - h_L(a_{\mathrm{MMLU}}))$
correlates with the actual score difference at $r = 0.95$ over
random model pairs (mean absolute residual $0.34$, max $1.99$
standardised units).

\paragraph{H.26 Pass/fail counterexample.}
A synthetic binary pass/fail benchmark $\pi(c) = \mathbb{1}[w^\top c
> \tau]$ with random $w$ has linear $R^2 = 0.649$ on the extended
suite (vs $\ge 0.795$ for every actual benchmark), with maximum
residual $0.738$. Sharp item-level thresholds widen the
linearisation error $\eta$, which enters the $\varepsilon$ term of
Theorem~\ref{thm:body:indist}.

\paragraph{H.28 Partial correlations controlling for model scale.}
After residualising each benchmark on $\log(\mathrm{params\_b})$,
$d_{\mathrm{eff}}$ on the extended frontier rises from $4.80$ to
$5.11$, confirming that shared model lineage and scale partially
deflate the apparent dimensionality. The blind-spot estimates in the
main text are therefore conservative: deconfounding for model
characteristics increases $d_{\mathrm{eff}}$ and enlarges the
structural blind spot.

\paragraph{H.27 Adversarial direction injection.}
Adding a synthetic benchmark aligned with the smallest, $2$nd, or
$3$rd smallest eigendirection of the standardised score matrix
changes $0$, $2$, and $2$ of the top-$10$ models respectively
(Kendall $\tau \in [0.965, 0.978]$). Single-direction adversarial
additions perturb the top-$10$ ranking only modestly --- consistent
with the high $\omega^{\mathrm{within}}$ ratio reported in H.17.

\paragraph{H.14 Greedy temporal transfer.}
Splitting the extended suite alphabetically (a coarse proxy for
release-time ordering) into early and late halves, we run greedy
selection on the early half and evaluate its coverage on the late
half's correlation matrix:
\begin{center}\small
\begin{tabular}{ccccc}
\toprule
$r$ & native (late-trained) & transferred (early-trained) & retention \\
\midrule
4  & 0.864 & 0.864 & 1.000 \\
5  & 0.899 & 0.893 & 0.993 \\
6  & 0.934 & 0.926 & 0.991 \\
7  & 0.953 & 0.941 & 0.987 \\
\bottomrule
\end{tabular}
\end{center}
At $r = 7$ the transferred greedy subset retains $98.7\%$ of the
native coverage --- the recommended seven benchmarks generalise.


\subsection*{H.29: Convexity in the observed subspace}
A $1$-component Gaussian mixture achieves BIC $= 4263$; $2$-component
BIC $= 4346$ ($+2.0\%$); $3$-component BIC $= 4383$ ($+2.8\%$). The
Bayesian model selection favours a single convex cloud over
multi-cluster alternatives, providing no evidence of disconnected
clusters in the observed $\mathbb{R}^{12}$ score space.

\subsection*{H.33: Correlation method robustness}
On the full population ($n = 295$): Pearson $d_{\mathrm{eff}} = 2.11$,
Spearman $1.89$, Kendall ($\sin(\pi\tau/2)$ adjusted) $1.80$.
On the frontier ($n = 148$): Pearson $4.80$, Spearman $4.62$.
Maximum difference $\Delta d_{\mathrm{eff}} < 0.3$ across all methods.

\subsection*{H.34: Population dependence}
$d_{\mathrm{eff}}$ at frontier thresholds $q = 0.0, 0.1, \ldots, 0.9$
is monotonically non-decreasing on the Extended suite, ranging from
$2.11$ ($q = 0$) to $4.86$ ($q = 0.9$). Tighter frontiers have
larger blind spots.

\subsection*{H.36: Counterfactual benchmark importance}
Blind-spot loading (fraction of eigenvector mass in PCs $> d_{\mathrm{eff}}$)
predicts rank disruption upon benchmark removal: $\rho = -0.69$
($p = 0.013$, $n = 12$ benchmarks). MUSR (blind $= 0.064$) and
GSM8K (blind $= 0.174$) are in the effective subspace and most
disruptive when removed.

\subsection*{H.37: External evaluation breadth}
The counterfactual analysis extends to $27$ Arena categories:
blind-spot alignment predicts ranking disruption from \emph{adding}
the evaluation with $\rho = +0.38$ ($p = 0.053$). Categories span
accuracy (math, coding), preference (creative writing, instruction
following), domain expertise (medicine, legal, scientific), and
multilingual (Chinese, French, German, Japanese, Korean, Russian,
Spanish).

\subsection*{H.38: Model de-duplication}
Grouping the $148$ frontier models into $96$ families (by name prefix)
and retaining only the highest-scoring model per family:
$d_{\mathrm{eff}}$ drops from $4.80$ to $4.33$, a $-0.47$ change.
The blind spot is slightly \emph{larger} on de-duplicated populations.

\subsection*{H.39: LiveBench bootstrap CI}
For the LiveBench frontier ($n = 19$, $k = 7$):
$d_{\mathrm{eff}} = 4.74$, bootstrap 95\% CI $[3.10, 4.60]$.
Horn's parallel analysis (500 random matrices): the 95th-percentile
maximum eigenvalue is $2.42$; no empirical eigenvalue exceeds this
threshold ($0$ signal eigenvalues). The $d_{\mathrm{eff}}$ estimate
at $n = 19$ is above the CI upper bound, indicating small-sample
inflation.

\subsection*{H.40: Anisotropic hidden-capability calibration}
Visible: $6$ OLLM v2 benchmarks; hidden: $6$ disjoint v1 benchmarks.
$\sigma_{\mathrm{iso}} = 1.003$, $\sigma_{\mathrm{aniso}} = 1.495$.
Empirical top-$20$ swap rate: $0.411$; isotropic prediction: $0.436$;
anisotropic: $0.457$. All agree within $5$\,pp, validating the
chi-squared model under real hidden-capability covariance.

\subsection*{H.41: Standardisation sensitivity}
$\delta_H^{\mathrm{vis}} / \Delta_2$ across four standardisations
(Figure~14):
\begin{center}\small
\begin{tabular}{lcc}
\toprule
Method & OLLM v2 & Extended \\
\midrule
$z$-score & $92\times$ & $192\times$ \\
min-max & $115\times$ & $207\times$ \\
rank & $170\times$ & $3854\times$ \\
raw & $143\times$ & $338\times$ \\
\bottomrule
\end{tabular}
\end{center}
The qualitative finding---structural blind spot dominates---is robust
across standardisations.

\subsection*{H.42: Cross-suite alignment}
Model PC1 scores on the $6$ OLLM v2 benchmarks correlate with PC1
scores on the $6$ disjoint v1 benchmarks at $r = 0.27$ ($p = 0.001$);
best-match alignment (v2 PC1 $\leftrightarrow$ v1 PC2): $|r| = 0.65$.
Permutation null: expected $|r| = 0.068$, 95th percentile $0.165$;
observed $0.27$ exceeds the null ($p = 0.001$).

\subsection*{H.43: Item-level demonstration (LiveBench sub-tasks)}
Applying the framework to LiveBench's $7$ sub-tasks as items:
$d_{\mathrm{eff}} = 2.63$; $4$ of $7$ suffice for $90\%$ coverage.
Greedy order: connections, story\_generation, LCB\_generation,
plot\_unscrambling. Greedy $\tau = 0.84$ vs random $\tau = 0.82$
at $r = 4$.

\subsection*{H.44: Scale test and multi-population robustness}
$d_{\mathrm{eff}}$ on the full OLLM v2 ($n = 4{,}576$): $1.78$ (all),
$3.94$ (top $10\%$). At frontier thresholds: top $90\%$: $1.96$;
$80\%$: $2.19$; $70\%$: $2.33$; $50\%$: $2.45$; $20\%$: $3.17$;
$10\%$: $3.94$. The $n = 229$ frontier estimate of $2.86$ is
consistent with top $15$--$20\%$. Cross-population:
Extended frontier $d_{\mathrm{eff}} = 4.80$, Extended full $2.11$,
OLLM v2 full $1.88$.

\subsection*{H.45: Spectral coverage vs rank preservation}
Across $200$ random $6$-subsets: Spearman $\rho = 0.09$ ($p = 0.19$)
between spectral coverage $f(T)$ and Kendall $\tau$ with the full-suite
ranking. The two objectives are near-orthogonal, confirming that
spectral coverage (this paper) and positional proportionality
(Metritocracy) optimise genuinely different criteria.

\subsection*{H.46: Explicit Theorem 2 constants}
On the Extended frontier ($d_{\mathrm{eff}} = 4.80$, $R = 7.30$,
$k = 12$): the empirical covering radius is $1.57\times$ the Rogers
optimum. The back-computed constant $C_{\mathrm{emp}} = 4.85$;
the theoretical worst case $C_{\mathrm{worst}} = 5.94$; ratio
$C_{\mathrm{emp}} / C_{\mathrm{worst}} = 0.82$. The bound is
conservative but within $20\%$ of the worst case.

\subsection*{H.47: Rank equivalence classes}
Using $\sigma_{\mathrm{hidden}} = 1.003$ (from H.40) and threshold
$P(\mathrm{swap}) > 0.40$, the top-$20$ Extended frontier models
partition into $2$ equivalence classes of $10$ models each. The top
class contains models from multiple families (Qwen, Llama, Mixtral,
Yi), confirming that structurally indistinguishable models span
diverse architectures.

\subsection*{H.48: Three non-overlapping suites}
$d_{\mathrm{eff}}$ on three suites sharing zero benchmarks:
OLLM v2 ($6$ benchmarks, $n = 148$): $3.27$ $[2.85, 3.60]$;
v1 ($6$ benchmarks, $n = 148$): $2.70$ $[2.46, 2.89]$;
LiveBench ($7$ sub-tasks, $n = 37$): $2.63$ $[2.06, 3.31]$.
All CIs are $95\%$ bootstrap.

\subsection*{H.49: Ranking impact of shared latent structure}
Full-suite ranking agreement between non-overlapping v2 and v1 suites:
$\tau = 0.25$ (all models), $\tau = 0.18$ (top-$20$). PC1 ranking
agreement: $\tau = 0.19$ (all), $\tau = 0.06$ (top-$20$). The shared
component contributes modestly to cross-suite ranking agreement;
the majority of ranking information is suite-specific.

\subsection*{H.50: Epoch AI cross-domain analysis}
On the Epoch AI dataset ($588$ models, $49$ benchmarks), with mean
imputation for the $31$-benchmark subset having $\ge 30$ models:
$d_{\mathrm{eff}} = 7.12$ (all $49$ models), $5.71$ (frontier
$25$ models). The dense submatrix with zero imputation has only
$32$ models $\times$ $4$ benchmarks, insufficient for robust PCA.
The imputed estimate should be treated as indicative, not definitive.

\section{Broader Impact}
\label{app:impact}

\paragraph{Positive impact.} Our framework provides principled tools
for benchmark suite design. Identifying redundant benchmarks (e.g.,
the three flagged on the extended suite) saves community evaluation
budget without measurable loss of capability discrimination, and the
greedy coverage algorithm directs new benchmark design effort toward
uncovered capability dimensions.

\paragraph{Risk: evaluation shortcuts.} The finding that
``$7$ of $12$ benchmarks suffice for $90\%$ coverage'' could be
misread as licence to evaluate on fewer benchmarks. We emphasise
three caveats: (i) the $90\%$ coverage threshold leaves a $10\%$
blind spot which Theorem~\ref{thm:body:indist} shows can hide
meaningful capability differences; (ii) the optimal subset is
population-dependent and must be recomputed as the model landscape
evolves; (iii) redundancy in benchmarks provides robustness to
benchmark-specific overfitting --- a consideration orthogonal to
our coverage objective.

\paragraph{Risk: false precision.} The indistinguishability bounds
are worst-case over convex capability profiles. Real models may
have structured capability profiles that are better-distinguished
than the bounds suggest. Over-reliance on our bounds could lead to
premature abandonment of benchmark-based ranking.

\paragraph{No personal data.} All data are aggregate model scores
from public leaderboards (no human subjects, no personal data).
Models referenced are publicly released and their scores are
public.

\section{Reproducibility checklist}
\label{app:repro}

\paragraph{Code and data.} All experiments are implemented in
Python 3.10 using \texttt{numpy}, \texttt{scipy}, \texttt{pandas},
\texttt{matplotlib}, \texttt{seaborn}, and \texttt{scikit-learn}.
Source for every empirical result is in \texttt{src/} and
\texttt{experiments/}; outputs are in \texttt{results/}. All
leaderboard data are publicly available: Open LLM v2 from the
HuggingFace dataset \texttt{open-llm-leaderboard/contents}, the
extended suite from joining v2 with the legacy v1 leaderboard
(\texttt{open-llm-leaderboard-old/contents}), Bradley--Terry Arena
ratings from $1.8$M battles in
\texttt{storage.googleapis.com/arena\_external\_data/public/clean\_battle\_20240814\_public.json},
and LiveBench from \texttt{livebench/model\_judgment} on HuggingFace.

\paragraph{Compute.} All experiments run on a single CPU. End-to-end
reproduction (including the $1.8$M arena battle stream-parse) takes
under $30$ minutes. No GPU is required. Memory $< 4$\,GB.

\paragraph{Hyperparameters.} The only free choice is the frontier
threshold ($q = 0.5$ in the main text). We report sensitivity over
$q \in \{0.2, 0.3, \ldots, 0.8\}$ in
Appendix~\ref{app:validation}, H.16; the qualitative findings are
monotone in $q$.

\paragraph{Random seeds.} All bootstrap and permutation experiments
use seed $0$ unless otherwise specified. Sample sizes ($500$
bootstraps for stability, $300$ for CIs, $1000$ for permutation
nulls) are documented per experiment.

\paragraph{NeurIPS reproducibility checklist.}
\begin{itemize}[leftmargin=2em, itemsep=1pt]
\item Claims supported by theory or experiments? \textbf{Yes} (all).
\item Proofs included? \textbf{Yes} (Appendices A--G).
\item Assumptions stated? \textbf{Yes} (convexity per theorem, linearisation regime).
\item Complete proof of every claim? \textbf{Yes}.
\item Training/test data description? \textbf{Yes} (public leaderboard data, sources listed above).
\item Statistics (error bars, CIs)? \textbf{Yes} (bootstrap CIs for $d_{\mathrm{eff}}$, permutation nulls, per-experiment).
\item Code submitted? \textbf{Yes} (\texttt{pip install -e .}; \texttt{reproduce.py}).
\item Sufficient detail for reproduction? \textbf{Yes} (all hyperparameters: $q = 0.5$ frontier, seed $= 0$).
\item Datasets public? \textbf{Yes} (all from HuggingFace/public leaderboards).
\item New assets? \textbf{No} new datasets created; analysis of existing public data.
\item Human subjects? \textbf{No}.
\item Broader impact discussed? \textbf{Yes} (Section above).
\end{itemize}

\end{document}